\pdfoutput=1
\documentclass[11pt]{article}
\usepackage[margin=1in]{geometry}
\usepackage{amsmath,amssymb,mathtools,amsthm}
\usepackage{bm}
\usepackage{enumitem}
\usepackage{hyperref}
\usepackage{xcolor}
\usepackage{caption}       
\usepackage{microtype}
\usepackage{booktabs}
\usepackage{float}
\usepackage{algorithm}
\usepackage{algorithmic}
\usepackage{placeins}
\usepackage{graphicx}
\usepackage{wrapfig}

\usepackage[
  backend=biber,
  style=numeric,
  maxbibnames=999,  
  maxcitenames=2    
]{biblatex}
\theoremstyle{remark}
\newtheorem{remark}{Remark}
\usepackage{authblk}

\title{\textbf{Predicting When Random Low-Dimensional Reparameterizations Train Neural Networks}}

\author[1,2]{Andrew Cheng\textsuperscript{\ensuremath{\dagger}}}
\author[3]{Ali Eslamian\textsuperscript{\ensuremath{\dagger}}}
\author[4]{Jie Cheng}
\author[5]{Mehdi Zargham}
\author[3,6]{Qiang Cheng\textsuperscript{*}}

\affil[1]{Department of Computer Science,
Tsinghua University, Beijing, China}

\affil[2]{Department of Computer Science,
The University of Manchester, Manchester, UK}

\affil[3]{Department of Computer Science,
University of Kentucky, Lexington, KY, USA}

\affil[4]{Department of Computer Science,
Miami University, Oxford, OH, USA}

\affil[5]{Department of Computer Science,
University of Dayton, Dayton, OH, USA}

\affil[6]{Institute for Biomedical Informatics,
University of Kentucky, Lexington, KY, USA}

\date{}
\hypersetup{
  colorlinks=true,
  linkcolor=blue!60!black,
  citecolor=blue!60!black,
  urlcolor=blue!60!black
}

\newtheorem{theorem}{Theorem}
\newtheorem{lemma}{Lemma}
\newtheorem{definition}{Definition}

\newtheorem{corollary}[theorem]{Corollary}
\newtheorem{approximation}{Approximation}

\begin{document}
\maketitle

\vspace{-3.8em}
\begin{center}
\small
\textsuperscript{\ensuremath{\dagger}}These authors contributed equally to this work.\\
\textsuperscript{*}Corresponding author: Qiang Cheng,
\href{mailto:qiang.cheng@uky.edu}{\texttt{qiang.cheng@uky.edu}}
\end{center}

\begin{abstract}
Neural networks can often be trained or fine-tuned using
\emph{random low-dimensional reparameterization}: instead of optimizing all
model parameters directly, one optimizes a small latent vector that is mapped
into a full parameter update through a frozen random map. This procedure
restricts training to a randomly chosen low-dimensional search space and
raises a practical question: how large must this search space be to access
parameters that achieve low loss? 
Random low-dimensional training is known to exhibit a sharp accessibility
transition governed by the geometry of the low-loss region. We first express
this phenomenon in an equivalent conic form, in which the transition for
compact convex targets is centered at the statistical dimension of the polar
cone. Our main theoretical contribution is an operational quadratic analysis
that goes beyond this geometric characterization. We derive an
orientation-resolved master formula that predicts the random-slice residual
from both the curvature spectrum and the reference-to-solution displacement
profile, thereby making explicit how displacement orientation relative to
curvature affects the required latent dimension. The formula yields a new
self-consistent isotropic-orientation predictor and, in its conservative
radius-only specialization, recovers the earlier Gaussian-width quadratic
bound.
Building on this analysis, we introduce Random Mapping Networks (RaMaN), a
scalable framework that instantiates the predicted latent dimension $d$ using
structured Hadamard mappings or seed-regenerated Gaussian maps for a network
with $P$ parameters. These constructions eliminate the $O(dP)$ frozen-map
storage required by dense random generators and reduce optimizer-state memory
from $O(P)$ to $O(d)$. We further develop matrix-free curvature approximations
and sweep-free dimension-selection procedures, and design a benchmark that
measures accessibility transitions across tasks, architectures, and map
families while accounting for trainable parameters, frozen-map storage,
optimizer-state memory, checkpoint size, and runtime. Across controlled
quadratic and neural-curvature experiments, the orientation-resolved predictor
closely tracks measured transition locations and substantially improves over
orientation-agnostic approximations in settings where displacement direction
matters; end-to-end experiments further demonstrate sharp, protocol-dependent
training transitions across image and language models.
\end{abstract}

\section{Introduction}
\label{sec:intro}

Modern deep neural networks are typically trained by directly optimizing all
parameters of a high-dimensional weight vector
$\theta \in \mathbb{R}^P$, where $P$ may range from millions to billions or more.
Although this direct parameterization is highly expressive, it is also
expensive: optimizers must maintain gradients and auxiliary states for every
parameter, checkpoints scale with $P$, and unrestricted optimization may
overfit when the effective degrees of freedom required by the task are much
smaller than the ambient parameter dimension. This gap between the ambient
dimension of the model and the effective dimension of the optimization
problem motivates low-dimensional reparameterization.

In {\emph{low-dimensional reparameterization}}, the full parameter vector
$\theta \in \mathbb{R}^P$ is not optimized directly. Instead, it is generated
from a much smaller trainable latent vector $z \in \mathbb{R}^d$, with
$d \ll P$, through a fixed or partially fixed mapping
\[
    \theta = g_\omega(z),
\]
where $\omega$ denotes random, structured, or otherwise non-trainable
parameters of the mapping. Equivalently, one may write
\[
    \theta = \theta_{\mathrm{ref}} + \Delta\theta(z),
\]
where $\theta_{\mathrm{ref}}$ is a reference point, such as a random initialization, a pretrained model, or a pilot solution, and $\Delta\theta(z)$ is a constrained parameter displacement generated from the low-dimensional latent variable $z$.
Training then optimizes only $z$, while the target network
performs its standard forward computation using the generated weights
$\theta(z)$.

This idea has a simple geometric interpretation.
Direct training searches over the full parameter space $\mathbb{R}^P$,
whereas low-dimensional reparameterization restricts training to the image
\[
    \operatorname{Im}(g_\omega)
    =
    \{g_\omega(z): z \in \mathbb{R}^d\},
\]
where $\omega$ denotes the frozen random draw defining the map.
When $g_\omega$ is smooth, this image has at most $d$ local degrees of
freedom, because the Jacobian of $g_\omega$ at any
$z\in\mathbb{R}^d$ has rank at most $d$.
Such a restriction can be beneficial if the low-loss region is wide in most
parameter directions and constrained only along a relatively small number of
important directions. In this case, the low-dimensional image need not cover
the full ambient space $\mathbb{R}^P$; it only needs enough degrees of freedom
to satisfy these effective constraints. We refer to the corresponding number
of constrained directions as the \emph{effective codimension} of the
low-loss region. This viewpoint provides a geometric explanation for
empirical observations that neural networks can often be trained or
fine-tuned successfully in surprisingly small random subspaces
\cite{aghajanyan2021intrinsic,li2018measuring}.

However, the central challenge is not merely whether a low-dimensional
random search space can reach a low-loss region. Prior work has already
shown that random-subspace trainability exhibits a sharp geometric
transition as the latent dimension increases
\cite{li2018measuring,larsen2022degrees}. The practically important
question is whether this transition can be predicted from quantities that
are accessible without performing an exhaustive sweep over latent
dimensions. In particular, given a frozen map
$g_\omega:\mathbb{R}^d\to\mathbb{R}^P$ drawn from a random or structured
random family, we would like to determine a dimension $d$ for which
\[
\operatorname{Im}(g_\omega)\cap S_\varepsilon\neq\emptyset
\]
with high probability over $\omega$, where
\[
S_\varepsilon
=
\left\{
\theta:
\mathcal{L}(\theta)\le\mathcal{L}^\star+\varepsilon
\right\}
\]
is an $\varepsilon$-low-loss sublevel set. A related algorithmic question
is whether gradient-based optimization over the resulting latent variables
can exploit such an accessible low-loss region under realistic training
budgets and map constructions.

Recent approaches provide substantial evidence that optimizing a small
number of trainable variables can sometimes approach or match
full-parameter training, including intrinsic-dimension
training~\cite{li2018measuring}, low-dimensional
fine-tuning~\cite{aghajanyan2021intrinsic}, random-basis
parameterizations~\cite{nooralinejad2022pranc,koohpayegani2023nola}, and
Mapping Networks~\cite{sen2026mapping}. Mapping Networks, for example,
optimize a low-dimensional latent vector $z\in\mathbb{R}^d$ that modulates
a randomly initialized mapping network with frozen base weights to generate
the $P\gg d$ parameters of a target network. These methods establish the
practical feasibility of low-dimensional parameterizations, but they do
not by themselves provide a prospective rule for selecting the latent
dimension from the local geometry of the target loss.

Larsen et al.~\cite{larsen2022degrees} provided an important geometric
explanation for this phenomenon. They characterized random-subspace
trainability through the angular Gaussian width of a low-loss sublevel set
viewed from the initialization, relating the observed transition to the
geometry of the target region. This characterization, however, is difficult
to use prospectively: the Gaussian width of an unknown neural-network
low-loss region is generally not directly available before performing the
dimension sweeps one would like to avoid. Moreover, in the quadratic
setting, the resulting radius-only approximation depends on the Hessian
spectrum and the scalar distance from the reference point to a minimum, but
does not explicitly resolve how that displacement is oriented relative to
the curvature eigendirections. This leaves an operational gap between a
geometric characterization of the transition and a computable predictor
that can be used for latent-dimension selection.

We address this gap in two stages. First, for compact convex target sets,
we express the known random-slice accessibility transition in an equivalent
conic form, in which its location is characterized by the statistical
dimension of the polar cone generated by the shifted low-loss region. This
formulation provides a convenient effective-codimension interpretation and
a bridge to the local quadratic analysis. Our main theoretical contribution
is then an orientation-resolved quadratic predictor. For a
positive-semidefinite local curvature model, we derive a master formula that
predicts the random-slice residual using both the curvature spectrum and the
complete reference-to-solution displacement profile. Consequently, two
problems with the same Hessian spectrum and the same displacement norm can
have different predicted latent dimensions when their displacements are
oriented differently relative to the curvature eigenspaces. A new
self-consistent isotropic-orientation predictor follows by imposing an
equal-energy displacement model, while an orientation-uniform radius-only
specialization recovers the earlier quadratic bound of
Larsen et al.~\cite{larsen2022degrees}. The latter provides a conservative
choice when directional displacement information is unavailable.

This local quadratic setting is both analytically tractable and empirically
relevant: studies of deep-network Hessians commonly report a small number
of large outlier eigenvalues together with a bulk of near-zero eigenvalues,
corresponding to a few stiff directions and many nearly flat
directions~\cite{sagun2017empirical,ghorbani2019investigation,
papyan2020traces}. Because exact spectral information is unavailable at
modern network scale, we further develop matrix-free approximations based
on leading curvature directions and stochastic estimates of the unresolved
spectral mass. These approximations support practical latent-dimension
selection without requiring a full sweep over candidate dimensions.

Building on this analysis, we introduce \emph{Random Mapping Networks}
(RaMaN), a scalable framework that couples latent-dimension selection with
fixed random low-dimensional reparameterizations. RaMaN supports
orientation-resolved selection when an estimated displacement profile is
available, radius-based selection when only a displacement scale is known,
and nested adaptive expansion when neither estimate is sufficiently
reliable. The selected dimension is instantiated using scalable map
families, including structured Hadamard mappings~(SHM), motivated by fast
structured Hadamard constructions such as SRHT and
Fastfood~\cite{ailon2009fast,tropp2011improved,le2013fastfood}, and
seed-regenerated Gaussian maps. A global seed-regenerated Gaussian
construction realizes the uniformly random subspace assumed by
Theorem~\ref{thm:hessian_effective_rank}, while the default layer-wise
Gaussian construction corresponds to the independent layer-wise setting of
Corollary~\ref{cor:layerwise} under its product-structure assumptions.
SHM provides a structured, memory-efficient alternative whose empirical
behavior is evaluated separately. These constructions avoid storing a
dense $P\times d$ random mapping object and reduce optimizer-state memory
from $O(P)$ to $O(d)$.

Our goal is not to claim that all large neural networks can be trained
through only a few thousand latent variables. Rather, we ask when a
substantial reduction in trainable dimension is geometrically plausible,
whether the required dimension can be predicted from measurable local
curvature and displacement information, and how map structure, optimizer,
and training protocol affect the resulting transition. This perspective
turns low-dimensional reparameterization from a purely empirical
dimension-sweep procedure into a framework with explicit, falsifiable
predictions.

Our main contributions are summarized as follows:
\begin{itemize}
\item \textbf{Theory.}
Building on the Gaussian-width characterization of random-subspace training, we recast the accessibility transition in an equivalent conic/statistical-dimension form and derive a new orientation-resolved quadratic master formula. Unlike prior radius-only quadratic bounds, the master-formula predictor depends jointly on the curvature spectrum and the reference-to-solution displacement profile. Two useful specializations of this predictor follow: a new self-consistent isotropic-orientation predictor, and an orientation-uniform predictor that recovers the earlier radius-only bound as a special case.

\item \textbf{Method.}
We develop Random Mapping Networks (RaMaN), which couple
curvature- and displacement-informed latent-dimension selection with
scalable frozen random reparameterizations. RaMaN supports
orientation-resolved and radius-based sweep-free selection together with
nested adaptive dimension expansion that preserves the previously active
random search space. To instantiate the selected dimension without storing
a dense $P\times d$ map, we develop a seeded Structured Hadamard Mapping
(SHM) based on nested selection of Hadamard basis directions,
$\operatorname{crop}\,H D I_{\Omega}$, and a seed-regenerated Gaussian
construction whose active columns form prefixes of a fixed virtual Gaussian
matrix. These constructions preserve the fixed-slice interpretation
underlying the theory while providing scalable matrix-free training.

\item \textbf{Measurement.}
We design a benchmark for measuring the phase transition of random
low-dimensional training across tasks, architectures, and mapping
families. The benchmark evaluates the orientation-resolved
master-formula predictor and its isotropic-orientation and
orientation-uniform specializations, compares Gaussian and structured
random-map families and practical dimension-selection strategies, and
reports not only trainable parameters but also frozen-map storage,
optimizer-state memory, checkpoint size, and wall-clock cost.
\end{itemize}

\section{Related Work}
\label{sec:related_work}

\paragraph{Low-dimensional reparameterization and intrinsic dimension.}
A growing line of work suggests that the number of degrees of freedom required
to train or adapt a neural network can be much smaller than its ambient
parameter dimension. Li et al.~\cite{li2018measuring} introduced the notion of
the intrinsic dimension of an objective landscape by training networks within
randomly oriented linear subspaces and measuring the subspace dimension at
which successful solutions first appear. Their large-scale experiments also
used structured Fastfood projections~\cite{le2013fastfood} to avoid explicitly
storing dense random projection matrices, motivating the use of scalable
structured maps for random-subspace training. Li et al.\ provided important
empirical evidence for a sharp dimension-dependent transition but did not
derive a geometric theory predicting its location.

Larsen et al.~\cite{larsen2022degrees} subsequently provided such a geometric
explanation using Gordon's escape-through-a-mesh theorem. They defined the
\emph{local angular dimension} of a low-loss sublevel set as the squared
Gaussian width of its spherical projection about the initialization and
related the threshold training dimension to the complement of this quantity.
For quadratic wells, they further derived a radius-only bound based on the
Hessian spectrum and the distance from the initialization to a minimum. Our
work builds directly on this geometric perspective. For compact convex
targets, we express the same accessibility geometry through the statistical
dimension of the cone generated by the shifted low-loss set and its polar.
Our main distinction is operational: we derive an orientation-resolved
quadratic master formula that retains the complete
reference-to-solution displacement profile relative to the curvature
eigendirections. Its self-consistent isotropic-orientation specialization
provides a new predictor, whereas its conservative orientation-uniform
specialization recovers the radius-only quadratic bound of
Larsen et al.

Aghajanyan et al.~\cite{aghajanyan2021intrinsic} extended the
intrinsic-dimension viewpoint to language-model fine-tuning, showing that
pretrained models can often be adapted through very low-dimensional random
reparameterizations. VeRA~\cite{kopiczko2024vera} further reduces the
trainable parameter count in parameter-efficient fine-tuning by sharing
frozen random low-rank matrices across layers and learning only small scaling
vectors. These studies demonstrate the practical utility of low-dimensional
and frozen-random parameterizations, while the required latent dimension is
typically selected or assessed empirically rather than predicted from the
local curvature and displacement geometry of the target loss.

\paragraph{Random-basis and generator-based parameterizations.}
Gressmann et al.~\cite{gressmann2020improving} studied neural-network
optimization in low-dimensional random bases and introduced independent
random projections for different parts of a network together with
on-demand pseudorandom generation from shared seeds to reduce projection
storage. Their optimization strategy differs fundamentally from the frozen
random-slice setting considered here: they found that keeping a projection
fixed could impair optimization and instead redraw the random subspace
during training. RaMaN retains a fixed random map because its geometric question concerns whether a particular frozen random search space can access a prescribed low-loss region; seed regeneration is used as a scalable implementation mechanism.

Other methods parameterize neural weights through frozen random bases.
PRANC~\cite{nooralinejad2022pranc} represents a model as a linear combination
of randomly initialized frozen basis networks, training only the combination
coefficients. NOLA~\cite{koohpayegani2023nola} applies a related idea to
parameter-efficient fine-tuning by representing low-rank LoRA
updates~\cite{hu2022lora} as linear combinations of random basis matrices.
These approaches reduce trainable parameters and checkpoint size but do not
provide a prospective rule for selecting the dimension of a fixed random
search space from the local geometry of the target objective.

Mapping Networks~\cite{sen2026mapping} take a generator-based approach in
which a compact trainable latent representation is mapped into the much
higher-dimensional parameter space of a target network. This provides
further evidence that target-network parameters can be controlled through
far fewer trainable degrees of freedom. However, such a generator-based
construction does not by itself characterize the accessibility threshold of
a particular frozen random parameterization. Moreover, a dense linear random
map from $\mathbb{R}^d$ to $\mathbb{R}^P$, when used to realize such a
low-dimensional parameterization directly, requires $O(dP)$ frozen storage.
RaMaN therefore combines dimension prediction with matrix-free structured or
seed-regenerated map realizations, rather than treating random-map generation
itself as the principal contribution.

\paragraph{Conic geometry and random intersection theory.}
Our theoretical development builds on tools from high-dimensional convex
geometry and Gaussian process theory. Gordon's escape-through-a-mesh theorem
provides foundational probability bounds for when random subspaces avoid a
subset of the sphere~\cite{gordon1988milman}. Conic integral geometry sharpens
this picture for convex cones: random cone intersections undergo phase
transitions governed by statistical dimension
\cite{amelunxen2014living}. Related Gaussian-width and high-dimensional
probability tools are summarized by Vershynin~\cite{vershynin2018high} and
have played central roles in convex optimization, compressed sensing, and
high-dimensional statistics.

Larsen et al.~\cite{larsen2022degrees} brought this random-intersection
viewpoint directly to neural-network training through the Gaussian width of a
loss sublevel set viewed from the initialization. Our conic analysis should
therefore be viewed as a refinement and reformulation of this geometric
foundation rather than as the first random-intersection characterization of
low-dimensional training. For compact convex targets, the spherical set used
in the Gaussian-width analysis is the spherical section of the cone generated
by the shifted target, allowing the transition to be expressed equivalently
through the statistical dimension of that cone or its polar. This
effective-codimension formulation provides the geometric foundation for our
main extension: an orientation-resolved quadratic residual analysis that
turns the otherwise difficult-to-evaluate geometric threshold into an
operational predictor based on curvature and displacement information.
Together with its isotropic-orientation and orientation-uniform
specializations, this predictor motivates the sweep-free dimension-selection
procedures used by RaMaN.

\section{Methodology}
\label{sec:method}
\subsection{Problem Setup and Random Reparameterization}
\label{subsec:notation}

Let $\theta \in \mathbb{R}^P$ denote the parameter vector of the target
network, where $P$ is the total number of parameters, and let
$\mathcal{L}(\theta)$ denote the training objective. We consider
low-dimensional reparameterizations of the form
\[
\theta = g_\omega(z),
\qquad
z \in \mathbb{R}^d,
\quad
d \ll P,
\]
where $z$ is trainable and $\omega$ denotes random, structured, or otherwise
non-trainable parameters of the mapping. Equivalently, for fine-tuning or
residual parameterization, we write
\[
\theta(z)
=
\theta_{\mathrm{ref}} + \Delta\theta(z),
\]
where $\theta_{\mathrm{ref}}$ is a reference parameter vector, such as a
pretrained model, a random initialization, or a pilot solution. In the
geometric analysis below, we write
$\theta_0:=\theta_{\mathrm{ref}}$ for this reference point.

For a target tolerance $\varepsilon>0$, define the low-loss sublevel set
\[
S_\varepsilon
:=
\left\{
\theta:
\mathcal{L}(\theta)
\le
\mathcal{L}^\star+\varepsilon
\right\},
\]
where $\mathcal{L}^\star$ denotes the best achievable loss within the
parameter region under consideration. For a set
$\mathcal{A}\subseteq\mathbb{R}^P$, define its conic hull by
\[
\operatorname{cone}(\mathcal{A})
:=
\overline{
\left\{
t a:
t\ge 0,\;
a\in\mathcal{A}
\right\}
},
\]
where the closure ensures that the resulting cone is closed. Given the
reference point $\theta_0$, we write
\[
C_\varepsilon
:=
\operatorname{cone}(S_\varepsilon-\theta_0).
\]
Geometrically, $C_\varepsilon$ contains the rays from $\theta_0$ toward the
low-loss region. Its polar cone is
\[
C_\varepsilon^\circ
:=
\left\{
v\in\mathbb{R}^P:
\langle v,u\rangle\le0
\ \text{for all }u\in C_\varepsilon
\right\}.
\]

For a general neural-network objective, $S_\varepsilon$ and
$C_\varepsilon$ need not be convex. The conic phase-transition results below
therefore apply to localized compact convex targets. In particular, in the
local quadratic setting we use
\[
S
=
S_\varepsilon
\cap
\overline{B}(\theta_0,R_{\mathrm{loc}})
\]
and write
\[
C
=
\operatorname{cone}(S-\theta_0).
\]
For the compact convex targets used in our theorems, this generated cone is
already closed, so no additional closure is required
(Lemma~\ref{lem:cone_reduction}).

For a closed convex cone $C$, its statistical dimension is
\[
\delta(C)
:=
\mathbb{E}_{g\sim\mathcal{N}(0,I_P)}
\left[
\|\Pi_C(g)\|_2^2
\right],
\]
where $\Pi_C(g)$ denotes the Euclidean projection of $g$ onto $C$.
Statistical dimension plays the role of an effective dimension for a convex
cone. Its polar counterpart $\delta(C^\circ)$ will characterize the
effective number of constrained directions that a random low-dimensional
reparameterization must capture.

\begin{remark}[Relation to the local angular dimension of Larsen et al.]
\label{rem:larsen_angular_dimension}
For the localized compact convex target $S$, define its spherical view from
$\theta_0$ by
\[
\Sigma_{\theta_0}(S)
:=
\left\{
\frac{\theta-\theta_0}
{\|\theta-\theta_0\|_2}:
\theta\in S
\right\}.
\]
Since
$C=\operatorname{cone}(S-\theta_0)$,
\[
\Sigma_{\theta_0}(S)
=
C\cap\mathbb{S}^{P-1}.
\]
Larsen et al.~\cite{larsen2022degrees} characterize random-subspace
trainability through the squared Gaussian width of this spherical set,
which they term the local angular dimension. For a closed convex cone,
squared Gaussian width and statistical dimension satisfy
\[
w^2\!\left(C\cap\mathbb{S}^{P-1}\right)
\le
\delta(C)
\le
w^2\!\left(C\cap\mathbb{S}^{P-1}\right)+1
\]
\cite{amelunxen2014living}.
Thus, the statistical-dimension formulation used here provides an equivalent
conic refinement of the earlier Gaussian-width perspective: the corresponding
Gaussian-width and statistical-dimension transition centers differ by at
most one dimension. We use this conic formulation as a convenient geometric
foundation for the orientation-resolved quadratic analysis developed below.
\end{remark}

With this notation, we study two related questions. First, can the latent
dimension required for a frozen random search space to access a prescribed
low-loss region be predicted from measurable properties of the local loss
geometry, rather than identified through an exhaustive sweep over $d$?
Second, once such a low-loss region is accessible within the random
parameterization, how effectively can gradient-based optimization over $z$
reach it under practical training conditions?

\subsection{Critical Dimension and Hessian Predictors}
\label{subsec:main_theory}

We first state the geometric accessibility transition underlying random
low-dimensional reparameterization and then address the more operational
question of predicting its location from local curvature and displacement
information. Choosing $d$ too small makes access to a prescribed low-loss
region unlikely, whereas choosing it unnecessarily large weakens the
computational advantages of low-dimensional training.


\begin{figure}[!tbh]
\centering
\includegraphics[
width=\linewidth,
height=0.5\textheight,
keepaspectratio,
page=1
]{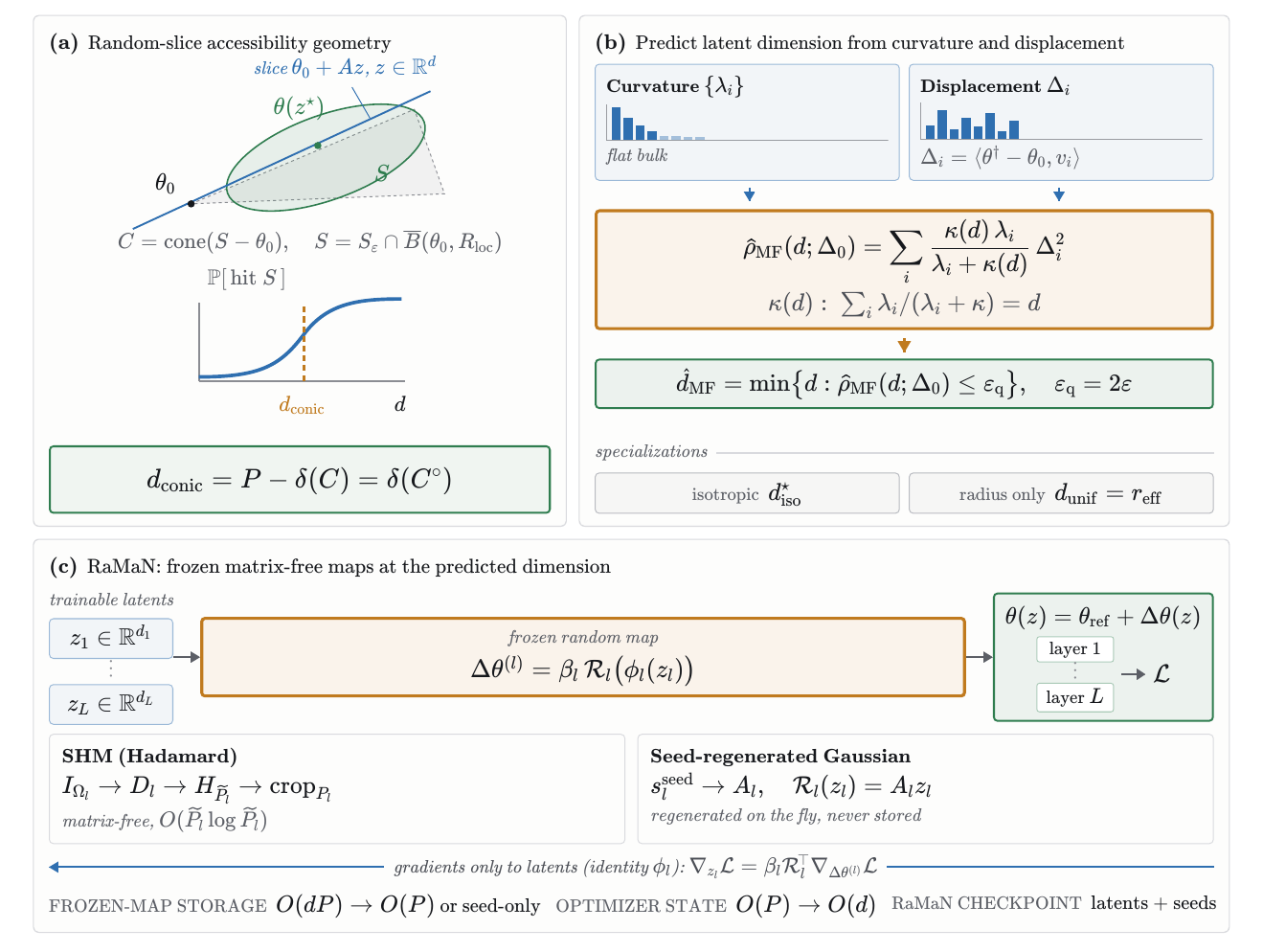}
\caption{\textbf{Overview of Random Mapping Networks (RaMaN).}
(a) Random low-dimensional training is formulated as an intersection problem:
a random affine slice from $\theta_0$ succeeds when it intersects the localized
target $S$, with $C=\operatorname{cone}(S-\theta_0)$ and the conic transition
centered at $d_{\mathrm{conic}}=\delta(C^\circ)$.
(b) In the local quadratic regime, the transition is estimated from the
curvature spectrum and reference-to-solution displacement using the
orientation-resolved master-formula predictor, with isotropic-orientation
and radius-only effective-rank specializations when less directional
information is available.
(c) RaMaN instantiates the predicted latent dimension with frozen
matrix-free maps, including Structured Hadamard Mapping (SHM) and
seed-regenerated Gaussian maps. Only the latent variables $\{z_l\}_{l=1}^{L}$ are trained; each generates a layer-wise update
$\Delta\theta^{(l)}=\beta_l\mathcal{R}_l(\phi_l(z_l))$, reducing dense-map
storage and optimizer-state memory.}
\label{fig:raman_arch}
\end{figure}

The conic formulation characterizes the accessibility transition through the
effective codimension of the localized low-loss region, quantified by the
statistical dimension of the corresponding polar cone. As discussed in
Remark~\ref{rem:larsen_angular_dimension}, this provides an equivalent
statistical-dimension refinement of the Gaussian-width perspective of
Larsen et al.~\cite{larsen2022degrees}. Because this geometric quantity is
generally intractable for neural-network losses, our main theoretical
development focuses on the local quadratic regime. In this setting, we
derive an orientation-resolved master formula that depends jointly on the
curvature spectrum and the reference-to-solution displacement profile. A
new self-consistent isotropic-orientation predictor follows as one
specialization, while the conservative orientation-uniform specialization
recovers the earlier radius-only quadratic bound. These results motivate the
practical three-mode dimension-selection procedure in
Algorithm~\ref{alg:hessian_dim_selection}. The general conic intersection
result for compact convex target sets is stated and proved in
Appendix~\ref{append:conicphase}.

The theory is stated for a positive-semidefinite Hessian at a local minimum.
Away from such a point, we use either the positive spectral part of the
Hessian or a positive-semidefinite curvature surrogate, such as the
Gauss--Newton or Fisher matrix. These substitutions are practical
approximations rather than consequences of
Theorem~\ref{thm:hessian_effective_rank}.

\begin{theorem}[Conic transition for a localized quadratic target]
\label{thm:hessian_effective_rank}
Let $\theta^\dagger$ be a local minimizer with local quadratic model
\[
\widetilde{\mathcal{L}}(\theta)
=
\mathcal{L}(\theta^\dagger)
+
\tfrac{1}{2}
(\theta-\theta^\dagger)^\top
H
(\theta-\theta^\dagger),
\qquad
H\succeq0,
\]
and, for $\varepsilon>0$ and localization radius
$R_{\mathrm{loc}}>0$, define
\[
S
=
\Bigl\{
\theta:
\tfrac{1}{2}
(\theta-\theta^\dagger)^\top
H
(\theta-\theta^\dagger)
\le \varepsilon
\Bigr\}
\cap
\overline{B}(\theta_0,R_{\mathrm{loc}}),
\]
assumed nonempty, with reference point $\theta_0\notin S$. Let
\[
C
=
\Bigl\{
t(\theta-\theta_0):
t\ge0,\;
\theta\in S
\Bigr\}.
\]
Let $E$ be a uniformly random $d$-dimensional subspace of
$\mathbb{R}^P$, and let $A\in\mathbb{R}^{P\times d}$ be any matrix with
$\operatorname{range}(A)=E$, so that the random slice is
\[
\theta_0+E
=
\bigl\{
\theta(z):=\theta_0+Az:
z\in\mathbb{R}^d
\bigr\},
\]
where $z\in\mathbb{R}^d$ is the trainable latent vector.
Equivalently, $A$ may be taken to have i.i.d.
$\mathcal{N}(0,1)$ entries, in which case $A$ has rank $d$ almost surely
and $E=\operatorname{range}(A)$ is uniformly distributed; the
intersection event depends only on $E$, not on the particular basis $A$.

The slice reaches $S$ if and only if
$E\cap C\neq\{0\}$. Moreover, for every $\eta\in(0,1)$, the intersection
probability is at least $1-\eta$ when
\[
d
\ge
P-\delta(C)+a_\eta\sqrt{P},
\]
and at most $\eta$ when
\[
d
\le
P-\delta(C)-a_\eta\sqrt{P},
\]
where
\[
a_\eta
=
\sqrt{8\log(4/\eta)}.
\]
Thus, the conic transition is centered at
\[
d_{\mathrm{conic}}
:=
P-\delta(C)
=
\delta(C^\circ),
\]
within an $O(\sqrt{P})$ transition window
(Theorem~\ref{thm:random_affine_slice_conic}, proved in
Appendix~\ref{append:conicphase}).
\end{theorem}

\paragraph{Implications of the theorem.}
Figure~\ref{fig:raman_arch} summarizes the progression from the geometric
phase-transition result to its quadratic predictors and scalable RaMaN
implementation. The identity
\[
d_{\mathrm{conic}}
=
\delta(C^\circ)
\]
gives the key geometric interpretation: the transition is governed by the
\emph{effective codimension} of the low-loss set as seen from $\theta_0$,
namely the effective number of constrained directions that the random slice
must satisfy, rather than directly by the ambient parameter dimension.
Small-dimensional accessibility is therefore predicted only in the fat-cone
regime, $\delta(C)\approx P$, where $S$ is wide in most parameter directions
and narrow along relatively few directions.

Whether this fat-cone regime holds depends on the localized geometry of the
low-loss set relative to the reference point $\theta_0$. In the quadratic
model, such a regime is favored by a Hessian spectrum with a small number of
large eigenvalues and a bulk of near-zero eigenvalues, although the actual
transition also depends on the displacement orientation, loss tolerance, and
localization radius.

Under this spectral structure, the quadratic surrogates introduced below can
predict a transition at a latent dimension substantially smaller than $P$.
The orientation-resolved master formula incorporates both the Hessian
spectrum and the displacement profile, while the orientation-uniform
predictor $r_{\mathrm{eff}}(\varepsilon,R)$ provides a conservative
spectrum-based mean-level estimate when only a displacement radius is
available. This qualitative spectral profile is widely reported in empirical
studies of deep-network Hessians
\cite{sagun2017empirical,ghorbani2019investigation,papyan2020traces},
and our benchmark tests whether the associated quadratic surrogates predict
the measured random-slice transitions.

Thus, the conic theorem identifies the geometric quantity controlling
random-slice accessibility, but it does not by itself provide a practical
way to evaluate that quantity for a neural-network loss. The quadratic
predictors developed next address this operational problem.

\begin{remark}[Two extreme geometries]
\label{rem:two_geometries}
Two examples delimit the regimes.

\textup{(i) Round target.}
If $S$ is a ball of radius $\rho$ centered at distance $D>\rho$ from
$\theta_0$, then $C$ is a circular cone with half-angle
\[
\alpha
=
\arcsin(\rho/D).
\]
Its statistical dimension is approximately
\cite{amelunxen2014living}
\[
\delta(C)
\approx
P\sin^2\alpha
=
P\frac{\rho^2}{D^2},
\]
so
\[
d_{\mathrm{conic}}
\approx
P\left(1-\frac{\rho^2}{D^2}\right).
\]
Thus, random slices are essentially ineffective for small, well-rounded
targets: for a fixed point at distance $D$ from $\theta_0$, a Haar-random
$d$-dimensional slice through $\theta_0$ has expected squared distance
$D^2(1-d/P)$ from that point, giving the typical scale
\[
D\sqrt{\frac{P-d}{P}}.
\]

\textup{(ii) Rank-$k$ quadratic residual.}
If
\[
H
=
\lambda(I_k\oplus0_{P-k}),
\]
then the unlocalized quadratic residual depends only on the $k$ stiff
directions. Under the isotropic displacement model, the
least-squares/master-formula transition can be computed exactly and is
controlled by $k$ rather than by $P$; see
Appendix~\ref{subsec:block_verification}. This example illustrates how a
small effective curvature dimension can produce a low-dimensional
random-slice residual transition. For the localized conic theorem, however,
$\delta(C^\circ)$ also depends on the reference point, tolerance, and
localization radius.

Empirical Hessian studies of deep networks often exhibit the qualitative
local-curvature structure underlying case~(ii), namely a small number of
stiff directions and many nearly flat directions. Our experiments test
whether this local spectral structure yields the small quadratic
random-slice transition predicted by the Hessian-based surrogates.
\end{remark}

The conic transition center $d_{\mathrm{conic}}=\delta(C^\circ)$ is geometric
but generally not directly computable from standard neural-network
diagnostics. We therefore seek an operational quadratic surrogate that
retains information discarded by radius-only approximations, in particular
the orientation of the reference-to-solution displacement relative to the
curvature eigenspaces.

\begin{definition}[Orientation-resolved master-formula predictor]
\label{def:predictors}
Let
\[
H
=
\sum_{i=1}^{P}
\lambda_i v_i v_i^\top,
\qquad
\lambda_i\ge0,
\]
be the eigendecomposition of the positive-semidefinite curvature operator,
and let
\[
\Delta_0
:=
\theta^\dagger-\theta_0,
\qquad
\Delta_i
:=
\langle\Delta_0,v_i\rangle.
\]
Thus, $\Delta_i$ is the component of the reference-to-solution displacement
along the $i$th curvature eigenvector. Let
\[
r
:=
\operatorname{rank}(H),
\qquad
\varepsilon_{\mathrm q}
:=
2\varepsilon,
\]
where $\varepsilon$ is the loss-excess tolerance in the quadratic model and
$\varepsilon_{\mathrm q}$ is the corresponding quadratic-form tolerance.

For each candidate dimension $d\in(0,r)$, let $\kappa(d)>0$ denote the
unique solution of
\begin{equation}
d
=
\operatorname{tr}\!\left[
H\bigl(H+\kappa(d)I\bigr)^{-1}
\right]
=
\sum_{i=1}^{P}
\frac{\lambda_i}{\lambda_i+\kappa(d)}.
\label{eq:kappa_d_main}
\end{equation}
The \emph{orientation-resolved master-formula predictor}, a
deterministic-equivalent mean-level approximation to the expected minimum
quadratic residual at latent dimension $d$, is
\begin{equation}
\widehat{\rho}_{\mathrm{MF}}(d;\Delta_0)
:=
\sum_{i=1}^{P}
\frac{
\kappa(d)\lambda_i
}{
\lambda_i+\kappa(d)
}
\Delta_i^2.
\label{eq:master_predictor_main}
\end{equation}
The corresponding predicted critical dimension is
\begin{equation}
    \widehat d_{\mathrm{MF}}
    :=
    \min
    \left\{
        d\in\{1,\ldots,r-1\}:
        \widehat{\rho}_{\mathrm{MF}}(d;\Delta_0)
        \le
        \varepsilon_{\mathrm q}
    \right\}.
    \label{eq:d_master_main}
\end{equation}
If the reference point already satisfies
$\Delta_0^\top H\Delta_0\le\varepsilon_{\mathrm q}$, we set
$\widehat d_{\mathrm{MF}}=0$. If the set in
\eqref{eq:d_master_main} is empty, we set
$\widehat d_{\mathrm{MF}}=r$.
\end{definition}

The predictor in Definition~\ref{def:predictors} is orientation resolved:
two displacements having the same Euclidean norm can yield substantially
different predicted dimensions if their energy is distributed differently
across the curvature eigenvectors. Directions with large $\lambda_i$ are
more costly because displacement along them contributes more strongly to the
quadratic residual and must therefore be represented more accurately by the
random search subspace.

The master formula is a mean-level predictor rather than an exact
high-probability guarantee. Its derivation, given in
Appendix~\ref{append:reff_derivation}, is exact up to the
deterministic-equivalent approximation for the expected Gaussian-span
projector. Accordingly, $\widehat d_{\mathrm{MF}}$ estimates the quadratic
random-slice transition but is distinct from the rigorous conic transition
center
\[
d_{\mathrm{conic}}
=
\delta(C^\circ)
\]
of Theorem~\ref{thm:hessian_effective_rank}.

The orientation-resolved predictor in
Definition~\ref{def:predictors} requires the complete displacement profile
$\{\Delta_i\}_{i=1}^{P}$. When this directional information is unavailable, the
master formula yields two useful spectrum-based specializations.

\begin{definition}[Specialized Hessian predictors]
\label{def:specialized_predictors}

Adopt the notation of Definition~\ref{def:predictors}. Let $R>0$ denote a
displacement-radius bound satisfying
\[
\|\Delta_0\|_2
\le
R,
\]
and let
\[
r
:=
\operatorname{rank}(H).
\]
When the displacement norm is known exactly, we take
$R=\|\Delta_0\|_2$.

For $d\in(0,r)$, let $\kappa(d)>0$ be determined by
\eqref{eq:kappa_d_main}. All expressions involving $\kappa(d)$ at $d=0$
and $d=r$ are interpreted by continuity, corresponding respectively to
$\kappa\to+\infty$ and $\kappa\to0$.

\textup{(i) Isotropic-orientation predictor.}
Under the isotropic-orientation model with displacement norm $R$,
\[
\Delta_i^2
=
\frac{R^2}{P},
\qquad
i=1,\ldots,P,
\]
the orientation-resolved master formula reduces to
\[
\widehat\rho_{\mathrm{iso}}(d)
=
\frac{R^2}{P}
\sum_{i=1}^{P}
\frac{
\kappa(d)\lambda_i
}{
\lambda_i+\kappa(d)
}
=
\frac{R^2}{P}\,\kappa(d)d.
\]
The \emph{isotropic-orientation predictor}
$d^\star_{\mathrm{iso}}$ is determined by
\begin{equation}
d^\star_{\mathrm{iso}}
=
\sum_{i=1}^{P}
\frac{\lambda_i}{\lambda_i+\kappa^\star},
\qquad
\kappa^\star d^\star_{\mathrm{iso}}
=
\frac{2\varepsilon P}{R^2}.
\label{eq:d_iso_main}
\end{equation}

\textup{(ii) Orientation-uniform predictor.}
For every displacement satisfying
$\|\Delta_0\|_2\le R$, the master formula obeys
\[
\begin{aligned}
\widehat\rho_{\mathrm{MF}}(d;\Delta_0)
&=
\sum_{i=1}^{P}
\frac{
\kappa(d)\lambda_i
}{
\lambda_i+\kappa(d)
}
\Delta_i^2
\\
&\le
\kappa(d)
\sum_{i=1}^{P}\Delta_i^2
\\
&\le
\kappa(d)R^2.
\end{aligned}
\]

Imposing the sufficient condition
$\kappa(d)R^2\le\varepsilon_{\mathrm q}$ gives the
\emph{orientation-uniform predictor}
\begin{equation}
\begin{aligned}
d_{\mathrm{unif}}(\varepsilon,R)
&:=
\sum_{i=1}^{P}
\frac{
\lambda_i
}{
\lambda_i+\varepsilon_{\mathrm q}/R^2
}
\\
&=
\sum_{i=1}^{P}
\frac{
\lambda_iR^2
}{
\lambda_iR^2+\varepsilon_{\mathrm q}
}.
\end{aligned}
\label{eq:d_unif_main}
\end{equation}

We equivalently denote this effective-rank predictor by
\begin{equation}
    r_{\mathrm{eff}}(\varepsilon,R)
    :=
    d_{\mathrm{unif}}(\varepsilon,R).
    \label{eq:reff_main}
\end{equation}
\end{definition}

\begin{remark}[Recovery of the quadratic radius bound of Larsen et al.]
Larsen et al.~\cite{larsen2022degrees} derive, for a positive-definite
quadratic well, the approximation
\[
d_{\mathrm{local}}(\varepsilon,R)
\gtrsim
\sum_i
\frac{r_i^2}{R^2+r_i^2},
\qquad
r_i^2
=
\frac{2\varepsilon}{\lambda_i}.
\]
The corresponding complementary threshold expression is
\[
    P
    -
    \sum_i
    \frac{r_i^2}{R^2+r_i^2}
    =
    \sum_i
    \frac{\lambda_iR^2}
         {\lambda_iR^2+2\varepsilon}
    =
    r_{\mathrm{eff}}(\varepsilon,R),
\]
where $\varepsilon_{\mathrm q}=2\varepsilon$ in
Definition~\ref{def:predictors}. Thus, the orientation-uniform
specialization of the master formula recovers the earlier radius-only
quadratic expression of Larsen et al., while the orientation-resolved
predictor developed here extends it by retaining the full displacement
profile $\{\Delta_i\}_{i=1}^{P}$.
\end{remark}

The three predictors correspond to different levels of displacement
information. The master-formula predictor
$\widehat d_{\mathrm{MF}}$ uses the complete orientation profile
$\{\Delta_i\}_{i=1}^{P}$. The isotropic predictor $d^\star_{\mathrm{iso}}$ replaces
that profile with an equal-energy orientation model, whereas the
orientation-uniform predictor
$d_{\mathrm{unif}}=r_{\mathrm{eff}}$ removes the orientation dependence
through a uniform upper bound. Consequently, the isotropic-orientation
predictor is sharper when its equal-energy assumption is a reasonable
approximation, whereas the orientation-uniform predictor applies when only
a displacement-radius estimate is available.

The proof of Theorem~\ref{thm:hessian_effective_rank} applies the conic
kinematic formula to the cone
$C=\operatorname{cone}(S-\theta_0)$ associated with the localized
quadratic low-loss set $S$. The corresponding general result for compact
convex target sets is stated and proved in
Appendix~\ref{append:conicphase}. Because direct computation of
$\delta(C^\circ)$ is generally intractable for neural-network losses,
Definitions~\ref{def:predictors} and
\ref{def:specialized_predictors} provide operational quadratic-model
surrogates for the corresponding accessibility transition.

These results motivate the three-mode dimension-selection procedure
developed next. When a reliable displacement profile is available, we use
the orientation-resolved master-formula predictor; when only a displacement
radius is available, we use the orientation-uniform predictor
$d_{\mathrm{unif}}=r_{\mathrm{eff}}$; and when neither estimate is
reliable, we use nested adaptive latent-dimension expansion. The
isotropic-orientation predictor $d^\star_{\mathrm{iso}}$ is retained
primarily as a sharper diagnostic predictor when its equal-energy
orientation model is appropriate.

\subsection{Latent Dimension Selection}
\label{subsec:hessian_effective_rank_predictor}

For practical curvature estimation, all eigenvalues, eigenvectors, and
trace estimates are computed from the same positive-semidefinite curvature
operator. This operator is the Hessian when it is positive semidefinite,
its positive spectral part when explicitly estimated, or a standard
positive-semidefinite surrogate such as the generalized Gauss--Newton or
Fisher matrix. For notational simplicity, we continue to denote this
operator by $H$. The resulting quantities are practical surrogates for the
quadratic predictors of Definition~\ref{def:predictors}; when $H$ is
replaced by a positive-semidefinite surrogate, they should not be interpreted
as direct estimators of the rigorous conic quantity
$d_{\mathrm{conic}}=\delta(C^\circ)$.

Define the unresolved spectral mass by
\[
\widehat T_{\mathrm{tail}}
:=
\left[
\widehat{\operatorname{tr}}\,H
-
\sum_{i=1}^{k}\lambda_i
\right]_+,
\]
where $[x]_+:=\max\{x,0\}$. We model the $P-k$ uncomputed eigenvalues by
their common average
\begin{equation}
\widehat\lambda_{\mathrm{tail}}
:=
\frac{
\widehat T_{\mathrm{tail}}
}{
P-k
},
\qquad
k<P.
\label{eq:equal_tail_eigenvalue}
\end{equation}

When $k=P$, all tail terms below are defined to be zero. This
\emph{equal-tail spectral representation} is exact when the unresolved
eigenvalues are equal. Within the deterministic-equivalent master formula,
it also avoids the artificial saturation that can arise when all unresolved
spectral contributions are linearized at zero curvature.

\paragraph{Mode 1: orientation-resolved selection.}
Suppose that a pilot, related-task, or diagnostic solution
$\widehat\theta$ provides the displacement estimate
\[
\widehat\Delta_0
:=
\widehat\theta-\theta_{\mathrm{ref}}.
\]

Its estimated components in the leading curvature eigenspace are
\[
\widehat\Delta_i
:=
\langle \widehat\Delta_0,v_i\rangle,
\qquad
i=1,\ldots,k,
\]
and the remaining displacement energy is
\begin{equation}
\widehat R_{\mathrm{tail}}^2
:=
\left[
\|\widehat\Delta_0\|_2^2
-
\sum_{i=1}^{k}\widehat\Delta_i^2
\right]_+.
\label{eq:tail_displacement_energy}
\end{equation}

For each candidate dimension $d$, let $\widehat\kappa(d)>0$ solve
\begin{equation}
d
=
\sum_{i=1}^{k}
\frac{\lambda_i}
{\lambda_i+\widehat\kappa(d)}
+
(P-k)
\frac{\widehat\lambda_{\mathrm{tail}}}
{\widehat\lambda_{\mathrm{tail}}+\widehat\kappa(d)}.
\label{eq:practical_kappa_d}
\end{equation}

The corresponding practical master formula is
\begin{equation}
\begin{aligned}
\widehat\rho_{\mathrm{MF}}^{(k)}(d;\widehat\Delta_0)
&:=
\sum_{i=1}^{k}
\frac{
\widehat\kappa(d)\lambda_i
}{
\lambda_i+\widehat\kappa(d)
}
\widehat\Delta_i^2
\\
&\quad+
\frac{
\widehat\kappa(d)\widehat\lambda_{\mathrm{tail}}
}{
\widehat\lambda_{\mathrm{tail}}+\widehat\kappa(d)
}
\widehat R_{\mathrm{tail}}^2.
\end{aligned}
\label{eq:practical_master_formula}
\end{equation}

Define the estimated rank of the equal-tail curvature representation by
\begin{equation}
\widehat r
:=
\#\{i\le k:\lambda_i>0\}
+
(P-k)\,
\mathbf{1}\{\widehat\lambda_{\mathrm{tail}}>0\}.
\label{eq:estimated_rank}
\end{equation}

For interior dimensions
$d\in\{1,\ldots,\widehat r-1\}$,
$\widehat\kappa(d)>0$ is determined by
\eqref{eq:practical_kappa_d}. At the two endpoints, we define the
practical residual by continuity:
\[
\widehat\rho_{\mathrm{MF}}^{(k)}
(0;\widehat\Delta_0)
:=
\sum_{i=1}^{k}
\lambda_i\widehat\Delta_i^2
+
\widehat\lambda_{\mathrm{tail}}
\widehat R_{\mathrm{tail}}^2,
\qquad
\widehat\rho_{\mathrm{MF}}^{(k)}
(\widehat r;\widehat\Delta_0)
:=
0.
\]

The raw orientation-resolved estimate is then
\begin{equation}
\widetilde d_{\mathrm{MF}}
:=
\min
\left\{
d\in\{0,\ldots,\widehat r\}:
\widehat\rho_{\mathrm{MF}}^{(k)}
(d;\widehat\Delta_0)
\le
\varepsilon_{\mathrm q}
\right\}.
\label{eq:practical_d_master}
\end{equation}

Since
$\widehat\rho_{\mathrm{MF}}^{(k)}(d;\widehat\Delta_0)$ is nonincreasing
in $d$, the smallest qualifying interior dimension can be found
efficiently by integer bisection.

\paragraph{Mode 2: orientation-uniform radius selection.}
If the full displacement profile is unavailable but a displacement-radius
estimate or upper bound $R$ is available, we use the equal-tail
approximation to the orientation-uniform effective-rank predictor:
\begin{equation}
\begin{aligned}
\widehat r_{\mathrm{eff}}^{\mathrm{eq}}
(\varepsilon,R)
&:=
\sum_{i=1}^{k}
\frac{
\lambda_i R^2
}{
\lambda_i R^2+\varepsilon_{\mathrm q}
}
\\
&\quad+
(P-k)
\frac{
\widehat\lambda_{\mathrm{tail}}R^2
}{
\widehat\lambda_{\mathrm{tail}}R^2
+\varepsilon_{\mathrm q}
}.
\end{aligned}
\label{eq:reff_equal_tail}
\end{equation}

When $R$ is a genuine upper bound on
$\|\theta^\dagger-\theta_{\mathrm{ref}}\|_2$, this retains the
orientation-uniform interpretation of
Definition~\ref{def:specialized_predictors}. When $R$ is estimated
empirically, the resulting quantity should instead be regarded as a
radius-calibrated practical predictor. The equal-tail representation of the
uncomputed spectrum remains an additional approximation.

The compact three-mode selection procedure is summarized in
Algorithm~\ref{alg:hessian_dim_selection}. The orientation-resolved and
radius modes use matrix-free curvature estimates of the leading spectrum
and unresolved spectral mass, while the adaptive mode uses nested
latent-dimension expansion that preserves previously active random
directions and latent coordinates.

\begin{algorithm}[t]
\caption{Three-mode latent-dimension selection for RaMaN: compact version}
\label{alg:hessian_dim_selection}
\begin{algorithmic}[1]
\REQUIRE Reference parameters $\theta_{\mathrm{ref}}$, parameter dimension
$P$, loss tolerance $\varepsilon$, multiplicative safety factor
$\gamma\ge1$, additive calibration margin $b\ge0$, optional curvature
operator $H$, optional displacement estimate $\widehat\Delta_0$, optional
radius estimate $R$, and adaptive-expansion controls.

\STATE Set $\varepsilon_{\mathrm q}\leftarrow2\varepsilon$.

\IF{a reliable curvature operator $H$ and displacement estimate
$\widehat\Delta_0$ are available}
\STATE Estimate the leading curvature eigenpairs and unresolved
spectral tail.
\STATE Compute the orientation-resolved estimate
$\widetilde d_{\mathrm{MF}}$ using
\eqref{eq:tail_displacement_energy}--\eqref{eq:practical_d_master}.
\STATE Set
\[
d
\leftarrow
\min\left\{
P,
\left\lceil
\gamma\widetilde d_{\mathrm{MF}}+b
\right\rceil
\right\},
\qquad
\mathrm{mode}\leftarrow\mathrm{MF}.
\]

\ELSIF{a reliable curvature operator $H$ and radius estimate $R$ are
available}
\STATE Estimate the leading curvature eigenvalues and unresolved
spectral tail.
\STATE Compute
$\widehat r_{\mathrm{eff}}^{\mathrm{eq}}(\varepsilon,R)$ using
\eqref{eq:reff_equal_tail}.
\STATE Set
\[
d
\leftarrow
\min\left\{
P,
\left\lceil
\gamma
\widehat r_{\mathrm{eff}}^{\mathrm{eq}}
(\varepsilon,R)
+b
\right\rceil
\right\},
\qquad
\mathrm{mode}\leftarrow\mathrm{Radius}.
\]

\ELSE
\STATE Apply nested adaptive expansion to obtain $d$ and an
adaptive success/failure flag.
\STATE Set $\mathrm{mode}\leftarrow\mathrm{Adaptive}$.
\ENDIF

\RETURN Selected latent dimension $d$, selection mode, and, for the
adaptive mode, the success/failure flag.
\end{algorithmic}
\end{algorithm}

Algorithm~\ref{alg:hessian_dim_selection} separates prediction from
fallback. The orientation-resolved mode is preferred when a pilot or
related-task displacement provides a reliable estimate of the displacement
orientation relative to the curvature eigenspaces. The radius mode discards
this directional information and is therefore more conservative when $R$
is a valid displacement bound. The adaptive mode makes no quadratic-model
prediction and instead identifies a sufficient dimension through successive
training runs.

The multiplicative factor $\gamma$ and additive margin $b$ convert the raw
mean-level predictor into a practical operating dimension. They are
calibration parameters rather than consequences of the conic theorem or
high-probability guarantees for the master formula. The conic theorem
provides a rigorous transition window centered at
$d_{\mathrm{conic}}=\delta(C^\circ)$, but it does not control the
deterministic-equivalent approximation error, the spectral-tail
approximation, or the accuracy of a pilot displacement estimate.

Nestedness is essential in the adaptive mode. When the dimension is
expanded from $d$ to $d_{\mathrm{new}}$, all previously active random
directions and latent coordinates must remain unchanged. For a
seed-regenerated Gaussian map, this is achieved by treating the active
matrix as a column prefix of a fixed virtual Gaussian matrix generated
from the same seed. For SHM, the coordinate sets are likewise generated as
prefixes of one fixed seeded ordering:
\[
\Omega(d)
=
(\omega_1,\ldots,\omega_d),
\qquad
\Omega(d_{\mathrm{new}})
=
(\omega_1,\ldots,\omega_{d_{\mathrm{new}}}),
\qquad
d<d_{\mathrm{new}},
\]
where $\Omega(d)$ is a prefix of $\Omega(d_{\mathrm{new}})$, while the
sign operator and Hadamard transform are retained unchanged. These
constructions ensure
\[
\mathcal{R}_{d_{\mathrm{new}}}([z_d;0])
=
\mathcal{R}_d(z_d).
\]
Resampling the Gaussian columns or SHM coordinate ordering after an
expansion would change the underlying random subspace rather than enlarge
the existing one. For the current $\mathrm{HDI}$ SHM implementation, the
sign pattern is also kept fixed to preserve exact parameterization
continuity; its active signs can be absorbed into the latent coordinates
and therefore do not change the reachable SHM subspace.

For AdamW dimension sweeps and adaptive expansion, we initialize the
learning-rate schedule across latent dimensions using
\[
\alpha(d)
=
\alpha(d_0)\sqrt{\frac{d_0}{d}}.
\]
This scaling is motivated by the dimension dependence of the induced
parameter-space update under the Gaussian map. For AdamW, however,
adaptive moment normalization modifies the exact step scaling, so this rule
should be regarded as an optimization heuristic rather than as part of the
geometric dimension-selection theory.

\begin{remark}[Conservative versus equal-tail correction]
\label{rem:tail_correction}
The equal-tail correction in \eqref{eq:reff_equal_tail} avoids the
saturation that can occur when the unresolved spectral contribution is
linearized. If a conservative upper estimate is required, it may be
replaced by
\[
\widehat r_{\mathrm{tail}}^{\mathrm{upper}}
:=
\min\left\{
P-k,
\frac{R^2}{\varepsilon_{\mathrm q}}
\widehat T_{\mathrm{tail}}
\right\}.
\]
This follows from
\[
\frac{\lambda R^2}
{\lambda R^2+\varepsilon_{\mathrm q}}
\le
\frac{\lambda R^2}{\varepsilon_{\mathrm q}}.
\]
The upper correction is conservative but may substantially overestimate
the required dimension. The equal-tail correction is typically less
conservative, but its accuracy depends on how well the average unresolved
eigenvalue represents the spectral tail.
\end{remark}

\FloatBarrier
\subsection{Random Mapping Networks}
\label{subsec:raman}

As summarized in Figure~\ref{fig:raman_arch}, RaMaN instantiates the
selected latent dimensions using frozen matrix-free random maps. Given a
reference parameter vector $\theta_{\mathrm{ref}}$, it trains only the
low-dimensional latent vectors $\{z_l\}_{l=1}^{L}$.

\subsubsection{Scalable Random Map Instantiations}
\label{subsubsec:raman}

A dense random map from $\mathbb{R}^d$ to $\mathbb{R}^P$ requires
storing a frozen $P\times d$ matrix, which is prohibitive for modern
networks. RaMaN avoids this persistent storage bottleneck using either
seed-regenerated Gaussian maps or Structured Hadamard Mapping (SHM).
In the idealized i.i.d.\ Gaussian construction, a global
seed-regenerated Gaussian map realizes the uniformly random subspace
assumed in Theorem~\ref{thm:hessian_effective_rank}, while the default
layer-wise construction corresponds to the independent layer-wise setting
of Corollary~\ref{cor:layerwise} under its product-structure assumptions.
SHM provides a structured, memory-efficient alternative whose behavior is
evaluated empirically rather than identified with the Haar-random model of
the theorem.

We therefore consider two default instantiations: SHM RaMaN, which uses a
seeded random selection of Hadamard basis directions implemented through a
fast Walsh--Hadamard transform, and seed-regenerated Gaussian RaMaN, which
regenerates Gaussian projection blocks on the fly from fixed random seeds.

We now specify the layer-wise RaMaN parameterization. Let the target
network have parameter vector
\[
\theta
=
\bigl[
\theta^{(1)},\ldots,\theta^{(L)}
\bigr],
\qquad
\theta^{(l)}\in\mathbb{R}^{P_l},
\qquad
P=\sum_{l=1}^{L}P_l.
\]
RaMaN parameterizes the target weights as
\[
\theta(z)
=
\theta_{\mathrm{ref}}
+
\Delta\theta(z),
\]
where $\theta_{\mathrm{ref}}$ is either a random initialization for
training from scratch or a pretrained parameter vector for fine-tuning.
The update $\Delta\theta(z)$ is generated from low-dimensional trainable
latent variables rather than learned directly in $\mathbb{R}^P$.

In the layer-wise form used by default, each layer receives its own latent
vector
\[
z_l\in\mathbb{R}^{d_l},
\qquad
d_l\ll P_l,
\]
and the corresponding parameter update is
\[
\Delta\theta^{(l)}
=
\beta_l\mathcal{R}_l\bigl(\phi_l(z_l)\bigr),
\qquad
l=1,\ldots,L.
\]
Here $\mathcal{R}_l:\mathbb{R}^{d_l}\to\mathbb{R}^{P_l}$ is a frozen
random map, $\beta_l>0$ is a layer-wise update scale, and $\phi_l$ is an
optional low-dimensional nonlinear modulation. Unless otherwise stated,
we set $\beta_l=1$ for all layers. In the simplest and default
version,
\[
\phi_l(z_l)=z_l,
\]
yielding a linear random reparameterization. When nonlinear modulation is
used, we take
\[
\phi_l(z_l)
=
z_l+\alpha_l\psi_l(z_l),
\]
where $\psi_l$ is a small frozen nonlinear map satisfying
\[
\psi_l(0)=0,
\qquad
D\psi_l(0)=0.
\]
Consequently,
$D\phi_l(0)=I$, so the local Jacobian at the origin is determined by the
random map $\mathcal{R}_l$, while the nonlinear term acts as a controlled
higher-order perturbation. The parameter $\alpha_l$ controls the strength
of this modulation. In the nonlinear SHM ablation used in our experiments,
$\psi_l(z)=\tanh(z)-z$ and $\alpha_l=0.5$, giving
\[
\phi_l(z)
=
0.5z+0.5\tanh(z).
\]

\paragraph{Layer-wise latent-dimension selection.}
The latent dimensions $\{d_l\}_{l=1}^{L}$ may be selected by applying the
curvature-based modes of Algorithm~\ref{alg:hessian_dim_selection} either
to individual layer blocks or to the full parameter space. Let
\[
H_l
=
\sum_i
\lambda_i^{(l)}
v_i^{(l)}
\bigl(v_i^{(l)}\bigr)^\top
\]
denote the $l$th diagonal block of the positive-semidefinite curvature
operator, where $H_l\in\mathbb{R}^{P_l\times P_l}$. Treating the layers
independently is a block-diagonal curvature approximation because
cross-layer curvature terms are omitted.

Let $\varepsilon_l\ge0$ be a layer-wise loss budget satisfying
\[
\sum_{l=1}^{L}\varepsilon_l
\le
\varepsilon
\]
under this block-diagonal quadratic approximation, and define
$\varepsilon_{\mathrm q,l}:=2\varepsilon_l$. The layer-wise budgets
$\{\varepsilon_l\}$ are user-specified subject to
$\sum_l\varepsilon_l\le\varepsilon$. In the absence of reliable blockwise
information, our default implementation instead selects a global $d$ and
uses the proportional allocation in
\eqref{eq:proportional_layer_allocation}.

When a reliable pilot, related-task, or diagnostic displacement is
available, define
\[
\widehat\Delta_0^{(l)}
:=
\widehat\theta^{(l)}
-
\theta_{\mathrm{ref}}^{(l)}.
\]
We then apply the orientation-resolved mode of
Algorithm~\ref{alg:hessian_dim_selection} to the $l$th curvature block and
loss budget. Let $\widetilde d_{\mathrm{MF}}^{(l)}$ denote the resulting
raw master-formula estimate. The calibrated layer-wise dimension is
\begin{equation}
d_l
=
\min\left\{
P_l,
\left\lceil
\gamma\widetilde d_{\mathrm{MF}}^{(l)}
+
b
\right\rceil
\right\}.
\label{eq:layerwise_mf_dimension}
\end{equation}

When the displacement orientation is unavailable but a layer-wise
displacement-radius estimate or upper bound $R_l$ is available, we instead
apply the orientation-uniform radius mode of
Algorithm~\ref{alg:hessian_dim_selection}. Let
$\widehat r_{\mathrm{eff},l}^{\mathrm{eq}}(\varepsilon_l,R_l)$ denote the
corresponding equal-tail estimate. Then
\begin{equation}
d_l
=
\min\left\{
P_l,
\left\lceil
\gamma
\widehat r_{\mathrm{eff},l}^{\mathrm{eq}}
(\varepsilon_l,R_l)
+
b
\right\rceil
\right\}.
\label{eq:layerwise_radius_dimension}
\end{equation}

Thus, \eqref{eq:layerwise_mf_dimension} uses the estimated displacement
orientation, whereas \eqref{eq:layerwise_radius_dimension} discards this
directional information and is conservative when $R_l$ is a valid upper
bound. When $R_l$ is estimated empirically rather than known as an upper
bound, the resulting quantity should be regarded as a radius-calibrated
practical predictor.

Under the product-structure assumptions of
Corollary~\ref{cor:layerwise}, a separate theorem-level simultaneous
intersection statement can be obtained by assigning failure probability
$\eta/L$ to each layer and applying a union bound. This probabilistic
statement concerns the conic threshold conditions of the corollary and
should not be interpreted as a guarantee for the calibrated
master-formula dimensions in
\eqref{eq:layerwise_mf_dimension} or
\eqref{eq:layerwise_radius_dimension}.

If reliable blockwise curvature information is unavailable, we first use
Algorithm~\ref{alg:hessian_dim_selection} to select a global dimension
$d$. For the default layer-wise RaMaN parameterization, we then allocate
that total dimension approximately in proportion to the layer sizes:
\begin{equation}
d_l
\approx
d\frac{P_l}{P},
\qquad
\sum_{l=1}^{L}d_l=d,
\label{eq:proportional_layer_allocation}
\end{equation}
with integer rounding adjusted to preserve the total dimension. This
proportional allocation is a practical heuristic and does not, by itself,
inherit the layer-wise probabilistic statement of
Corollary~\ref{cor:layerwise}.

When neither reliable displacement nor radius estimates are available,
we use the adaptive mode of
Algorithm~\ref{alg:hessian_dim_selection}. In the layer-wise
implementation, expansion must be nested within every affected layer:
previously active random directions and latent coordinates are preserved,
and newly added latent coordinates are initialized to zero.

For general neural-network losses, where cross-layer curvature and
target-set interactions need not vanish, both the block-diagonal predictor
and the proportional allocation in
\eqref{eq:proportional_layer_allocation} are practical approximations whose
accuracy is evaluated empirically.

\paragraph{Structured Hadamard Mapping RaMaN.}
The first scalable instantiation uses a Hadamard-based structured mapping
(SHM), motivated by fast structured projection methods based on
Walsh--Hadamard transforms, including SRHT
\cite{ailon2009fast,tropp2011improved} and Fastfood
\cite{le2013fastfood}. Li et al.~\cite{li2018measuring} also used Fastfood
to realize scalable random-subspace experiments. Our SHM construction is
not itself an SRHT; rather, it uses a seeded random selection of Hadamard
basis directions together with a fast Walsh--Hadamard implementation.

For each layer $l$, let $\widetilde P_l$ be the smallest power of two
satisfying
\[
\widetilde P_l\ge P_l.
\]
Let
$H_{\widetilde P_l}\in
\mathbb{R}^{\widetilde P_l\times\widetilde P_l}$
be the normalized Walsh--Hadamard matrix, and let $D_l$ be a fixed random
diagonal Rademacher sign matrix of the same size.

Using a fixed layer-wise seed, sample once an ordered sequence
\[
\Omega_l^{\max}
=
\bigl(
\omega_{l,1},\ldots,\omega_{l,d_{l,\max}}
\bigr)
\]
of distinct indices from
$\{1,\ldots,\widetilde P_l\}$, where
\[
d_{l,\max}\le P_l
\]
is the maximum latent dimension allowed for layer $l$. Equivalently,
$\Omega_l^{\max}$ may be taken as the first $d_{l,\max}$ entries of a
seeded uniformly random permutation of
$\{1,\ldots,\widetilde P_l\}$. At active dimension $d_l$, define
\[
\Omega_l(d_l)
:=
\bigl(
\omega_{l,1},\ldots,\omega_{l,d_l}
\bigr).
\]
Thus, each $\Omega_l(d_l)$ has the distribution of a uniformly sampled
ordered $d_l$-tuple of distinct indices, while dimensions are nested:
\[
d_l<d_{l,\mathrm{new}}
\quad\Longrightarrow\quad
\Omega_l(d_l)
\text{ is a prefix of }
\Omega_l(d_{l,\mathrm{new}}).
\]

Define the coordinate-injection operator
\[
I_{\Omega_l(d_l)}:
\mathbb{R}^{d_l}
\to
\mathbb{R}^{\widetilde P_l}
\]
by
\[
\left[
I_{\Omega_l(d_l)}u
\right]_j
=
\begin{cases}
u_a,
&
j=\omega_{l,a}
\text{ for some }a\in\{1,\ldots,d_l\},
\\
0,
&
j\notin
\{\omega_{l,1},\ldots,\omega_{l,d_l}\}.
\end{cases}
\]
Because the coordinate sequence itself is sampled uniformly, no separate
permutation operator is required.

The SHM random map is defined by
\[
\mathcal{R}_{l,d_l}^{\mathrm{SHM}}(u)
=
\sqrt{\frac{\widetilde P_l}{P_l}}\,
\operatorname{crop}_{P_l}
\left(
H_{\widetilde P_l}
D_l
I_{\Omega_l(d_l)}u
\right),
\qquad
u\in\mathbb{R}^{d_l}.
\]
The factor
$\sqrt{\widetilde P_l/P_l}$ normalizes the Euclidean norm of each cropped
Hadamard basis column: before this scaling, a normalized Hadamard column
restricted to the first $P_l$ coordinates has squared norm
$P_l/\widetilde P_l$.

Because $D_l$ is diagonal,
\[
D_l I_{\Omega_l(d_l)}
=
I_{\Omega_l(d_l)}
D_{\Omega_l(d_l)},
\]
where
\[
D_{\Omega_l(d_l)}
:=
I_{\Omega_l(d_l)}^{*}
D_l
I_{\Omega_l(d_l)}
\]
is an invertible $d_l\times d_l$ diagonal sign matrix. Therefore,
\[
\begin{aligned}
\operatorname{range}
\left(
\operatorname{crop}_{P_l}
H_{\widetilde P_l}
D_l
I_{\Omega_l(d_l)}
\right)
&=
\operatorname{range}
\left(
\operatorname{crop}_{P_l}
H_{\widetilde P_l}
I_{\Omega_l(d_l)}
\right).
\end{aligned}
\]
Thus, for the default identity modulation, the randomness of the reachable
SHM search subspace is determined by the seeded selection
$\Omega_l(d_l)$ of Hadamard basis directions. We retain $D_l$ for implementation convenience, not as an additional source of subspace randomization.

The sign operator $D_l$, the Hadamard transform
$H_{\widetilde P_l}$, and the ordered coordinate sequence
$\Omega_l^{\max}$ are fixed throughout training. Consequently, for
$d_l<d_{l,\mathrm{new}}$,
\begin{equation}
\mathcal{R}_{l,d_{l,\mathrm{new}}}^{\mathrm{SHM}}
\bigl([u;0]\bigr)
=
\mathcal{R}_{l,d_l}^{\mathrm{SHM}}(u).
\label{eq:shm_nestedness}
\end{equation}
This prefix construction is required by adaptive dimension expansion:
resampling $\Omega_l$ after an expansion would replace the existing
structured search subspace rather than enlarge it. The sign pattern is
also retained unchanged to preserve exact parameterization continuity,
although, as shown above, its active signs can be absorbed into the latent
coordinates and do not alter the reachable SHM subspace.

The cropping operation removes the padded coordinates. The map is
evaluated from right to left without materializing either $D_l$ or
$H_{\widetilde P_l}$. The operator
$I_{\Omega_l(d_l)}$ is implemented as a sparse coordinate scatter,
$D_l$ as elementwise multiplication by a fixed or seed-regenerated
Rademacher sign vector, and $H_{\widetilde P_l}$ using an in-place fast
Walsh--Hadamard transform. Consequently, the map requires
$O(\widetilde P_l\log\widetilde P_l)$ computation and
$O(\widetilde P_l)$ transform working memory, rather than persistent
storage of a dense $\widetilde P_l\times d_l$ random matrix.

The backward pass applies the adjoints of the forward operators in reverse
order. Given
\[
g_l
=
\nabla_{\Delta\theta^{(l)}}\mathcal{L}
\in\mathbb{R}^{P_l},
\]
we first zero-pad $g_l$ to length $\widetilde P_l$ and then compute
\[
\nabla_{u_l}\mathcal{L}
=
\beta_l
\sqrt{\frac{\widetilde P_l}{P_l}}\,
I_{\Omega_l(d_l)}^{*}
D_l
H_{\widetilde P_l}
\operatorname{pad}_{\widetilde P_l}(g_l)
\in\mathbb{R}^{d_l}.
\]
Here,
\[
I_{\Omega_l(d_l)}^{*}:
\mathbb{R}^{\widetilde P_l}
\to
\mathbb{R}^{d_l}
\]
denotes the adjoint of the coordinate-injection operator
$I_{\Omega_l(d_l)}$. It gathers the coordinates indexed by
$\Omega_l(d_l)$; specifically, for any
$x\in\mathbb{R}^{\widetilde P_l}$,
\[
\left[
I_{\Omega_l(d_l)}^{*}x
\right]_a
=
x_{\omega_{l,a}},
\qquad
a=1,\ldots,d_l.
\]

In the real Euclidean setting,
$\operatorname{pad}_{\widetilde P_l}$ is the adjoint of
$\operatorname{crop}_{P_l}$, while
$D_l^{*}=D_l$ and
$H_{\widetilde P_l}^{*}=H_{\widetilde P_l}$ because $D_l$ is a real
diagonal sign operator and $H_{\widetilde P_l}$ is the normalized
symmetric Hadamard transform. Thus, both forward and backward SHM
operations use fast transforms rather than dense matrix multiplication.

\paragraph{Seed-regenerated Gaussian RaMaN.}
The second instantiation uses a Gaussian random map that is never stored.
For each layer, define
\[
\mathcal{R}_l^{\mathrm{SG}}(u)
=
A_lu,
\qquad
A_l\in\mathbb{R}^{P_l\times d_l},
\]
where
\[
(A_l)_{ij}
=
\frac{1}{\sqrt{P_l}}\xi_{ij},
\qquad
\xi_{ij}\sim\mathcal{N}(0,1).
\]
The normalization does not affect the Gaussian column space but controls
the scale of the induced parameter update.

Instead of materializing $A_l$, its entries are regenerated on demand
using a counter-based pseudorandom number generator with fixed layer seed
$s_l^{\mathrm{seed}}$. The counter indexing is chosen so that increasing
$d_l$ exposes a column prefix of one fixed virtual Gaussian matrix, thereby
preserving nestedness under adaptive dimension expansion.

For memory-efficient evaluation, the rows of $A_l$ are partitioned into
$B_l$ blocks:
\[
A_l
=
\begin{bmatrix}
A_l^{(1)}
\\
A_l^{(2)}
\\
\vdots
\\
A_l^{(B_l)}
\end{bmatrix},
\qquad
A_l^{(b)}\in\mathbb{R}^{m_b\times d_l},
\qquad
\sum_{b=1}^{B_l}m_b=P_l.
\]
During the forward pass, each row block $A_l^{(b)}$ is regenerated to
compute the corresponding output block $A_l^{(b)}u$. During the backward
pass, the output gradient is partitioned conformably as
\[
g_l
=
\begin{bmatrix}
g_l^{(1)}
\\
g_l^{(2)}
\\
\vdots
\\
g_l^{(B_l)}
\end{bmatrix},
\qquad
g_l^{(b)}\in\mathbb{R}^{m_b},
\]
and the same row blocks are regenerated to accumulate
\[
A_l^\top g_l
=
\sum_{b=1}^{B_l}
\left(A_l^{(b)}\right)^\top g_l^{(b)}.
\]

Consequently, for
$u_l=\phi_l(z_l)$ and
$\Delta\theta^{(l)}=\beta_l A_lu_l$, the latent gradient is
\[
\nabla_{z_l}\mathcal{L}
=
\beta_l
D\phi_l(z_l)^\top
A_l^\top g_l,
\qquad
g_l
:=
\nabla_{\Delta\theta^{(l)}}\mathcal{L}.
\]
For the identity modulation $\phi_l(z_l)=z_l$, this reduces to
\[
\nabla_{z_l}\mathcal{L}
=
\beta_l A_l^\top g_l.
\]

Thus, in the idealized random-map model, seed-regenerated Gaussian RaMaN
realizes the same i.i.d.\ Gaussian projection as a materialized dense map.
In implementation, the projection requires no persistent
$P_l\times d_l$ matrix: the persistent map state is the random seed, while
temporary block storage is controlled by the chosen block size.

\subsection{Training and Complexity}

\paragraph{Training objective.}
RaMaN trains only the latent variables
\[
    z
    =
    \{z_l\}_{l=1}^L .
\]
The target network is used in the standard forward pass with parameters
$\theta(z)$. Given the training dataset $\mathcal{D}$, the latent variables
are optimized according to
\[
    \min_{z}
    \mathcal{L}\bigl(f_{\theta(z)};\mathcal{D}\bigr),
\]
where $\mathcal{L}$ denotes the task-specific empirical loss evaluated on
$\mathcal{D}$.
Gradients are propagated through the frozen random maps $\mathcal{R}_l$:
\[
    \nabla_{z_l}\mathcal{L}
    =
    \beta_l D\phi_l(z_l)^\top
    \mathcal{R}_l^\top
    \nabla_{\Delta\theta^{(l)}}\mathcal{L}.
\]
When $\phi_l$ is the identity, this reduces to
\[
    \nabla_{z_l}\mathcal{L}
    =
    \beta_l
    \mathcal{R}_l^\top
    \nabla_{\Delta\theta^{(l)}}\mathcal{L}.
\]
Thus optimizer states are maintained only for the latent vectors
$\{z_l\}$ rather than for all $P$ target parameters.

\begin{remark}[Optimization on the slice]
In the quadratic model of Theorem~\ref{thm:hessian_effective_rank}, with
the linear slice $\theta(z)=\theta_0+Az$, optimization over $z$ reduces to
the convex least-squares problem
\[
    \min_z
    \frac{1}{2}
    \left\|
    H^{1/2}(Az-\Delta_0)
    \right\|_2^2,
    \qquad
    \Delta_0:=\theta^\dagger-\theta_0.
\]
Thus, whenever the slice intersects $S_\varepsilon$, the global
least-squares optimum attains loss within the prescribed tolerance. The
convergence rate of gradient descent is governed by the spectrum of
$A^\top H A\in\mathbb{R}^{d\times d}$.

For isotropic Gaussian $A$, one has
\[
    \mathbb{E}[A^\top H A]
    =
    \frac{\operatorname{tr}(H)}{P}I_d,
\]
suggesting that random reparameterization can also improve conditioning when
the projected curvature concentrates around its expectation. A quantitative
analysis of this effect is left to future work.
\end{remark}

\paragraph{Memory and computational complexity.}
Dense Mapping-Network implementations can require $O(dP)$ frozen mapping
parameters in addition to the target-network parameters. This becomes
prohibitive for large models because $P$ may range from millions to
billions. RaMaN removes this bottleneck. With SHM maps, the random
projection is represented by signs, coordinate subsets, and
Hadamard transforms, requiring $O(P)$ memory or only random seeds when the
random objects are regenerated. With seed-regenerated Gaussian maps, the
projection has the same distribution as a dense Gaussian map but requires
$O(d)$ trainable memory plus $O(1)$ seed storage. The optimizer-state memory
is reduced from $O(P)$ to $O(d)$, or from $O(P_l)$ to $O(d_l)$ per layer.

Seed regeneration removes the $O(P_l d_l)$ storage cost but does not remove
the corresponding dense arithmetic: each forward or backward application of
the Gaussian map requires $O(P_l d_l)$ operations, in addition to random-number
regeneration overhead. By contrast, the SHM requires
$O(\widetilde P_l\log\widetilde P_l)$ operations. Thus, the two RaMaN
instantiations provide different memory--computation tradeoffs: the Gaussian
variant exactly preserves the relevant random-subspace distribution in the
settings described above, whereas the SHM variant offers substantially
faster structured transforms.

Importantly, RaMaN does not claim to reduce the number of parameters used
by the target network during inference unless combined with orthogonal
compression techniques such as pruning, quantization, or low-rank
decomposition. Its primary advantages are reduced trainable dimension,
reduced optimizer-state memory, reduced checkpoint size, and removal of
the dense $P\times d$ frozen-mapping memory wall.

\subsection{RaMaN Algorithm}
\label{sec:raman_compact}

A compact version of the RaMaN training algorithm is shown in
Algorithm~\ref{alg:raman_compact}.

\begin{algorithm}[!thp]
\caption{RaMaN training: compact version}
\label{alg:raman_compact}
\begin{algorithmic}[1]
\REQUIRE Reference parameters $\theta_{\mathrm{ref}}$, loss $\mathcal{L}$,
training data $\mathcal{D}$, latent dimensions $\{d_l\}_{l=1}^L$, map type
$\mathrm{type}\in\{\mathrm{SHM},\mathrm{SeedGaussian}\}$.
\STATE For each layer $l$, initialize the trainable latent vector
$z_l=0\in\mathbb{R}^{d_l}$ and a frozen random map
$\mathcal{R}_l:\mathbb{R}^{d_l}\to\mathbb{R}^{P_l}$.
\REPEAT
    \STATE Generate layer-wise updates
    \[
        \Delta\theta^{(l)}
        =
        \beta_l \mathcal{R}_l(\phi_l(z_l)),
        \qquad l=1,\ldots,L .
    \]
    \STATE Set
    \[
        \theta^{(l)}
        =
        \theta_{\mathrm{ref}}^{(l)}
        +
        \Delta\theta^{(l)} .
    \]
    \STATE Evaluate $\mathcal{L}(f_{\theta(z)};\mathcal{D})$ and update only
    the latent variables $\{z_l\}_{l=1}^L$ by backpropagating through the
    frozen maps $\{\mathcal{R}_l\}_{l=1}^L$.
\UNTIL{convergence}
\RETURN Trained latents $\{z_l\}_{l=1}^{L}$ and the random seeds defining
the frozen maps; generate $\theta(z)$ on demand when needed.
\end{algorithmic}
\end{algorithm}


\section{Experimental Results}
\label{sec:experiments}

We evaluate the geometric theory, its quadratic predictors, and the resulting
Random Mapping Network (RaMaN) training procedure through twelve experiments.
Experiments~1--3 examine the quadratic random-slice model: the existence and
location of the phase transition, its dependence on displacement orientation,
and its transfer to curvature operators obtained from trained neural networks.
Experiment~4 tests whether end-to-end RaMaN optimization exhibits a comparable
transition on MLP and CNN models. Experiments~5--6 evaluate the frozen-map
families and the three-mode latent-dimension selection procedure.
Experiments~7--9 extend the end-to-end study to a vision transformer, a
pretrained language model, and a deep residual network. Experiments~10--11
examine the sensitivity of the measured training midpoint to the optimizer and
the loss-excess tolerance. Finally, Experiment~12 applies the quadratic-probe
pipeline to a ViT-scale curvature surrogate. Unless stated otherwise, $d$
denotes the global latent dimension; for a layer-wise parameterization, it
denotes the total latent dimension $d=\sum_{l=1}^{L}d_l$.

\subsection{Experimental Protocol}
\label{subsec:exp_protocol}

\paragraph{Quadratic hit criterion and empirical transition.}
For a positive-semidefinite curvature operator $H$, a
reference-to-solution displacement
$\Delta_0=\theta^\dagger-\theta_0$, and a random Gaussian map
$A\in\mathbb{R}^{P\times d}$, we evaluate the least-squares residual from
\eqref{eq:ls_reduction},
\[
    \rho(A)
    =
    \min_{z\in\mathbb{R}^{d}}
    (Az-\Delta_0)^\top H(Az-\Delta_0).
\]
A trial is successful when
$\rho(A)\le\varepsilon_{\mathrm q}$, where
$\varepsilon_{\mathrm q}=2\varepsilon$. For each candidate $d$, the
empirical success probability $\widehat p(d)$ is the fraction of random-map
trials satisfying this criterion. We denote by $d_p$ the interpolated
dimension at which $\widehat p(d)=p$; in particular, $d_{50}$ is the empirical
transition midpoint and $d_{90}-d_{10}$ is its empirical width.

\paragraph{Quadratic predictors.}
We compare the empirical transition with the three mean-level predictors
defined in Section~\ref{subsec:main_theory}: the orientation-resolved
master-formula predictor $\widehat d_{\mathrm{MF}}$, the
isotropic-orientation predictor $d^\star_{\mathrm{iso}}$, and the
orientation-uniform predictor
$d_{\mathrm{unif}}=r_{\mathrm{eff}}(\varepsilon,R)$. When only leading Ritz
eigenpairs and a stochastic trace estimate are available, we use the practical
estimates $\widetilde d_{\mathrm{MF}}$ and
$\widehat r_{\mathrm{eff}}^{\mathrm{eq}}$ from
Section~\ref{subsec:hessian_effective_rank_predictor}. 
These predictors are tested directly in Experiments~1--3. Experiment~12
extends the same calculation to a ViT-scale GGN--Ritz surrogate. Because its
measured midpoint and its practical predictors are computed from the same
Ritz-plus-tail approximation, Experiment~12 is an internal consistency test of
the surrogate rather than an independent validation of the true network
curvature or of the end-to-end training midpoint.

\paragraph{End-to-end success criterion.}
For RaMaN training, a run is counted as successful when its training loss
reaches the prescribed target $\mathcal{L}_{\mathrm{target}}$ used by the
dimension-sweep protocol. The empirical training-success probability and
$d_{50}$ are computed across random seeds. For the layer-wise map families,
$d$ always denotes $\sum_l d_l$. 
For the AdamW sweeps, the learning-rate search is initialized using
\[
    \alpha(d)
    =
    \alpha(d_0)\sqrt{\frac{d_0}{d}},
\]
which approximately normalizes the induced parameter-space step across
latent dimensions. This scaling is used as an initialization for per-$d$
tuning rather than as an optimizer-independent rule; Experiment~10 treats
SGD separately.

\paragraph{Role of the full-parameter reference.}
Experiments~1--3 are designed to evaluate the quadratic transition
predictors and do not constitute comparisons of RaMaN generalization
against full-parameter training. In the end-to-end experiments, the
full-parameter model primarily defines the training-loss target and provides
a descriptive accuracy reference. Unless a protocol-matched control is
explicitly stated, differences in test accuracy should therefore not be
interpreted as compute-matched generalization comparisons. The corresponding
$d_{50}$ values characterize the training transition under the specified
reference and optimization protocol.

\paragraph{Datasets, models, and roles.}
Table~\ref{tab:exp_models} summarizes the neural-network experiments. The
smallest MNIST model admits a complete eigendecomposition of the positive
spectral part of the Hessian. Larger curvature probes use the generalized
Gauss--Newton (GGN) operator, leading Ritz eigenpairs, and a stochastic trace
estimate. 
The ViT experiments include both end-to-end training sweeps and a separate
quadratic probe based on a GGN--Ritz surrogate. The pretrained language-model
experiment is an end-to-end fine-tuning study only. For \texttt{bert-tiny},
the embedding module is held fixed; consequently, we distinguish the full
parameter count $P$ from the reparameterized count $P_{\mathrm{rep}}$.

\begin{table}[t]

\caption{Neural-network settings used in the experiments. ``Exact $H_+$''
denotes a complete eigendecomposition of the positive spectral part of the
Hessian; ``GGN--Ritz'' denotes a matrix-free GGN approximation based on
leading Ritz eigenpairs and a stochastic trace estimate.}

\label{tab:exp_models}

\centering

\footnotesize

\setlength{\tabcolsep}{4pt}

\renewcommand{\arraystretch}{1.08}

\begin{tabular}{@{}llllr@{}}

\toprule

Model & Dataset & Input/setting & Experimental role & Parameter count \\

\midrule

TinyMLP & MNIST & $14\times14$ & Exact $H_+$; training & $P=6{,}634$ \\

SmallMLP & MNIST & $28\times28$ & GGN--Ritz probe & $P=109{,}386$ \\

SmallCNN & CIFAR-10 & $32\times32$ & GGN--Ritz; training & $P=19{,}466$ \\

ResNet-18 & CIFAR-10/100 & $32\times32$
& Training transition; generalization controls
& $P\approx11.2\times10^6$ \\

ViT-Tiny & CIFAR-10 & patch size 4
& Training transition; GGN--Ritz probe & $P=1{,}806{,}538$ \\

ViT-Tiny & CIFAR-100 & patch size 4
& Training transition; GGN--Ritz probe & $P=1{,}823{,}908$ \\

\texttt{bert-tiny} & SST-2 & Pretrained
& Fine-tuning transition
& $P=4{,}386{,}178$ \\
& & & & $P_{\mathrm{rep}}=413{,}314$ \\

\bottomrule

\end{tabular}

\end{table}

\paragraph{Data subsets and evaluation splits.}
Unless otherwise stated, all from-scratch image-classification experiments
use fixed subsets of 10,000 training examples, with the same subset used
for the full-parameter and RaMaN conditions. This protocol applies to the
MNIST and CIFAR experiments in Experiments~3--7, 9--12, except that
Experiments~10 and 11 reuse the saved runs from Experiment~4 and
Experiment~12 reuses the ViT-Tiny data configuration from Experiment~7.
For Experiment~8, we use 8,000 of the 67,349 SST-2 training examples and
evaluate on all 872 validation examples. Test-set sizes, data augmentation,
and the data subsets used to estimate Hessian or GGN curvature are specified
below for the applicable experiments.

\paragraph{Transition localization and reporting.}
An initial dimension grid brackets each transition, after which a denser local
grid estimates $d_{10}$, $d_{50}$, and $d_{90}$. Unless otherwise stated,
reported uncertainty bands and error bars summarize variation across the
evaluated random seeds. 

\subsection{Controlled Quadratic Phase Transitions}
\label{subsec:exp_synthetic}

\paragraph{Experiment 1: transition location.}
We construct five synthetic positive-semidefinite quadratic instances whose
spectra span low-rank, rapidly decaying, outlier-plus-bulk, and slowly
decaying regimes. In every case, we use an equal-energy displacement,
$\Delta_i^2=R^2/P$.
Figure~\ref{fig:transition} shows the empirical success curves, and
Table~\ref{tab:synthetic} compares their midpoints with
$d^\star_{\mathrm{iso}}$ and the orientation-uniform predictor
$d_{\mathrm{unif}}$.

\begin{figure}[t]
\centering
\includegraphics[width=\linewidth]{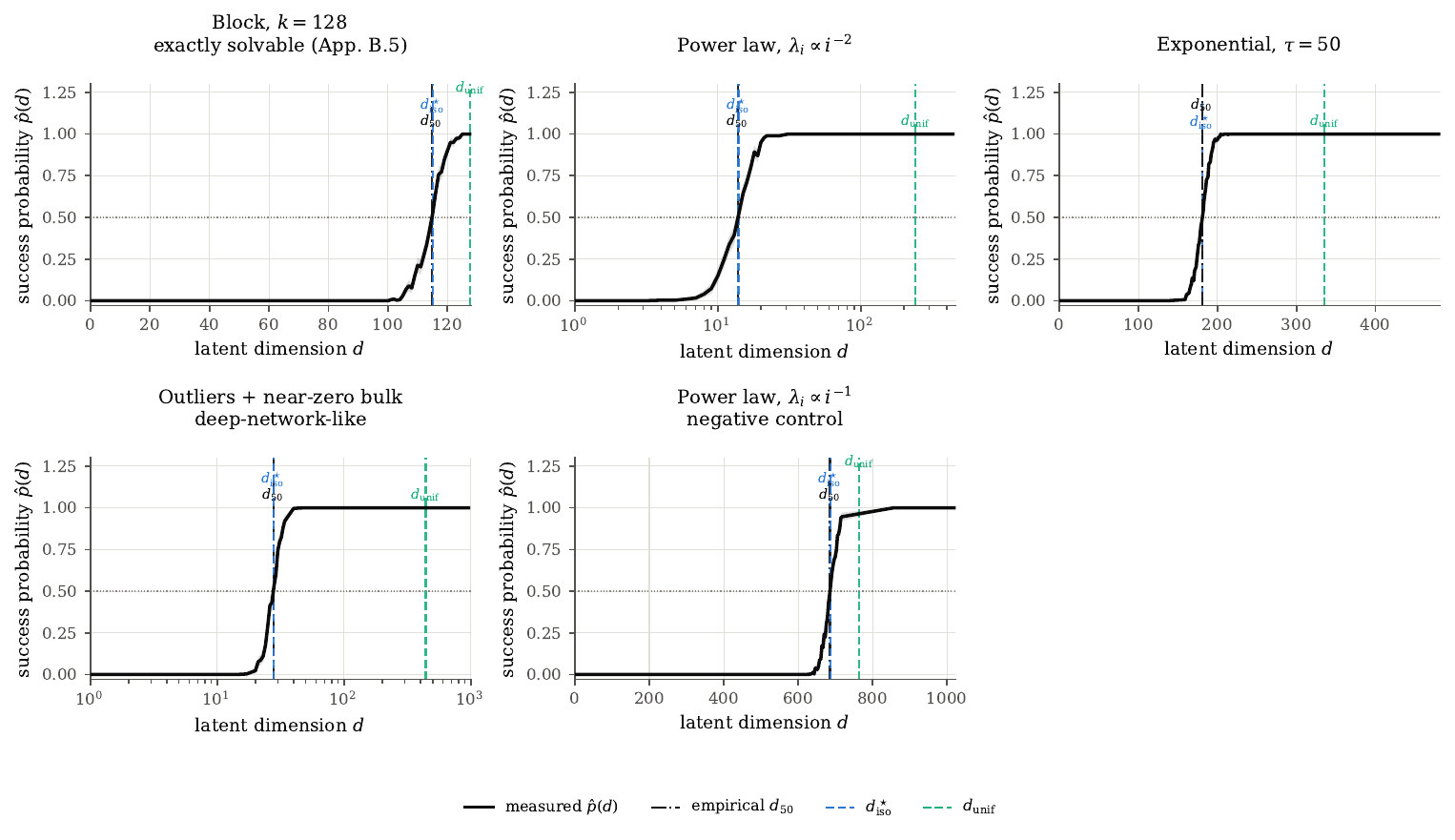}
\caption{\textbf{Controlled quadratic phase transitions.}
Empirical success probability $\widehat p(d)$ as a function of latent
dimension for five curvature spectra, using 300 independent random-slice
trials per evaluated dimension. The first four cases use $P=4{,}096$,
whereas the slowly decaying power-law negative control,
$\lambda_i\propto i^{-1}$, uses $P=1{,}024$.
Under the equal-energy displacement used here,
$d^\star_{\mathrm{iso}}$ closely matches the empirical midpoint $d_{50}$,
whereas $d_{\mathrm{unif}}$ is conservative to a spectrum-dependent
degree.}
\label{fig:transition}
\end{figure}

\begin{table}[t]
\caption{Empirical transition midpoints for five controlled synthetic
curvature spectra under the equal-energy displacement model. Here,
$r_{\mathrm{num}}
=\#\{i:\lambda_i>10^{-14}\lambda_{\max}\}$
is the numerical rank used by the residual solver, and
$d_{\mathrm{unif}}=r_{\mathrm{eff}}(\varepsilon,R)$.
The first four cases use $P=4{,}096$, whereas the slowly decaying
power-law control uses $P=1{,}024$.}
\label{tab:synthetic}
\centering
\footnotesize
\setlength{\tabcolsep}{4pt}
\renewcommand{\arraystretch}{1.08}
\begin{tabular}{@{}lrrrrrr@{}}
\toprule
Spectrum & $P$ & $r_{\mathrm{num}}$ & $d_{50}$ & $d^\star_{\mathrm{iso}}$
& $d_{\mathrm{unif}}$ & $d^\star_{\mathrm{iso}}/d_{50}$ \\
\midrule
Block, $k=128$                        & 4,096 & 128   & 114.8 & 115.2 & 127.6 & 1.004 \\
Power law, $\lambda_i\propto i^{-2}$ & 4,096 & 4,096 & 13.8  & 13.9  & 241.3 & 1.007 \\
Exponential, $\tau=50$                & 4,096 & 1,612 & 180.9 & 180.7 & 335.5 & 0.999 \\
Outliers with near-zero bulk          & 4,096 & 4,096 & 27.7  & 28.0  & 439.9 & 1.010 \\
Power law, $\lambda_i\propto i^{-1}$ & 1,024 & 1,024 & 685.8 & 686.8 & 763.6 & 1.002 \\
\bottomrule
\end{tabular}
\end{table}

Across the five spectra, whose empirical midpoints span almost a
$50\times$ range, the absolute relative error of
$d^\star_{\mathrm{iso}}$ is at most $1.1\%$ and has median $0.36\%$. The
orientation-uniform predictor does not underestimate $d_{50}$ in these tests,
but its ratio to $d_{50}$ ranges from approximately $1.1$ to $17.5$.

\paragraph{Transition width and localization diagnostic.}
To separate transition width from transition location, we repeat the
block-spectrum experiment at three ambient dimensions while holding the stiff
rank $k=128$ and the tolerance fraction fixed.

\begin{figure}[t]
\centering
\includegraphics[width=\linewidth]{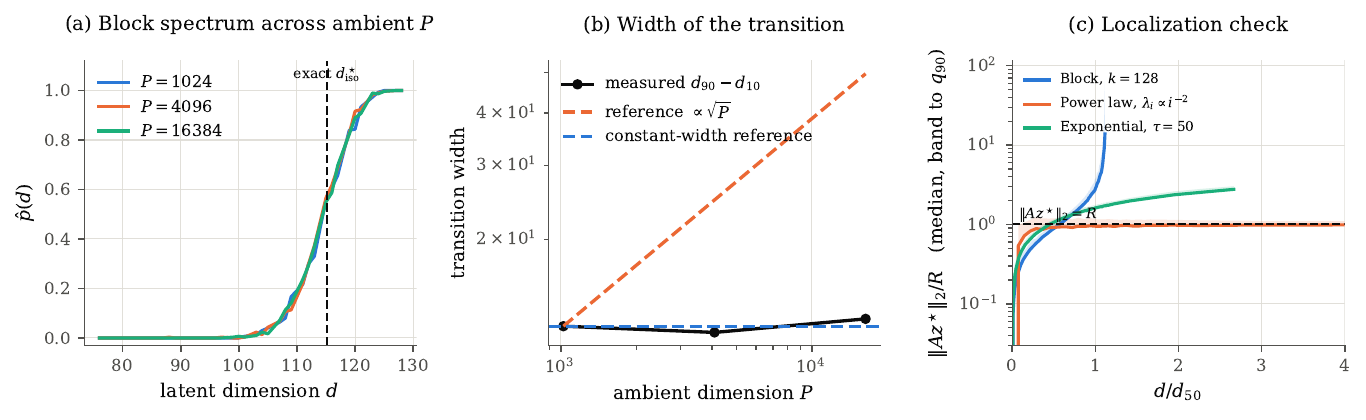}
\caption{\textbf{Scaling and localization diagnostics.}
Left: empirical success curves for the block spectrum at three ambient
dimensions. Middle: the measured width $d_{90}-d_{10}$ remains approximately
constant over the tested $16\times$ range of $P$; the plotted
$\sqrt P$ curve is a normalized reference rather than a prediction that the
measured width must follow. Right: norm of the unconstrained least-squares
minimizer relative to $\|\Delta_0\|_2$. The smooth spectra remain below
approximately three over the plotted range, whereas the block spectrum grows
much larger near its hard rank boundary.}
\label{fig:scaling}
\end{figure}

\begin{table}[t]
\caption{Transition width for the block spectrum as the ambient dimension
increases. The approximately constant observed width is narrower than the
general $O(\sqrt P)$ envelope in
Theorem~\ref{thm:hessian_effective_rank} and is compatible with the refined
fat-cone window in Remark~\ref{rem:sharp_window}.}
\label{tab:scaling}
\centering
\small
\setlength{\tabcolsep}{7pt}
\begin{tabular}{@{}rrrrr@{}}
\toprule
$P$ & $k$ & $d_{50}$ & $d^\star_{\mathrm{iso}}$
& $d_{90}-d_{10}$ \\
\midrule
1,024  & 128 & 114.5 & 115.2 & 12.4 \\
4,096  & 128 & 114.3 & 115.2 & 12.0 \\
16,384 & 128 & 114.4 & 115.2 & 12.9 \\
\bottomrule
\end{tabular}
\end{table}

The right panel of Figure~\ref{fig:scaling} is a diagnostic rather than a
validation of the localization assumption in
Theorem~\ref{thm:hessian_effective_rank}. In particular, the unconstrained
least-squares minimizer can lie well outside the ball of radius
$\|\Delta_0\|_2$, especially for a spectrum with a hard rank boundary.
Consequently, Experiments~1--3 test the quadratic residual predictors; a
direct test of the localized conic theorem would additionally need to enforce
or independently calibrate the search-radius constraint.


\subsection{Orientation-Resolved Prediction}
\label{subsec:exp_orientation}

\paragraph{Experiment 2: dependence on displacement orientation.}
We next vary the displacement direction while keeping its Euclidean norm
fixed. The interpolation begins at a baseline random direction constrained to
be orthogonal to the leading curvature eigenvector and ends at complete
alignment with that stiff direction. The baseline endpoint is therefore not
the equal-energy model used to define $d^\star_{\mathrm{iso}}$.
For each spectrum, the tolerance is calibrated once using the equal-energy
reference displacement and then held fixed throughout the orientation sweep,
so the curvature spectrum, displacement norm, and target sublevel set remain
unchanged as $\alpha$ varies. Figure~\ref{fig:orientation} shows how the
empirical transition midpoint and the corresponding quadratic predictors
change across the orientation sweep.

\begin{figure}[t]
\centering
\includegraphics[width=\linewidth]{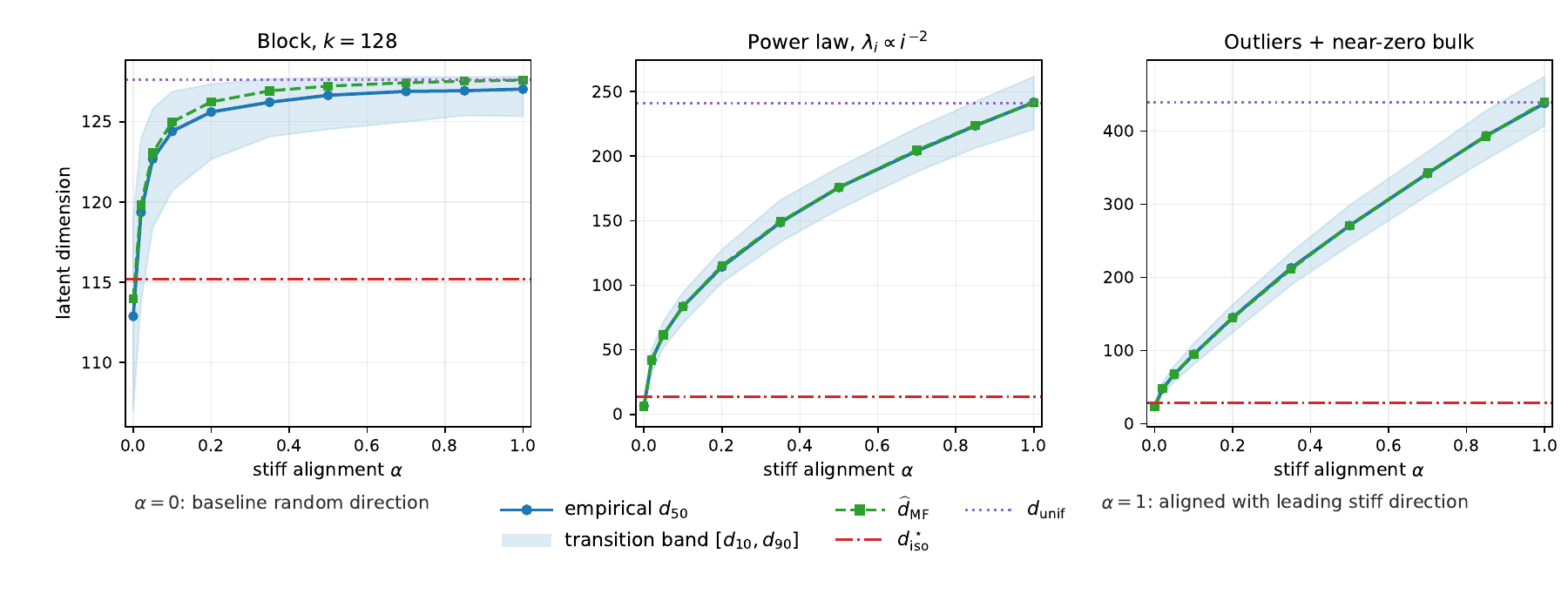}
\caption{\textbf{Orientation dependence of the transition.}
Empirical midpoint $d_{50}$ as the displacement changes from a baseline
random direction orthogonal to the leading curvature eigenvector
($\alpha=0$) to complete alignment with that stiff direction
($\alpha=1$). The shaded transition band spans
$[d_{10},d_{90}]$, and the orientation-resolved master-formula prediction
$\widehat d_{\mathrm{MF}}$ tracks the empirical midpoint throughout the
sweep. Horizontal references show $d^\star_{\mathrm{iso}}$ and
$d_{\mathrm{unif}}$; the former corresponds to the idealized equal-energy
model $\Delta_i^2=R^2/P$ and therefore need not coincide with the
$\alpha=0$ endpoint.}
\label{fig:orientation}
\end{figure}

\begin{table}[t]
\caption{Accuracy of the orientation-resolved master-formula predictor over
30 orientation--spectrum combinations.}
\label{tab:orientation}
\centering
\small
\begin{tabular}{@{}lr@{}}
\toprule
Quantity & Value \\
\midrule
Median absolute relative error & 0.45\% \\
Maximum absolute relative error & 5.36\% \\
Largest transition range induced by orientation & $37\times$ \\
\bottomrule
\end{tabular}
\end{table}

As shown in Figure~\ref{fig:orientation}, the measured transition changes
substantially as the displacement becomes increasingly aligned with the stiff
curvature direction, even though the curvature spectrum and
$R=\|\Delta_0\|_2$ remain fixed. Across the three spectra, orientation changes
the transition midpoint by as much as a factor of $37$.
Table~\ref{tab:orientation} summarizes the prediction accuracy over all
30 orientation--spectrum combinations: $\widehat d_{\mathrm{MF}}$ has a
median absolute relative error of $0.45\%$ and a maximum absolute relative
error of $5.36\%$. These results support retaining the full displacement
profile when it can be estimated reliably and explain why a spectrum-only
specialization can be either sharp or conservative depending on displacement
orientation.


\subsection{Quadratic Probes of Neural-Network Curvature}
\label{subsec:exp_neural_probe}

\paragraph{Experiment 3: transfer to trained-network curvature.}
We apply the same quadratic hit test to curvature operators obtained from
trained neural networks. The reference point is either the random
initialization or a short pilot solution, and the diagnostic displacement
points from that reference to the final full-parameter solution obtained
under the prescribed training schedule. To quantify anisotropy, let
$\mathcal I_{1\%}$ contain the largest-curvature one percent of
eigen-directions and define
\[
    m_{\mathrm{stiff}}
    :=
    \frac{
        \sum_{i\in\mathcal I_{1\%}}\lambda_i\Delta_i^2
    }{
        \sum_{i=1}^{P}\lambda_i\Delta_i^2
    }.
\]
Values near one indicate that most of the curvature-weighted displacement is
concentrated in a small set of stiff directions, whereas values near zero
indicate that the displacement places little curvature-weighted mass in those
directions and is instead distributed primarily over less-curved directions.
Intermediate values reflect progressively mixed orientation between these two
regimes.
Figure~\ref{fig:neuralprobe} summarizes the representative empirical success
curves together with cross-case predictor comparisons, and
Table~\ref{tab:neural_probe} reports the corresponding numerical transition
midpoints and anisotropy scores.

\begin{figure}[t]
\centering
\includegraphics[width=\linewidth]{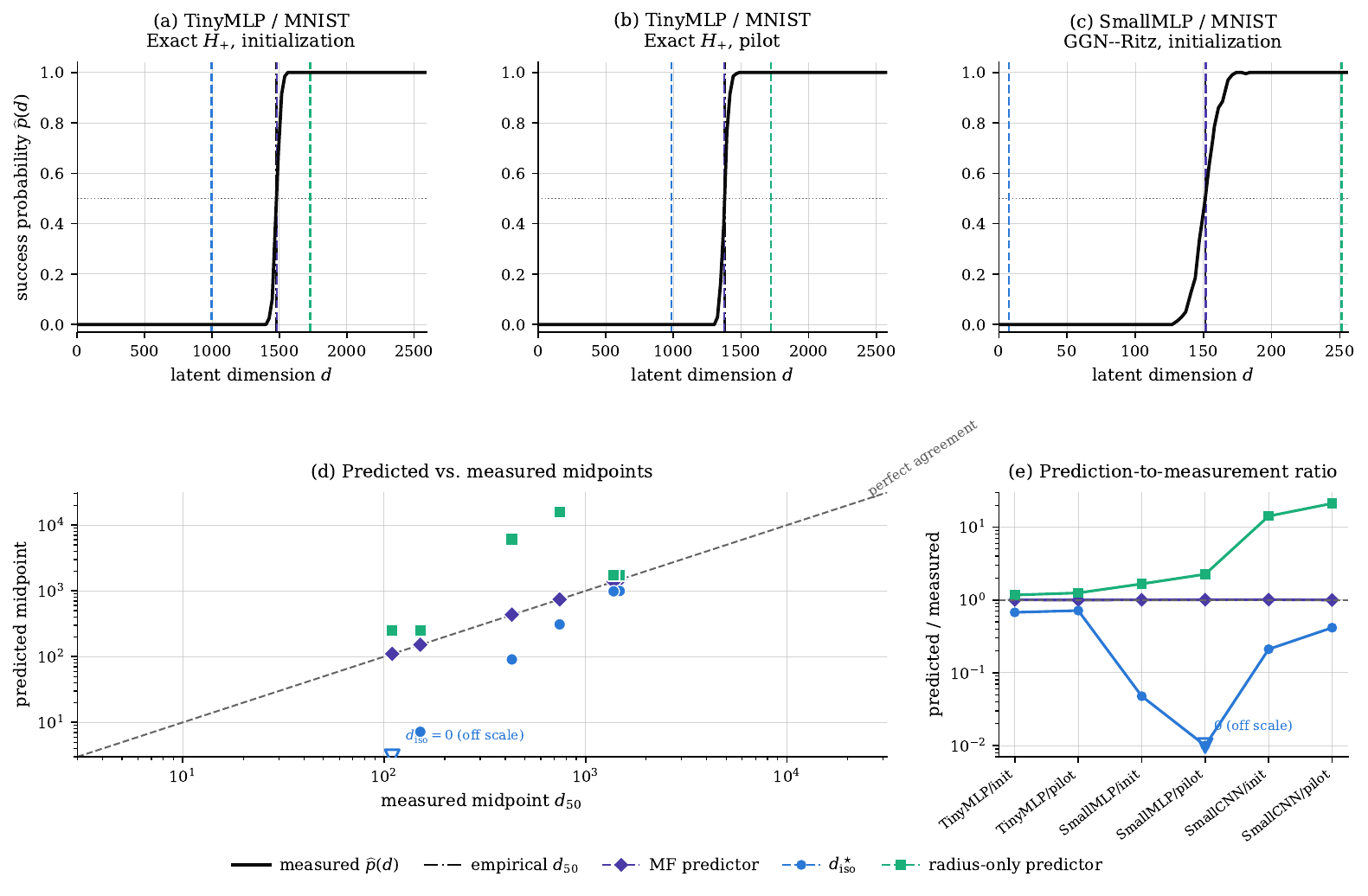}
\caption{\textbf{Quadratic probes of neural-network curvature.}
Top: representative empirical success curves for exact-Hessian and
GGN--Ritz curvature probes. Bottom left: predicted versus measured
transition midpoints across the six model/reference combinations.
Bottom right: prediction-to-measurement ratios, where unity denotes
perfect agreement. The orientation-resolved MF prediction remains close
to the measured midpoint, whereas the equal-energy specialization
$d^\star_{\mathrm{iso}}$ can substantially underestimate the transition
when the displacement is concentrated in stiff directions. The
radius-only predictor is conservative, particularly for the GGN--Ritz
cases.}
\label{fig:neuralprobe}
\end{figure}

\begin{table}[t]
\caption{Quadratic phase-transition probes on trained neural networks.
``MF prediction'' denotes $\widehat d_{\mathrm{MF}}$ when the full spectrum
is available and $\widetilde d_{\mathrm{MF}}$ for the matrix-free GGN cases.}
\label{tab:neural_probe}
\centering
\scriptsize
\resizebox{\linewidth}{!}{%
\begin{tabular}{@{}llrrrrrr@{}}
\toprule
Model / reference & Curvature & $P$ & $R$ & $d_{50}$ & MF prediction
& $d^\star_{\mathrm{iso}}$ & $m_{\mathrm{stiff}}$ \\
\midrule
TinyMLP / initialization & Exact $H_+$ & 6,634 & 9.540 & 1,478.3 & 1,480.4 & 998.1 & 0.105 \\
TinyMLP / pilot          & Exact $H_+$ & 6,634 & 3.300 & 1,378.6 & 1,374.9 & 985.1 & 0.305 \\
SmallMLP / initialization & GGN--Ritz & 109,386 & 16.040 & 151.3 & 151.6 & 7.2 & 1.000 \\
SmallMLP / pilot          & GGN--Ritz & 109,386 & 6.073 & 109.7 & 110.4 & 0.0 & 1.000 \\
SmallCNN / initialization & GGN--Ritz & 19,466 & 9.487 & 431.0 & 434.9 & 90.7 & 0.943 \\
SmallCNN / pilot          & GGN--Ritz & 19,466 & 4.862 & 743.9 & 742.1 & 309.3 & 0.854 \\
\bottomrule
\end{tabular}%
}
\end{table}

As shown by the representative success curves in
Figure~\ref{fig:neuralprobe}, the empirical transition remains sharply
defined after replacing synthetic quadratic spectra by curvature operators
extracted from trained neural networks. The lower panels of
Figure~\ref{fig:neuralprobe} further show that the orientation-resolved MF
prediction closely follows the measured midpoint across all six verified
cases, while $d^\star_{\mathrm{iso}}$ can substantially underestimate the
required dimension and the radius-only predictor is typically conservative.

Table~\ref{tab:neural_probe} quantifies these trends. The MF prediction
differs from the measured midpoint by at most $0.91\%$ across the six
verified cases. By contrast, $d^\star_{\mathrm{iso}}$ underestimates the
required dimension by as much as $21.0\times$ and returns zero in one case.
The observed range $m_{\mathrm{stiff}}\in[0.105,1]$ spans markedly different
orientation regimes, from displacements with relatively little
curvature-weighted mass in the stiffest directions to cases in which that
mass is almost entirely concentrated there. Thus, the agreement of the
orientation-resolved predictor is not confined to a narrow class of
displacement geometries, while the failures of the equal-energy
specialization become more pronounced when the displacement is strongly
aligned with stiff directions.

\paragraph{Predictor-only scale test.}
The same matrix-free pipeline is applied to ResNet-18. The model has
$P=11{,}173{,}962$ parameters. Empirical verification is not feasible on the
available 20~GB GPU because the diagnostic least-squares calculation scales
as $O(rd^2)$ time and $O(rd)$ memory per random trial. From a random
initialization, the reported predictor values are
$d^\star_{\mathrm{iso}}=8.6$ and
$\widetilde d_{\mathrm{MF}}=437$. We treat this run as a scale diagnostic,
not as a verified transition measurement.


\subsection{End-to-End RaMaN Training on MLP and CNN Models}
\label{subsec:exp_raman_training}

\paragraph{Experiment 4: optimization transition.}
We next replace exact quadratic minimization by gradient-based RaMaN
training. We compare a global seed-regenerated Gaussian map, which realizes
the uniformly random subspace used by the theory, with the default
layer-wise seed-regenerated Gaussian and Structured Hadamard Mapping (SHM)
parameterizations. Figure~\ref{fig:ramantrans} shows the resulting
training-success transitions and test-accuracy curves, while
Table~\ref{tab:raman_transition} summarizes the corresponding empirical
transition midpoints.

\begin{figure}[t]
\centering
\includegraphics[width=0.92\linewidth]{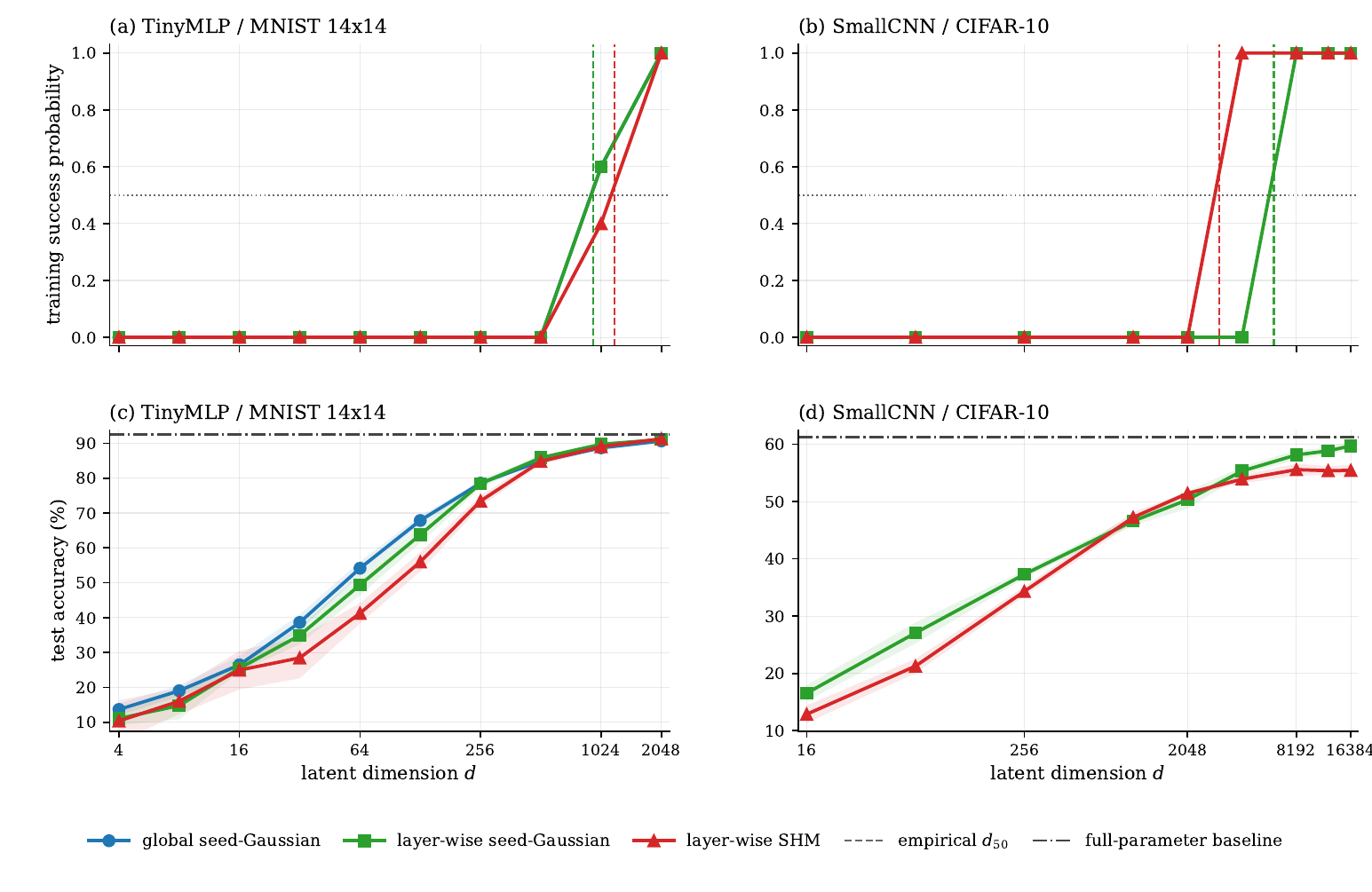}
\caption{\textbf{End-to-end RaMaN transition on MLP and CNN models.}
Top: training-success probability as a function of latent dimension.
Bottom: mean test accuracy, with the dash-dotted horizontal line denoting
the corresponding full-parameter baseline. Vertical dashed lines indicate
the empirical midpoint $d_{50}$ for each random-map construction. Both
models exhibit a pronounced optimization transition, and test accuracy
improves substantially once the transition region is crossed. For
layer-wise maps, $d=\sum_l d_l$.}
\label{fig:ramantrans}
\end{figure}

\begin{table}[t]
\caption{Empirical training-transition midpoint and predictive performance
by model and random-map construction. ``Best RaMaN accuracy'' denotes the
highest mean test accuracy observed over the evaluated latent dimensions.
The global Gaussian construction is omitted for SmallCNN because its direct
per-step cost scales as $O(dP)$.}
\label{tab:raman_transition}
\centering
\small
\setlength{\tabcolsep}{5pt}
\resizebox{\linewidth}{!}{%
\begin{tabular}{@{}llrrr@{}}
\toprule
Model & Random-map construction & $d_{50}$
& Best RaMaN accuracy & Full-parameter accuracy \\
\midrule
TinyMLP ($P=6{,}634$)
& Global seed-regenerated Gaussian
& 939 & 90.67\% & 92.55\% \\

& Layer-wise seed-regenerated Gaussian
& 939 & 91.09\% & 92.55\% \\

& Layer-wise SHM
& 1,195 & 91.27\% & 92.55\% \\

\addlinespace
SmallCNN ($P=19{,}466$)
& Layer-wise seed-regenerated Gaussian
& 6,144 & 59.71\% & 61.25\% \\

& Layer-wise SHM
& 3,072 & 55.56\% & 61.25\% \\
\bottomrule
\end{tabular}%
}
\end{table}

As shown in Figure~\ref{fig:ramantrans}, all evaluated RaMaN constructions
display a clear optimization transition: below a model-dependent latent
dimension, successful training is rare, whereas beyond the transition the
training-success probability rises rapidly toward one. The lower panels of
Figure~\ref{fig:ramantrans} further show that test accuracy improves in
parallel with this transition and approaches the full-parameter baseline
more closely as $d$ increases. Because the RaMaN sweeps and the full-parameter reference use different
training budgets and tuning protocols, these accuracy curves are intended
as descriptive references rather than as controlled generalization
comparisons; the primary quantity measured in this experiment is the
training-success transition.

Table~\ref{tab:raman_transition} quantifies both the transition location and
the predictive performance achieved over the evaluated dimension range.
For TinyMLP, the global seed-regenerated Gaussian map and the layer-wise
seed-regenerated Gaussian map have the same empirical midpoint
($d_{50}=939$), while layer-wise SHM requires a somewhat larger latent
dimension ($d_{50}=1{,}195$).
All three constructions nevertheless attain best observed test accuracies
within approximately two percentage points of the $92.55\%$ full-parameter
baseline. At $d_{50}$, however, the transition criterion only indicates
that the empirical probability of reaching the prescribed training-loss
target is $0.5$; test accuracy has not yet saturated and continues to
improve as the latent dimension $d$ increases beyond the transition.
For SmallCNN, the ordering reverses: layer-wise SHM transitions earlier
($d_{50}=3{,}072$) than the layer-wise seed-regenerated Gaussian construction
($d_{50}=6{,}144$), although the latter achieves the higher best observed
test accuracy over the evaluated dimension range ($59.71\%$ versus
$55.56\%$). Thus, transition location and eventual predictive performance
need not induce the same ordering across map families.

\subsection{Ablation Studies}
\label{subsec:exp_ablation}

Experiments~5--6 examine two practical design choices of RaMaN: the
frozen-map family and the strategy for selecting or expanding the latent
dimension.

\subsubsection{Map-Family Accuracy and Cost}
\label{subsubsec:map_ablation}

\paragraph{Experiment 5: frozen-map implementation.}
All map families are trained at a fixed dimension above the largest measured
transition midpoint for the corresponding model: $d=1{,}912$ for TinyMLP and
$d=9{,}831$ for SmallCNN. We report accuracy together with measured frozen-map
storage, optimizer-state memory, checkpoint size, peak GPU memory, and
wall-clock time. The full-parameter rows are baselines measured under the
ablation protocol and therefore need not coincide exactly with the references
in Experiment~4. Figure~\ref{fig:mapablation} summarizes the resulting
accuracy--systems-cost tradeoffs across map families, while
Table~\ref{tab:map_ablation} reports the corresponding numerical measurements.
For the nonlinear SHM ablation, we use
\[
    \phi_l(z)
    =
    z+0.5\bigl(\tanh(z)-z\bigr)
    =
    0.5z+0.5\tanh(z).
\]
This coordinate-wise map is bijective and therefore leaves the reachable
SHM subspace unchanged; the ablation tests the effect of nonlinear latent
modulation on optimization rather than an enlargement of the representable
parameter set. The full-parameter row is included as a descriptive reference rather than
as a compute-matched accuracy baseline; the primary comparisons in this
ablation are among the random-map implementations at the same latent
dimension.

\begin{figure}[t]
\centering
\includegraphics[width=\linewidth]{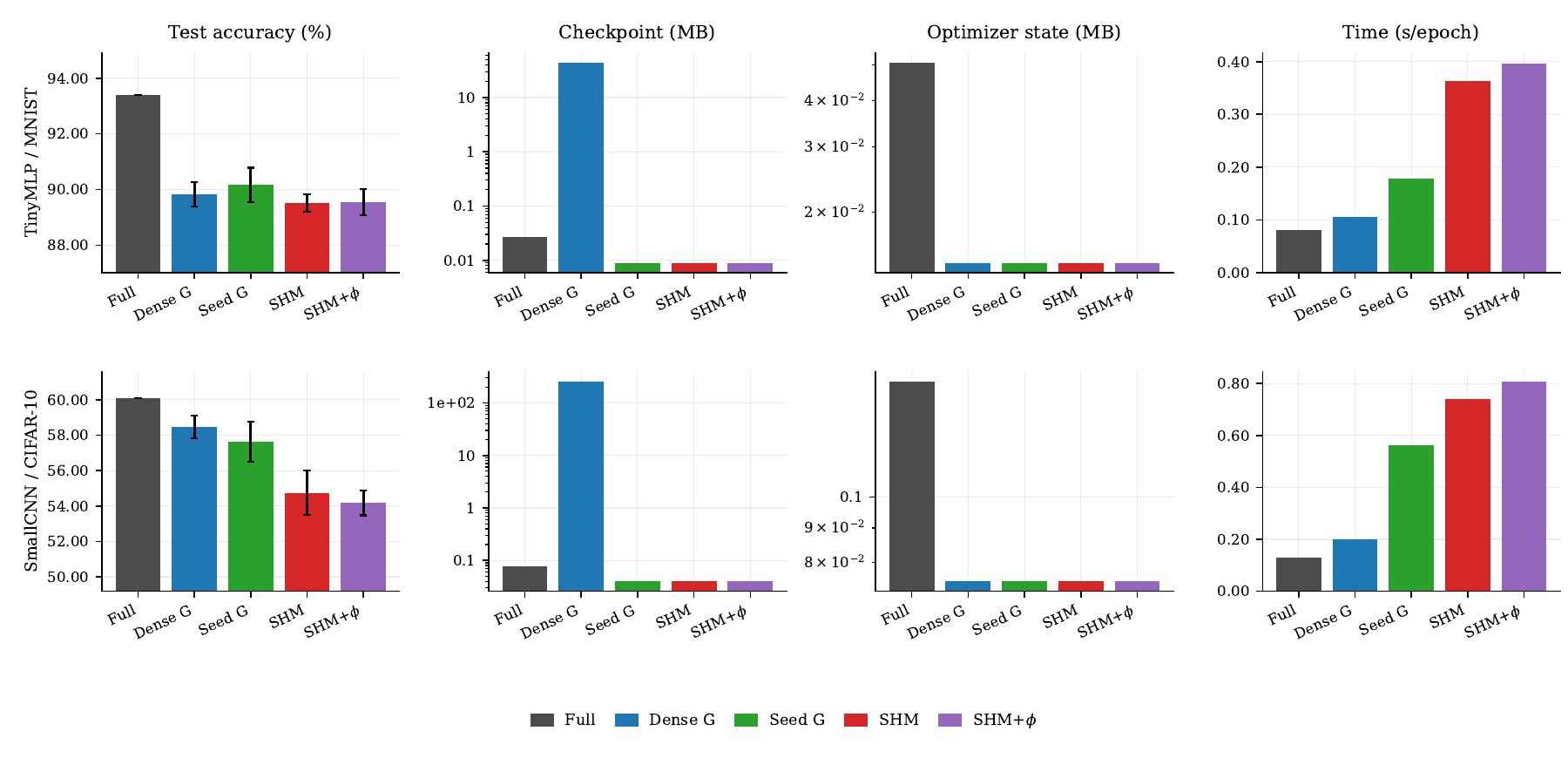}
\caption{\textbf{Map-family accuracy and systems cost.}
Test accuracy, checkpoint size, optimizer-state memory, and time per epoch for
full-parameter training and four random-map parameterizations. Matrix-free
maps eliminate dense frozen-map storage, while their accuracy and wall-clock
cost remain model dependent. Error bars denote one standard deviation over
the evaluated seeds.}
\label{fig:mapablation}
\end{figure}

\begin{table}[t]
\caption{Map-family ablation. Memory and checkpoint values are measured.
For random-map parameterizations, accuracy is reported as mean $\pm$ one
standard deviation over three seeds; each full-parameter baseline is a
single reference run and is therefore reported without an uncertainty
estimate. The full-parameter baseline directly optimizes all $P$ model
parameters without a random reparameterization map; all other rows optimize
only the latent $d$ variables.}
\label{tab:map_ablation}
\centering
\footnotesize
\setlength{\tabcolsep}{3.2pt}
\renewcommand{\arraystretch}{1.12}

\resizebox{\linewidth}{!}{%
\begin{tabular}{@{}llrrrrrrr@{}}
\toprule
Model &
Parameterization &
\shortstack{Trainable\\parameters} &
\shortstack{Frozen map\\(MB)} &
\shortstack{Optimizer\\state (MB)} &
\shortstack{Checkpoint\\(MB)} &
\shortstack{Peak GPU\\(MB)} &
\shortstack{Time\\(s/epoch)} &
\shortstack{Test accuracy\\(\%)} \\
\midrule

\shortstack[l]{TinyMLP\\MNIST}
& Full parameters
& 6,634 & 0.00 & 0.05 & 0.03 & 18 & 0.08
& $93.40$ \\

& Dense Gaussian
& 1,912 & 45.40 & 0.02 & 45.41 & 109 & 0.11
& $89.83\pm0.45$ \\

& Seed-regenerated Gaussian
& 1,912 & 0.00 & 0.02 & 0.01 & 137 & 0.18
& $90.17\pm0.62$ \\

& SHM
& 1,912 & 0.00 & 0.02 & 0.01 & 19 & 0.36
& $89.52\pm0.32$ \\

& SHM with nonlinear $\phi_l$
& 1,912 & 0.00 & 0.02 & 0.01 & 19 & 0.40
& $89.55\pm0.46$ \\

\addlinespace[3pt]

\shortstack[l]{SmallCNN\\CIFAR-10}
& Full parameters
& 19,466 & 0.00 & 0.16 & 0.08 & 47 & 0.13
& $60.10$ \\

& Dense Gaussian
& 9,831 & 267.69 & 0.08 & 267.73 & 405 & 0.20
& $58.48\pm0.65$ \\

& Seed-regenerated Gaussian
& 9,831 & 0.00 & 0.08 & 0.04 & 364 & 0.56
& $57.63\pm1.12$ \\

& SHM
& 9,831 & 0.00 & 0.08 & 0.04 & 91 & 0.74
& $54.73\pm1.25$ \\

& SHM with nonlinear $\phi_l$
& 9,831 & 0.00 & 0.08 & 0.04 & 91 & 0.81
& $54.18\pm0.71$ \\
\bottomrule
\end{tabular}%
}
\end{table}

Figure~\ref{fig:mapablation} illustrates the central implementation tradeoff:
matrix-free parameterizations remove the potentially dominant frozen-map and
checkpoint storage associated with a dense Gaussian map, but this storage
advantage does not necessarily imply lower runtime or identical predictive
performance. In particular, the accuracy and time costs of the different map
families remain model dependent.

Table~\ref{tab:map_ablation} quantifies these differences. Seed regeneration
differs from dense-Gaussian accuracy by only $0.34$ percentage points on
TinyMLP and $0.85$ points on SmallCNN while eliminating storage of the
$P\times d$ random map. Relative to the dense Gaussian implementation,
checkpoint size falls from $45.41$ to $0.01$~MB on TinyMLP and from
$267.73$ to $0.04$~MB on SmallCNN. SHM differs from the dense Gaussian
result by only $0.31$ points on TinyMLP but by $3.75$ points on SmallCNN,
again indicating model-dependent behavior. The optional nonlinear modulation yields no resolved accuracy improvement
over the corresponding SHM baseline at the evaluated seed counts. At these model sizes, however,
the storage reductions do not translate into lower wall-clock time because
map regeneration and per-layer transform overhead dominate.

\subsubsection{Sweep-Free Latent-Dimension Selection}
\label{subsubsec:dimension_selection_ablation}

\paragraph{Experiment 6: sweep-free dimension selection.}
We evaluate three strategies for obtaining a workable latent dimension
without performing a full independent sweep. The oracle strategy uses
curvature estimated at the final full-parameter solution together with the
diagnostic radius
$R=\|\theta^\dagger-\theta_{\mathrm{ref}}\|_2$, whereas the pilot strategy
uses the corresponding curvature and radius estimates from a short pilot
run. The adaptive strategy uses nested adaptive expansion and requires neither
quantity.
For the oracle and pilot strategies, we compare the conservative
linearized-tail correction of Remark~\ref{rem:tail_correction} with the
equal-tail approximation in \eqref{eq:reff_equal_tail}, which is used by
the practical default of Algorithm~\ref{alg:hessian_dim_selection}.

\begin{figure}[t]
\centering
\includegraphics[width=0.68\linewidth]{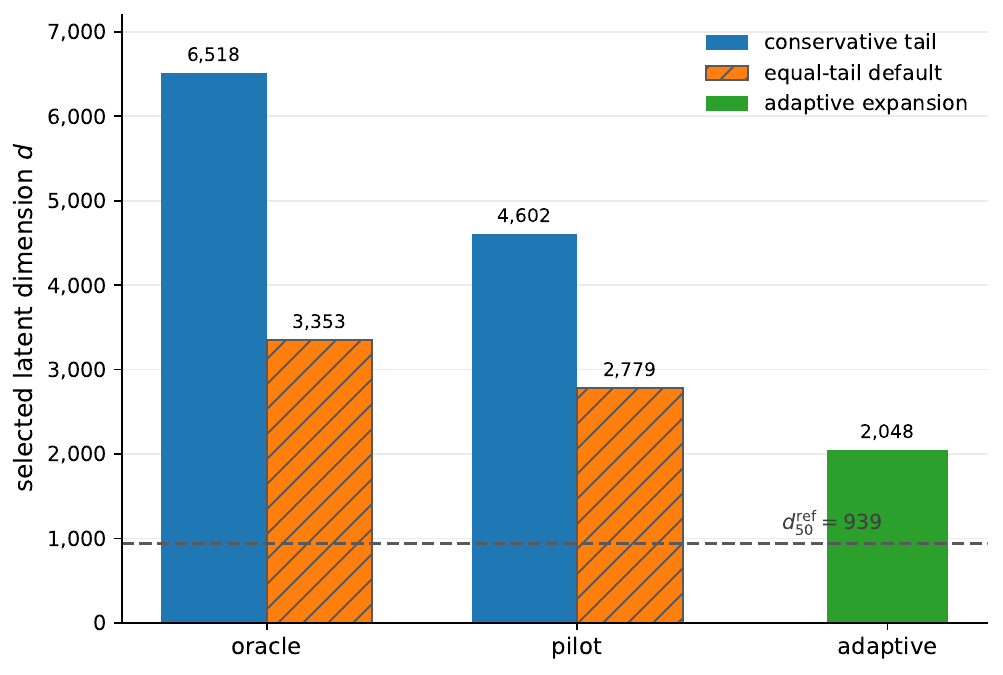}
\caption{\textbf{Sweep-free dimension selection on TinyMLP/MNIST.}
Selected latent dimensions for the conservative linear-tail rule, the
equal-tail default, and adaptive expansion. The horizontal dashed line marks
the cross-map Gaussian reference $d_{50}^{\mathrm{ref}}=939$ from
Experiment~4. All three selection experiments use global SHM maps; oracle
and pilot use curvature--radius estimates, whereas adaptive expansion
increases $d$ until the target-loss criterion is reached.}
\label{fig:dimselection}
\end{figure}

\begin{table}[t]
\caption{Sweep-free latent-dimension selection for TinyMLP/MNIST.
The oracle and pilot strategies report both the conservative linear-tail
selection and the equal-tail default. The reference
$d_{50}^{\mathrm{ref}}=939$ is the global seed-regenerated Gaussian midpoint
from Experiment~4, whereas the selection experiments use global SHM maps;
therefore, $d/d_{50}^{\mathrm{ref}}$ is a cross-map reference ratio rather
than a same-map overshoot factor. For oracle and pilot, verification is the
fraction of retraining seeds that reach the target loss. Adaptive expansion
is evaluated on its returned trajectory.}
\label{tab:dim_selection}
\centering
\small
\setlength{\tabcolsep}{4pt}
\resizebox{\linewidth}{!}{%
\begin{tabular}{@{}lrrrrr@{}}
\toprule
Strategy & $R$ & Conservative-tail $d$ & Default selected $d$
& $d/d_{50}^{\mathrm{ref}}$ & Verification \\
\midrule
Oracle radius      & 9.54 & 6,518 & 3,353 & 3.57 & 1.00 \\
Pilot radius       & 7.26 & 4,602 & 2,779 & 2.96 & 1.00 \\
Adaptive expansion & --   & --    & 2,048 & 2.18 & Reached \\
\bottomrule
\end{tabular}%
}
\end{table}

Figure~\ref{fig:dimselection} and
Table~\ref{tab:dim_selection} compare the three sweep-free strategies.
The conservative linear-tail correction remains highly conservative:
the oracle and pilot strategies select $d=6{,}518$ and $d=4{,}602$,
corresponding to approximately $98\%$ and $69\%$ of the full parameter
count $P=6{,}634$. Replacing this correction by the equal-tail default
reduces the selected dimensions to $3{,}353$ and $2{,}779$, respectively,
a reduction of approximately $40$--$49\%$ while retaining successful
retraining in all evaluated verification runs.

Adaptive expansion reaches the target at the smallest selected dimension,
$d=2{,}048$, without requiring either a displacement profile or an explicit
radius estimate. Its result, however, is obtained on the expansion trajectory
itself rather than from an independent retraining ensemble and should
therefore not be interpreted as a directly comparable success probability.

All three strategies therefore provide feasible alternatives to a full
dimension sweep, but with different tradeoffs. The equal-tail approximation
substantially reduces the conservatism of the linearized tail correction,
whereas adaptive expansion provides a curvature-free fallback when no
reliable curvature--radius estimate is available. Because the sweep
reference and the selection experiments use different random-map families,
the ratios in Table~\ref{tab:dim_selection} quantify relative scale only;
they do not measure same-map excess dimension above the true transition.

\subsection{Extensions to Larger and More Diverse Architectures}
\label{subsec:exp_extensions}

Experiments~7--9 test whether the sharp end-to-end training transition from
Experiment~4 persists for an attention-based model trained from scratch, a
pretrained language model adapted from a nonrandom reference point, and a deep
residual network.

\subsubsection{Vision Transformer Training from Scratch}
\label{subsubsec:exp_vit}

\paragraph{Experiment 7: ViT-Tiny on CIFAR-10 and CIFAR-100.}
We train a patch-based vision transformer with patch size 4, embedding
dimension 192, and depth 6. The model contains $P=1{,}806{,}538$ parameters
on CIFAR-10 and $P=1{,}823{,}908$ on CIFAR-100. Each sweep evaluates seven
latent dimensions, two layer-wise map families, and five random seeds.

\begin{figure}[t]
\centering
\includegraphics[width=0.68\linewidth]{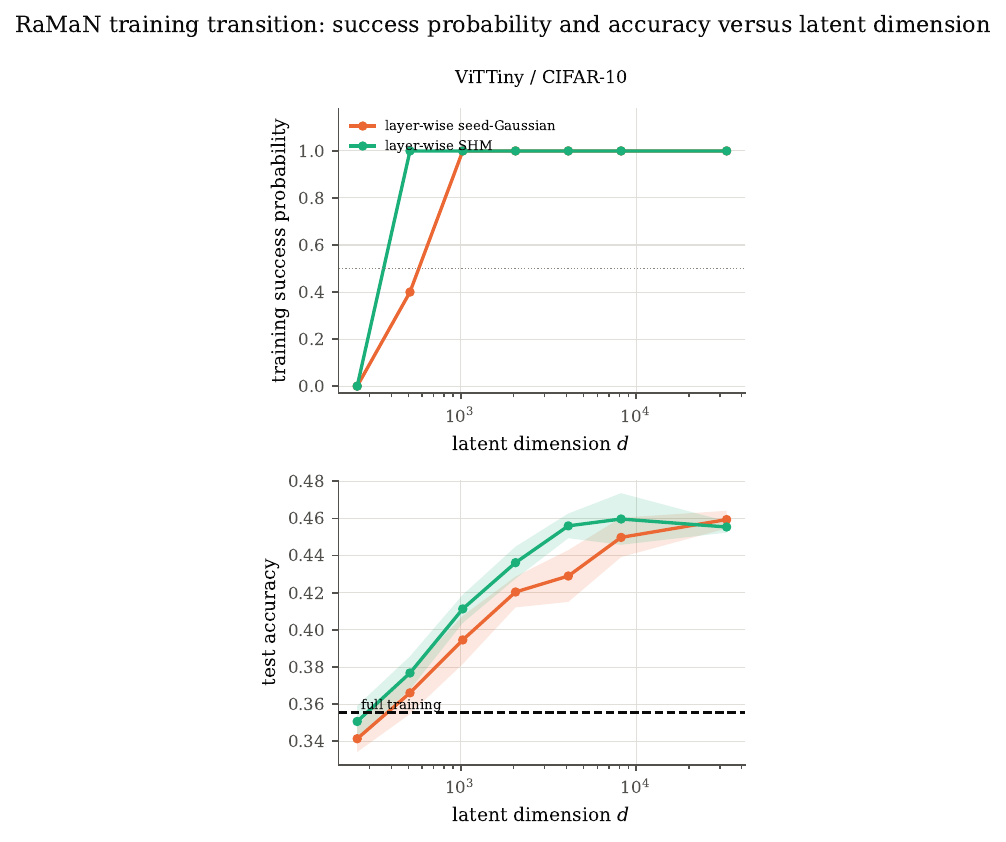}
\caption{\textbf{ViT-Tiny training transition on CIFAR-10.}
Empirical training-success probability (top) and test accuracy (bottom) as
functions of total latent dimension $d$ for layer-wise seed-regenerated
Gaussian and SHM maps. The dashed line is the full-parameter reference under
the reported baseline protocol.}
\label{fig:vit_cifar10}
\end{figure}

\begin{figure}[t]
\centering
\includegraphics[width=0.68\linewidth]{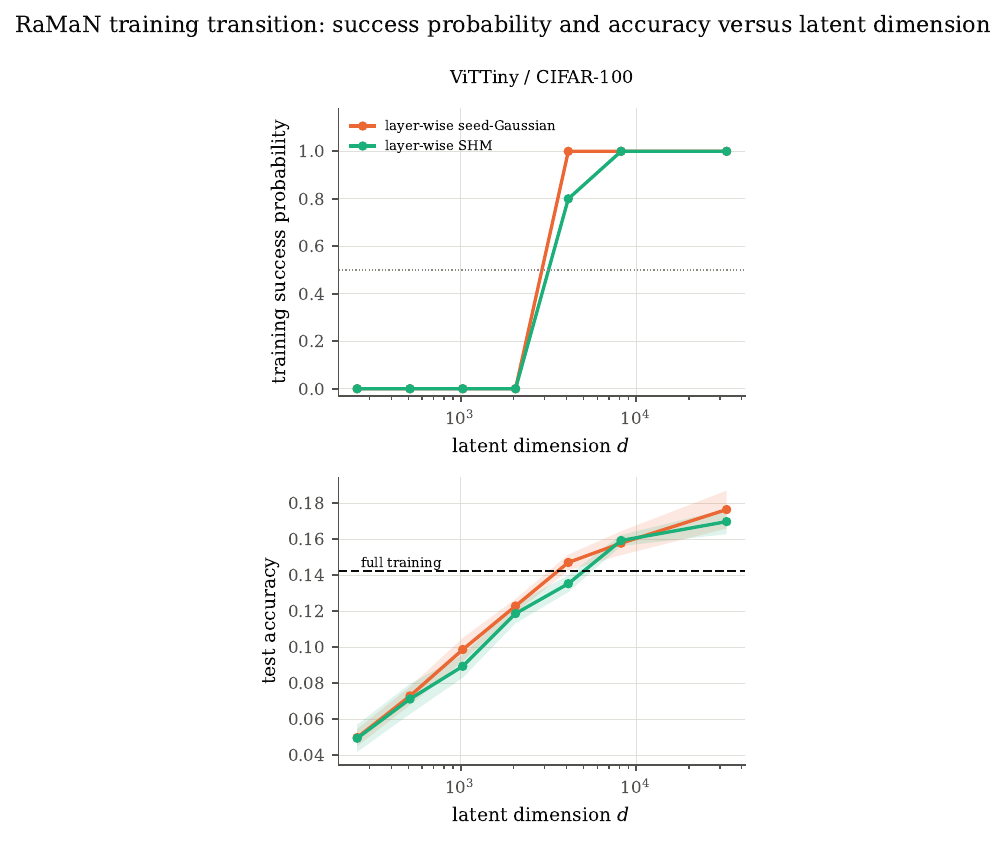}
\caption{\textbf{ViT-Tiny training transition on CIFAR-100.}
The same protocol as Figure~\ref{fig:vit_cifar10}. The transition occurs at a
larger latent dimension than on CIFAR-10 despite the nearly identical
parameter count, showing that the required dimension depends on the task and
training objective rather than on $P$ alone.}
\label{fig:vit_cifar100}
\end{figure}

\begin{table}[t]
\centering
\small
\setlength{\tabcolsep}{4pt}
\renewcommand{\arraystretch}{1.1}
\begin{tabular}{@{}llrrr@{}}
\toprule
Task &
Map construction &
$d_{50}$ &
$d_{50}/P$ &
\shortstack{Full-parameter\\accuracy (\%)} \\
\midrule

CIFAR-10
& Layer-wise seed-regenerated Gaussian
& 597
& $3.31\times10^{-4}$
& 35.55 \\

CIFAR-10
& Layer-wise SHM
& 384
& $2.13\times10^{-4}$
& 35.55 \\

\midrule

CIFAR-100
& Layer-wise seed-regenerated Gaussian
& 3,072
& $1.68\times10^{-3}$
& 14.25 \\

CIFAR-100
& Layer-wise SHM
& 3,328
& $1.82\times10^{-3}$
& 14.25 \\

\bottomrule
\end{tabular}
\caption{Empirical training-transition midpoints for ViT-Tiny.
The last column reports the full-parameter reference accuracy under the
baseline protocol used to define the target loss. The reported transition
values and reference accuracies are conditional on that protocol.}
\label{tab:vit_transition}
\end{table}

Both tasks exhibit a sharp transition, but the transition location changes
substantially with the dataset: the CIFAR-100 midpoint is approximately
$5.1\times$ higher than the CIFAR-10 midpoint for the Gaussian map and
$8.7\times$ higher for SHM. The map-family ordering also reverses across the
two tasks at fixed architecture. The relative midpoints are approximately
$2.1\times10^{-4}$--$1.8\times10^{-3}$ of $P$, much smaller than those
measured for TinyMLP and SmallCNN in Experiment~4.

The accuracy curves exceed the original full-parameter reference after the
transition, but this apparent advantage is sensitive to the reference
training protocol. Retraining the full-parameter model under a protocol
matched more closely to RaMaN raises the reference accuracy substantially
and removes the apparent RaMaN advantage. Because the original ViT sweeps
do not retain the per-seed losses needed to recompute the transition under
the stricter matched target, we therefore interpret Experiment~7 primarily
as evidence for the existence and approximate location of the training
transition under the reported protocol, rather than as an accuracy
comparison with full-parameter training.

\paragraph{Runtime at ViT scale.}
The reported production sweeps required 150.8--155.4 minutes for the
layer-wise seed-regenerated Gaussian map and 211.4--233.0 minutes for SHM on a
single V100-32GB GPU. At the very small ratios $d/P$ selected here, the fixed
per-layer Hadamard-transform overhead dominates, so SHM is slower despite its
more favorable scaling as $d$ approaches $P$.

\subsubsection{Pretrained Language-Model Fine-Tuning}
\label{subsubsec:exp_bert}

\paragraph{Experiment 8: \texttt{bert-tiny} on SST-2.}
We next set $\theta_{\mathrm{ref}}$ to a pretrained \texttt{bert-tiny}
checkpoint and fine-tune it on SST-2 through layer-wise random maps. The model
contains $P=4{,}386{,}178$ parameters. Because the word, position, and
token-type embedding tensors account for most of the model and are held fixed
at their pretrained values, only
$P_{\mathrm{rep}}=413{,}314$ parameters are reparameterized. Thus, the
relevant denominator for the reported latent fraction is $P_{\mathrm{rep}}$,
not the full model size.

\begin{figure}[t]
\centering
\includegraphics[width=0.68\linewidth]{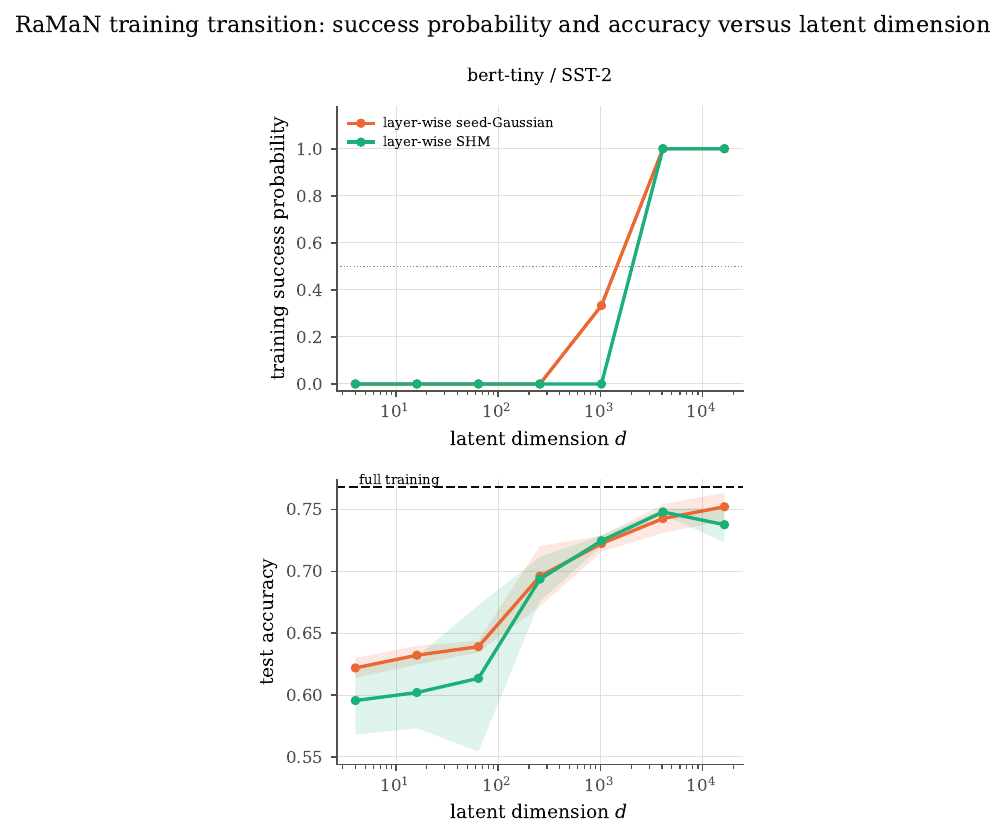}
\caption{\textbf{RaMaN fine-tuning transition for \texttt{bert-tiny} on
SST-2.}
Empirical training-success probability (top) and validation accuracy
(bottom) as functions of total latent dimension. Word, position, and token-type embeddings
remain fixed at their pretrained values. The dashed line is the full
fine-tuning reference.}
\label{fig:bert_sst2}
\end{figure}

\begin{table}[t]
\caption{Empirical fine-tuning midpoint for \texttt{bert-tiny}/SST-2.
The ratio is computed relative to the reparameterized budget
$P_{\mathrm{rep}}=413{,}314$. The full fine-tuning reference has validation
accuracy $76.83\%$.}
\label{tab:bert_transition}
\centering
\small
\setlength{\tabcolsep}{7pt}
\begin{tabular}{@{}lrrr@{}}
\toprule
Random-map construction & $d_{50}$ & $d_{50}/P_{\mathrm{rep}}$ & Full val.\ accuracy \\
\midrule
Layer-wise seed-regenerated Gaussian & 1,792 & 0.43\% & 76.83\% \\
Layer-wise SHM                       & 2,560 & 0.62\% & 76.83\% \\
\bottomrule
\end{tabular}
\end{table}

Both map families exhibit a sharp transition between the tested dimensions.
The empirical midpoint uses only $0.43$--$0.62\%$ of the reparameterized
budget, approximately $35$--$75\times$ smaller as a fraction than the
$14$--$32\%$ observed for TinyMLP and SmallCNN trained from scratch. This is
qualitatively consistent with prior evidence that pretrained models can be
adapted in very low-dimensional random subspaces
\cite{aghajanyan2021intrinsic,li2018measuring}. The comparison is not a
controlled magnitude comparison, however: the model, dataset, initialization,
training budget, frozen embedding tensors, and success criterion differ from
the earlier experiments. Because only one small pretrained language model and
one task are evaluated, these results should not be generalized to large
language models without additional experiments.

\subsubsection{Deep Residual Network Training}
\label{subsubsec:resnet}

\paragraph{Experiment 9: ResNet-18 on CIFAR-10 and CIFAR-100.}
We next examine whether the end-to-end training transition persists for a
substantially deeper architecture. The GroupNorm-based ResNet-18 contains
$P=11{,}173{,}962$ parameters on CIFAR-10 and $P=11{,}220{,}132$ on
CIFAR-100. Because seed-regenerated Gaussian maps are prohibitively expensive
for this architecture, we use the layer-wise SHM parameterization. 
Under the same transition-measurement protocol used in the preceding
experiments, both tasks exhibit a sharp training-success transition at latent
dimensions that are small fractions of the full parameter count. The measured
midpoint is approximately $6\times10^3$ on CIFAR-10 and
$d_{50}=10{,}752$ on CIFAR-100, corresponding to latent fractions of order
$10^{-4}$--$10^{-3}$ of $P$.

These numerical midpoints should be interpreted as protocol-dependent
end-to-end thresholds rather than intrinsic constants of the architecture or
dataset. In particular, retraining the full-parameter reference under an
optimization protocol matched to RaMaN substantially lowers the reference
training loss, and none of the latent dimensions in the original ResNet-18
grid reaches the resulting stricter target. Consequently, a
protocol-matched $d_{50}$ is not defined over the evaluated grid rather
than simply being shifted to a larger measured value.

The initially observed post-transition accuracy gap is not, by itself,
evidence of an unavoidable cost of random reparameterization. We therefore
repeated the CIFAR-10 comparison under a conventional training protocol
with standard data augmentation, a 5,000-example validation split, and
validation-selected learning rate and checkpoint. The epoch budget,
optimizer, weight decay, cosine schedule, training data, and seed count
were matched between RaMaN and full-parameter training. Because the
learning-rate grids used in the original comparison placed both conditions
at their respective grid boundaries, we additionally widened the
learning-rate search for both conditions.

\begin{table}[t]
\centering
\small
\caption{
Protocol-matched ResNet-18/CIFAR-10 generalization comparison with
$45{,}000$ training examples. Both RaMaN and full-parameter training use
the same augmentation, validation split, epoch budget, optimizer, weight
decay, and cosine schedule, with learning rate and checkpoint selected by
validation performance from widened learning-rate grids. The generalization
gap is training accuracy minus test accuracy. Results use three seeds.
}
\label{tab:resnet_tuned_generalization}
\setlength{\tabcolsep}{5pt}
\begin{tabular}{@{}lrrrr@{}}
\toprule
Method / $d$
& $d/P$
& Train acc.
& Test acc.
& Gen.\ gap \\
\midrule
Full parameters
& $100\%$
& $98.89\%$
& $89.95\%$
& $8.9$ \\
\midrule
RaMaN, $1{,}048{,}576$
& $9.4\%$
& $96.33\%$
& $90.33\%$
& $6.0$ \\
RaMaN, $1{,}572{,}864$
& $14.1\%$
& $99.22\%$
& $90.61\%$
& $8.6$ \\
RaMaN, $2{,}097{,}152$
& $18.8\%$
& $98.13\%$
& $90.36\%$
& $7.8$ \\
\bottomrule
\end{tabular}
\end{table}

Table~\ref{tab:resnet_tuned_generalization} reports the resulting
protocol-matched comparison at $n_{\mathrm{train}}=45{,}000$.
At $d=1{,}048{,}576$, corresponding to only $9.4\%$ of the full parameter
count, RaMaN reaches $90.33\%$ test accuracy, $0.38$ percentage points
above the $89.95\%$ full-parameter reference.
This improves the RaMaN result at the
same dimension by $1.64$ percentage points relative to the original
narrow-grid run. Notably, RaMaN attains this test accuracy despite lower
training accuracy ($96.33\%$ versus $98.89\%$), resulting in a smaller
generalization gap ($6.0$ versus $8.9$ percentage points). Thus, the
remaining accuracy deficit observed under the original learning-rate grid
is not intrinsic to the random reparameterization: under the matched tuning
protocol, the test-accuracy crossover occurs by $d/P=9.4\%$.
Because the resulting test-accuracy margin is small and the selected
full-parameter learning rate remains at the boundary of its search grid,
we interpret this result primarily as evidence of accuracy parity and
crossover rather than as a statistically resolved superiority claim.

\FloatBarrier

\subsection{Sensitivity Analyses}
\label{subsec:exp_sensitivity}

Experiments~10--11 examine two factors involved in measuring an end-to-end
training transition: the optimizer used to search the random slice and the
loss-excess tolerance used to define training success. These experiments test
the robustness of the observed transition and clarify which quantities must
accompany a reported training midpoint $d_{50}$.

\subsubsection{Optimizer Sensitivity}
\label{subsubsec:optimizer_ablation}

\paragraph{Experiment 10: AdamW versus SGD.}
The end-to-end training midpoint combines two effects: whether the random
slice contains a low-loss solution and whether the chosen optimizer finds it.
To assess the optimization contribution to the observed transition, we repeat
the layer-wise Gaussian and SHM sweeps using SGD with momentum $0.9$, while
holding the model, latent-dimension grid, random seeds, reference solution,
and success tolerance fixed. Because the
$1/\sqrt d$ learning-rate scaling used for AdamW is not an optimizer-independent
rule, SGD receives a wider per-$d$ learning-rate search.

\begin{figure}[t]
\centering
\includegraphics[width=0.68\linewidth]{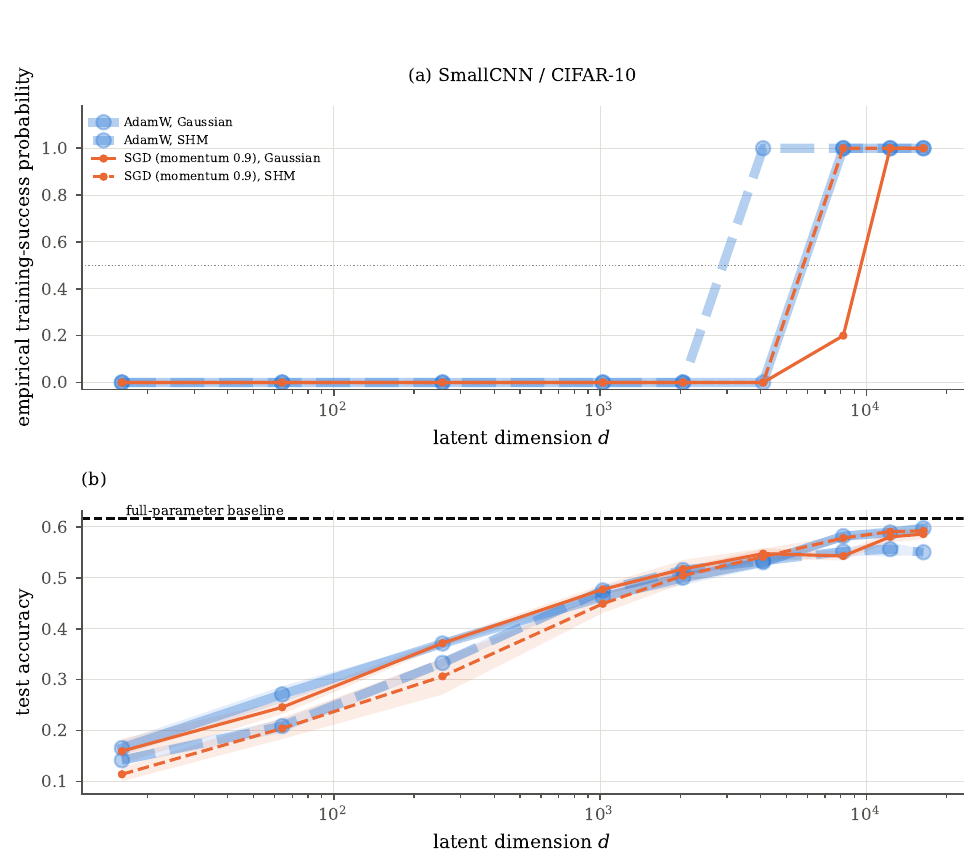}
\caption{\textbf{Optimizer sensitivity of the SmallCNN transition.}
(a) Empirical training-success probability and (b) test accuracy versus total
latent dimension under AdamW and SGD with momentum $0.9$, for layer-wise
seed-regenerated Gaussian and SHM maps. Both optimizers retain a sharp
training transition, but SGD shifts its location toward larger $d$ on
SmallCNN. TinyMLP results, for which no corresponding shift is resolved on
the evaluated grid, are summarized in Table~\ref{tab:optimizer_ablation}.}
\label{fig:optimizer_ablation}
\end{figure}

\begin{table}[t]
\caption{Empirical training midpoint under AdamW and SGD with momentum $0.9$
at the common success tolerance $\varepsilon=0.05$. The comparison is resolved
only to the spacing of the dimension grid.
$^\dagger$The SGD/Gaussian estimate for SmallCNN is potentially upper biased because
the selected learning rate reached the boundary of the search grid near the
crossing.}
\label{tab:optimizer_ablation}
\centering
\small
\setlength{\tabcolsep}{5pt}
\begin{tabular}{@{}llrrr@{}}
\toprule
Model & Layer-wise map & AdamW $d_{50}$ & SGD $d_{50}$ & SGD/AdamW \\
\midrule
TinyMLP & Seed-regenerated Gaussian & 768   & 768   & 1.00 \\
        & SHM                       & 1,195 & 1,195 & 1.00 \\
\addlinespace
SmallCNN & Seed-regenerated Gaussian & 6,144 & $9,728^{\dagger}$ & 1.58 \\
         & SHM                       & 3,072 & 6,144 & 2.00 \\
\bottomrule
\end{tabular}
\end{table}

Figure~\ref{fig:optimizer_ablation} shows that the sharp end-to-end
training transition persists under both optimizers on SmallCNN, but its
location shifts toward larger latent dimensions under SGD. The corresponding
test-accuracy curves remain broadly similar across optimizers, indicating
that the main effect of the optimizer change in this experiment is on the
dimension required to reach the prescribed training-loss criterion rather
than on a qualitatively different accuracy trend.

Table~\ref{tab:optimizer_ablation} quantifies the transition locations across
both models. The two optimizers agree within the resolution of the evaluated
grid on TinyMLP, whereas SGD requires approximately $1.6$--$2.0\times$ more
latent dimensions on SmallCNN. Thus, the existence of a sharp transition is
not specific to AdamW, but the empirical end-to-end midpoint is
optimizer dependent. End-to-end $d_{50}$ should therefore be reported
together with the optimizer and tuning protocol. This dependence does not
alter the optimizer-free least-squares transitions of
Experiments~1--3 and~12.  
For TinyMLP, the optimizer-ablation sweep uses the common comparison grid
rather than the denser transition-localization grid used in Experiment~4;
hence its AdamW midpoint is a grid-level estimate and need not coincide
exactly with the refined value reported there.

\subsubsection{Sensitivity to the Loss-Excess Tolerance}
\label{subsubsec:tolerance_sensitivity}

\paragraph{Experiment 11: varying $\varepsilon$.}
The end-to-end success event is defined by
\[
    \mathcal L(\theta_{\mathrm{RaMaN}})
    \le
    \mathcal L(\theta_{\mathrm{full}})+\varepsilon.
\]
The main experiments use $\varepsilon=0.05$. We evaluate alternative
tolerances to determine whether the observed sharp transition depends on
this particular choice of $\varepsilon$.

\begin{figure}[t]
\centering
\includegraphics[width= \linewidth]{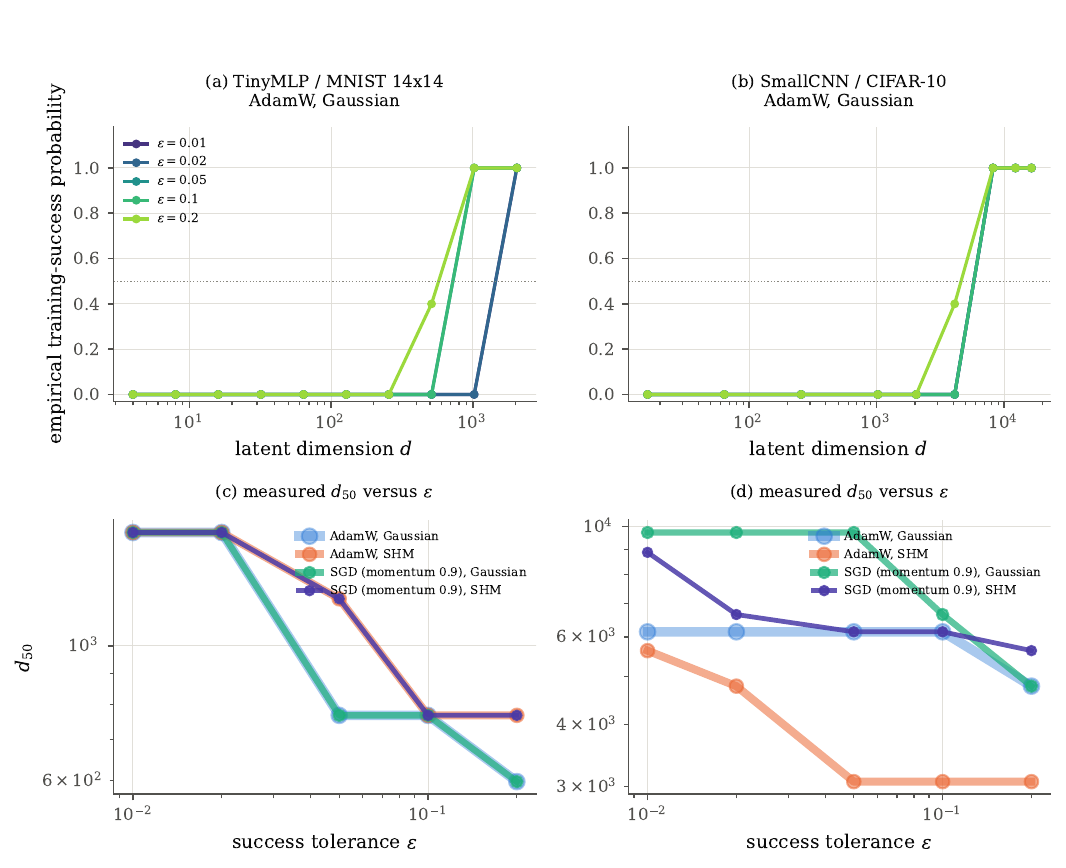}
\caption{\textbf{Sensitivity to the success tolerance.}
(a--b) Representative AdamW/Gaussian training-success curves for five values
of $\varepsilon$ on TinyMLP and SmallCNN.
(c--d) Empirical midpoint $d_{50}$ versus $\varepsilon$ for both optimizers
and both layer-wise map families. Relaxing the tolerance shifts the transition
toward smaller $d$, while the representative success curves remain sharply
localized across the tested range.}
\label{fig:tolerance_sensitivity}
\end{figure}

\begin{table}[t]
\caption{AdamW transition midpoint as the success tolerance varies over a
$20\times$ range, using the common comparison-grid sweeps from
Experiment~10. All values are recomputed from the same saved per-seed
training losses without retraining; ``Spread'' denotes the ratio of the
largest to the smallest midpoint within each row.}
\label{tab:tolerance_sensitivity}
\centering
\small
\setlength{\tabcolsep}{4pt}
\begin{tabular}{@{}llrrrrrr@{}}
\toprule
 & & \multicolumn{5}{c}{$d_{50}$ at tolerance $\varepsilon$} & \\
\cmidrule(lr){3-7}
Model & Layer-wise map & 0.01 & 0.02 & 0.05 & 0.10 & 0.20 & Spread \\
\midrule
TinyMLP & Seed-regenerated Gaussian & 1,536 & 1,536 & 768   & 768   & 597   & $2.6\times$ \\
        & SHM                       & 1,536 & 1,536 & 1,195 & 768   & 768   & $2.0\times$ \\
\addlinespace
SmallCNN & Seed-regenerated Gaussian & 6,144 & 6,144 & 6,144 & 6,144 & 4,779 & $1.3\times$ \\
         & SHM                       & 5,632 & 4,779 & 3,072 & 3,072 & 3,072 & $1.8\times$ \\
\bottomrule
\end{tabular}
\end{table}

Figure~\ref{fig:tolerance_sensitivity}(a--b) shows representative
training-success curves as the loss-excess tolerance varies over
$\varepsilon\in\{0.01,0.02,0.05,0.10,0.20\}$. As expected, relaxing the
tolerance shifts the transition toward smaller latent dimensions because
every run that succeeds at a stricter tolerance also succeeds at a looser
one. More importantly, the representative curves remain sharply localized
throughout this $20\times$ tolerance range rather than becoming diffuse as
the success criterion changes.

Figure~\ref{fig:tolerance_sensitivity}(c--d) extends this comparison to both
optimizers and both map families, while
Table~\ref{tab:tolerance_sensitivity} reports the corresponding AdamW
midpoints numerically. Across the tested tolerance range, the AdamW midpoint
changes by factors of $1.3$--$2.6$, depending on the model and map family.
Thus, the existence of a sharp end-to-end transition is robust to the
particular tolerance used to define success, although its numerical location
is necessarily tolerance dependent. Empirical midpoints should therefore be
compared only under a common $\varepsilon$.

This tolerance dependence should not be interpreted as uncertainty in a
midpoint measured at fixed $\varepsilon$. Likewise, predictor comparisons
performed at a common tolerance remain well defined. If the quadratic
predictors are compared across different tolerances, however, they must be
recomputed using the corresponding quadratic tolerance
$\varepsilon_{\mathrm q}=2\varepsilon$.

\FloatBarrier
\subsection{ViT-Scale Quadratic-Surrogate Diagnostic}
\label{subsec:exp_vit_probe}

\paragraph{Experiment 12: quadratic probes for ViT-Tiny.}
Experiments~7--9 directly measure end-to-end training transitions at
million-parameter scale but do not attach curvature-based predictors to those
transitions. We therefore apply the quadratic-probe pipeline of
Experiment~3 to the ViT-Tiny configurations used in Experiment~7. The probe
uses the GGN operator, $k=512$ leading Ritz eigenpairs, a stochastic trace
estimate for the unresolved spectral tail, and the exact least-squares hit
criterion applied to the resulting surrogate. The parameter counts match the
corresponding end-to-end models exactly.

\begin{table}[t]
\caption{Practical quadratic predictors and measured transition midpoint for
the ViT-scale GGN--Ritz surrogate. Here,
$\widetilde d_{\mathrm{MF}}$ is the practical orientation-resolved predictor,
$\widehat r_{\mathrm{eff}}^{\mathrm{eq}}$ is the practical equal-tail
orientation-uniform estimate, and ``resolved fraction'' is the fraction of
squared displacement represented on the $k=512$ computed Ritz directions.}
\label{tab:vit_probe}
\centering
\footnotesize
\setlength{\tabcolsep}{3.5pt}
\renewcommand{\arraystretch}{1.08}
\resizebox{\linewidth}{!}{%
\begin{tabular}{@{}llrrrrrrr@{}}
\toprule
Reference & Task & $P$ & $R$ & Surrogate $d_{50}$
& $\widetilde d_{\mathrm{MF}}$ & $d^\star_{\mathrm{iso}}$
& $\widehat r_{\mathrm{eff}}^{\mathrm{eq}}$ & Resolved fraction \\
\midrule
Initialization & CIFAR-10  & 1,806,538 & 46.61 & 122.9 & 122.8 & 0.0 & 5,807.7 & 0.4\% \\
Pilot         & CIFAR-10  & 1,806,538 & 27.93 & 157.5 & 158.1 & 0.0 &   813.5 & 6.0\% \\
Initialization & CIFAR-100 & 1,823,908 & 61.42 &  90.7 &  91.4 & 0.0 &   509.1 & 0.2\% \\
Pilot         & CIFAR-100 & 1,823,908 & 38.94 &  99.1 &  99.4 & 0.0 &   482.5 & 2.1\% \\
\bottomrule
\end{tabular}%
}
\end{table}

\begin{figure}[t]
\centering
\includegraphics[width=\linewidth]{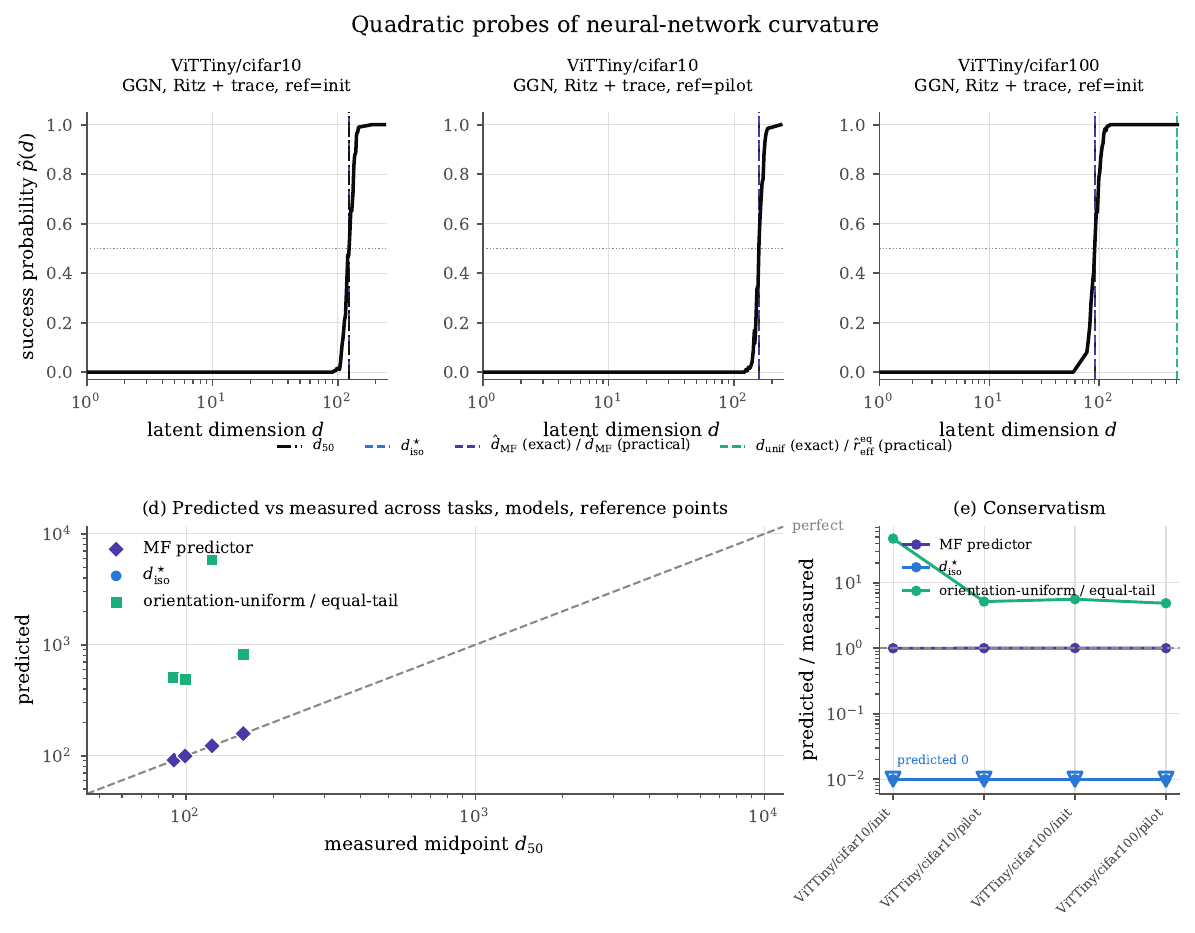}
\caption{\textbf{Quadratic-surrogate probes at ViT scale.}
The top row shows three representative success curves for random slices of
the GGN--Ritz surrogate. The bottom-left panel compares practical predictions
with the measured surrogate midpoint, and the bottom-right panel reports the
prediction-to-measurement ratio. The close agreement of
$\widetilde d_{\mathrm{MF}}$ with the diagonal is an internal consistency
result for the surrogate; it is not an independent validation of the true
network curvature or of the end-to-end transition in Experiment~7.}
\label{fig:vit_probe}
\end{figure}

Within the GGN--Ritz surrogate, the practical master-formula predictor tracks
the measured midpoint to within 0.04--0.79\% across the four cells. The
isotropic-orientation predictor evaluates to zero, whereas the practical
orientation-uniform estimate exceeds the measured midpoint by factors of
$4.9$--$47.3$. This ordering is consistent with Experiment~3 and indicates
that the equal-energy orientation model is inappropriate for these fitted
surrogates.

The result nevertheless requires a strict qualification. Only 0.2--6.0\% of
the squared displacement is resolved on the computed Ritz directions; the
remaining 94--99.8\% is represented by the tail model. Moreover, an
independent exact-Hessian-vector spot check used only three trials at each of
$\{0.5,1,2\}\times d_{50}$ and produced zero successes at the surrogate
midpoint in the tested cells. The current data therefore support the predictor
ordering and approximate scale of the surrogate transition, but they do not
establish sub-percent predictor accuracy for the true ViT curvature. More
exact probes concentrated near the surrogate midpoint are needed for that
claim.

\FloatBarrier

\subsection{Implementation and Reproducibility Considerations}
\label{subsec:exp_reproducibility}

Several controls materially affect the measured end-to-end transition. For
AdamW, holding the latent learning rate fixed across $d$ makes the induced
parameter-space step grow approximately as $\sqrt d$ and can produce a
nonmonotone apparent success curve; the AdamW sweeps therefore use the
reported dimension-normalized learning-rate initialization followed by
per-$d$ tuning. Experiment~10 shows that this scaling should not be transferred
unchanged to SGD, for which a broader per-$d$ search is used. A coarse uniform
$d$-grid can also mislocate a sharp transition, so the reported midpoints use a
two-stage grid or local interpolation around the crossing.

Experiments~10--11 further show that an end-to-end midpoint is conditional on
the optimizer and the loss-excess tolerance. These quantities should be
reported whenever $d_{50}$ is quoted. They do not affect the optimizer-free
least-squares transitions in the quadratic experiments, but a change in
$\varepsilon$ must be accompanied by the corresponding change
$\varepsilon_{\mathrm q}=2\varepsilon$ when a quadratic predictor is
recomputed.

Finally, adaptive map growth uses a fixed master seed and nested prefixes so
that
\[
    \mathcal R_{d_{\mathrm{new}}}([z_d;0])
    =
    \mathcal R_d(z_d),
\]
rather than resampling the previously trained random subspace. 
The original ViT-Tiny and ResNet-18 transition sweeps use the baseline
reference protocol specified for those experiments, so their full-parameter
accuracy comparisons should be interpreted conditionally on that protocol.
Experiment~9 additionally includes a conventionally trained, protocol-matched
ResNet-18 control to separate the training-transition measurement from
post-transition generalization. In the language-model experiment, the embedding
module is explicitly excluded from reparameterization, so the reported latent
fraction is relative to $P_{\mathrm{rep}}$ rather than the full model size.

\section{Discussion and Conclusion}
\label{sec:discussion}

\paragraph{Low-loss geometry, rather than ambient dimension, governs
accessibility.}
Consistent with the geometric perspective of prior random-subspace analyses, the ambient parameter count P alone does not determine accessibility. Our quadratic analysis further shows that, beyond the curvature spectrum, the orientation of the reference-to-solution displacement can materially alter the predicted transition. For a compact
convex low-loss region, the random-affine-slice analysis instead relates the
transition to its effective codimension, formalized through the statistical
dimension of the corresponding polar cone. In the local quadratic model,
this geometric quantity depends not only on the curvature spectrum but also
on the orientation of the reference-to-solution displacement relative to
that spectrum. The orientation-resolved master formula makes this dependence
explicit. Experiments~1--3 support this quadratic prediction across synthetic
spectra and neural-network curvature operators: the orientation-resolved
predictor closely tracks the measured quadratic transition, whereas
spectrum-only or orientation-uniform specializations can be substantially
less accurate or more conservative. These experiments validate the
quadratic residual predictions rather than the full localized conic theorem;
a direct empirical test of the latter would additionally require enforcing
or independently calibrating its localization constraint.

\paragraph{Geometric accessibility and end-to-end optimization are distinct.}
The end-to-end experiments show that the sharp transition predicted by the
geometric picture has an empirical counterpart under gradient-based
training. Pronounced training-success transitions appear across MLPs, CNNs,
vision transformers, ResNet-18, and pretrained language-model fine-tuning,
suggesting that the phenomenon is not restricted to a particular model
class or initialization regime. At the same time, the numerical transition
midpoint should not be interpreted as an intrinsic constant of the loss
landscape. Unlike the quadratic least-squares transition, an end-to-end
$d_{50}$ combines whether the random search space contains an acceptable
solution with whether the chosen optimizer can find one. Experiments~10--11
show explicitly that its location depends on the optimizer, tuning protocol,
and loss-excess tolerance, and the ResNet-18 controls further show sensitivity
to the reference-training protocol. Thus, empirical $d_{50}$ values are most
meaningful when accompanied by the map family, optimizer, success tolerance,
and reference protocol under which they were measured.

\paragraph{Training accessibility does not imply full-model generalization.}
A second distinction emerging from the experiments is between the dimension
needed to reach a prescribed training-loss target and the dimension needed
to approach the predictive performance of full-parameter training. It is
useful to denote these conceptually different scales by
$d_{\mathrm{train}}$ and $d_{\mathrm{gen}}$, respectively. The phase
transition studied in this work concerns primarily $d_{\mathrm{train}}$.
The ResNet-18 experiments show that substantially larger latent dimensions
may be required before test accuracy approaches the full-parameter
reference. At the same time, this generalization gap is not intrinsic to
random reparameterization: under a conventional protocol with widened
learning-rate tuning applied to both conditions, RaMaN at
$d=1{,}048{,}576$ ($9.4\%$ of $P$) slightly exceeds the corresponding
full-parameter test accuracy while exhibiting a smaller generalization gap.
Thus, the dimension needed for training accessibility can be much smaller
than the dimension needed for full-model predictive performance, but the
latter need not approach $P$.

\paragraph{Practical implications for RaMaN.}
The theory also provides a principled starting point for designing scalable
random reparameterizations. Dense Gaussian maps preserve the ideal random
subspace model but require an $O(dP)$ frozen mapping object. RaMaN replaces
this storage bottleneck with seed-regenerated Gaussian or structured
Hadamard maps while maintaining only $d$ trainable latent variables and
$O(d)$ optimizer state. The experiments show, however, that no map family is
uniformly preferable: Gaussian and SHM constructions can differ in transition location,
predictive behavior, and runtime, and their ordering
can reverse across models or datasets. Similarly, the sweep-free
dimension-selection experiments show that the conservative spectral-tail
correction can substantially overestimate the required dimension, while the
equal-tail approximation provides a less conservative practical default and
adaptive expansion provides a curvature-free alternative. Because these
selection experiments and the independent transition reference use different
map families, their numerical ratios should be interpreted as relative
scales rather than same-map optimality guarantees.

\paragraph{Limitations and future directions.}
Several limitations delimit the scope of the present results. First, the
exact random-affine-slice theory applies to uniformly random Gaussian
subspaces, whereas the layer-wise and SHM constructions used for scalable
training introduce additional structure. Extending the theory directly to
such product and structured random maps would reduce the current gap between
the clean geometric model and the practical implementation. Second, the
Hessian-based predictors rely on a local quadratic description of the
low-loss geometry. Exact curvature information is already expensive for
moderate networks, and the ViT-scale analysis necessarily relies on a
GGN--Ritz approximation and a modeled spectral tail. Its close predictor
agreement should therefore be viewed as an internal consistency result for
that surrogate rather than as sub-percent validation of the true
large-network curvature. Developing more reliable matrix-free curvature and
displacement estimators is important for practical one-shot dimension
selection at larger scale.

A further open problem is to characterize $d_{\mathrm{gen}}$, rather than
only the accessibility threshold $d_{\mathrm{train}}$. The present theory
explains when a random low-dimensional search space can intersect a low-loss
region, but it does not determine which point within that region gradient
optimization will select or how that choice affects generalization.
Understanding this second scale will likely require combining the geometry
developed here with optimization dynamics, implicit regularization, sample
size, and properties of the data distribution. The same distinction is
relevant to nonlinear random maps. In particular, a coordinate-wise
bijective transformation $\phi$ in a parameterization $A\phi(z)$ leaves the
reachable set equal to $\operatorname{col}(A)$ and therefore cannot enlarge
the representable parameter space. Genuinely state-dependent mappings can
bend the reachable manifold beyond a fixed linear subspace, but whether such
additional geometry improves generalization enough to offset its effects on
conditioning and optimization remains an open question. These directions,
together with evaluation on larger pretrained models and more diverse tasks,
provide natural extensions of the present framework.

\paragraph{Conclusion.}

This work develops a geometric and algorithmic framework for understanding
when neural networks can be trained through frozen random low-dimensional
reparameterizations. Building on the effective-codimension view of
random-subspace accessibility, our quadratic analysis identifies
displacement orientation as an additional determinant of the latent
dimension required for a random search space to access a low-loss region,
and the resulting quadratic predictors are supported across controlled
and neural-curvature experiments.
RaMaN translates this perspective into scalable matrix-free
random mappings that avoid the $O(dP)$ storage of dense random generators,
while end-to-end experiments show that sharp training transitions persist
across substantially different architectures and training regimes. At the
same time, the experiments establish an important boundary on the
interpretation of this phenomenon: the measured transition is conditional
on the optimization and success protocol, and crossing it does not by itself
guarantee full-model generalization. Taken together, these results provide a
principled account of when random low-dimensional training becomes feasible,
while separating that question from the distinct problems of optimization
efficiency and predictive generalization.

\printbibliography

@inproceedings{li2018measuring,
  title={Measuring the intrinsic dimension of objective landscapes},
  author={Li, Chunyuan and Farkhoor, Heerad and Liu, Rosanne and Yosinski, Jason},
  booktitle={International Conference on Learning Representations},
  year={2018}
}

@inproceedings{aghajanyan2021intrinsic,
  title     = {Intrinsic Dimensionality Explains the Effectiveness of Language Model Fine-Tuning},
  author    = {Aghajanyan, Armen and Gupta, Sonal and Zettlemoyer, Luke},
  booktitle = {Proceedings of the 59th Annual Meeting of the Association for Computational Linguistics and the 11th International Joint Conference on Natural Language Processing (Volume 1: Long Papers)},
  pages     = {7319--7328},
  year      = {2021},
  month     = aug,
  address   = {Online},
  publisher = {Association for Computational Linguistics},
  doi       = {10.18653/v1/2021.acl-long.568},
  url       = {https://aclanthology.org/2021.acl-long.568/}
}

@article{nooralinejad2022pranc,
  title        = {{PRANC}: Pseudo {RAndom} Networks for Compacting Deep Models},
  author = {Nooralinejad, Parsa and Abbasi, Ali and Koohpayegani, Soroush Abbasi and Pourahmadi Meibodi, Kossar and Khan, Rana Muhammad Shahroz and Kolouri, Soheil and Pirsiavash, Hamed},
  journal      = {arXiv preprint arXiv:2206.08464},
  year         = {2022},
  url          = {https://arxiv.org/abs/2206.08464}
}

@article{koohpayegani2023nola,
  title        = {{NOLA}: Compressing {LoRA} using Linear Combination of Random Basis},
  author       = {Koohpayegani, Soroush Abbasi and Navaneet, K. L. and Nooralinejad, Parsa and Kolouri, Soheil and Pirsiavash, Hamed},
  journal      = {arXiv preprint arXiv:2310.02556},
  year         = {2023},
  url          = {https://arxiv.org/abs/2310.02556}
}

@inproceedings{sen2026mapping,
  title        = {Mapping Networks},
  author       = {Sen, Lord and Mukherjee, Shyamapada},
  booktitle    = {Proceedings of the IEEE/CVF Conference on Computer Vision and Pattern Recognition},
  year         = {2026}
}

@article{amelunxen2014living,
  title        = {Living on the Edge: Phase Transitions in Convex Programs with Random Data},
  author       = {Amelunxen, Dennis and Lotz, Martin and McCoy, Michael B. and Tropp, Joel A.},
  journal      = {Information and Inference: A Journal of the IMA},
  volume       = {3},
  number       = {3},
  pages        = {224--294},
  year         = {2014},
  publisher    = {Oxford University Press},
  doi          = {10.1093/imaiai/iau005}
}

@incollection{gordon1988milman,
  title        = {On Milman's Inequality and Random Subspaces which Escape through a Mesh in {$\mathbb{R}^n$}},
  author       = {Gordon, Yehoram},
  booktitle    = {Geometric Aspects of Functional Analysis},
  series       = {Lecture Notes in Mathematics},
  volume       = {1317},
  pages        = {84--106},
  publisher    = {Springer},
  year         = {1988}
}

@book{vershynin2018high,
  title        = {High-Dimensional Probability: An Introduction with Applications in Data Science},
  author       = {Vershynin, Roman},
  series       = {Cambridge Series in Statistical and Probabilistic Mathematics},
  publisher    = {Cambridge University Press},
  year         = {2018},
  doi          = {10.1017/9781108231596}
}

@article{moreau1962decomposition,
  title={D{\'e}composition orthogonale d'un espace hilbertien selon deux c{\^o}nes mutuellement polaires},
  author={Moreau, Jean Jacques},
  journal={Comptes rendus hebdomadaires des s{\'e}ances de l'Acad{\'e}mie des sciences},
  volume={255},
  pages={238--240},
  year={1962}
}

@article{rubio2011spectral,
  title={Spectral convergence for a general class of random matrices},
  author={Rubio, Francisco and Mestre, Xavier},
  journal={Statistics \& Probability Letters},
  volume={81},
  number={5},
  pages={592--602},
  year={2011},
  publisher={Elsevier}
}

@article{ledoit2011eigenvectors,
  author  = {Ledoit, Olivier and P\'ech\'e, Sandrine},
  title   = {Eigenvectors of some large sample covariance matrix
             ensembles},
  journal = {Probability Theory and Related Fields},
  volume  = {151},
  pages   = {233--264},
  year    = {2011},
	doi = {10.1007/s00440-010-0298-3}
}

@inproceedings{derezinski2020surrogate,
  author    = {Derezi\'nski, Micha{\l} and Liang, Feynman and
               Mahoney, Michael W.},
  title     = {Exact expressions for double descent and implicit
               regularization via surrogate random design},
  booktitle = {Advances in Neural Information Processing Systems},
	volume    = {33},
	year      = {2020}
}

@inproceedings{hu2022lora,
  title     = {{LoRA}: Low-Rank Adaptation of Large Language Models},
  author    = {Hu, Edward J. and Shen, Yelong and Wallis, Phillip and
               Allen-Zhu, Zeyuan and Li, Yuanzhi and Wang, Shean and
               Wang, Lu and Chen, Weizhu},
  booktitle = {International Conference on Learning Representations},
  year      = {2022}
}

@inproceedings{le2013fastfood,
  title     = {Fastfood - Computing Hilbert Space Expansions in Loglinear Time},
  author    = {Le, Quoc and Sarlos, Tamas and Smola, Alexander},
  booktitle = {Proceedings of the 30th International Conference on Machine Learning},
  pages     = {244--252},
  year      = {2013},
  editor    = {Dasgupta, Sanjoy and McAllester, David},
  volume    = {28},
  number    = {3},
  series    = {Proceedings of Machine Learning Research},
  address   = {Atlanta, Georgia, USA},
  publisher = {PMLR},
  url       = {https://proceedings.mlr.press/v28/le13.html}
}

@article{ailon2009fast,
  title={The fast Johnson--Lindenstrauss transform and approximate nearest neighbors},
  author={Ailon, Nir and Chazelle, Bernard},
  journal={SIAM Journal on computing},
  volume={39},
  number={1},
  pages={302--322},
  year={2009},
  publisher={SIAM}
}

@article{tropp2011improved,
  title={Improved analysis of the subsampled randomized Hadamard transform},
  author={Tropp, Joel A},
  journal={Advances in Adaptive Data Analysis},
  volume={3},
  number={01n02},
  pages={115--126},
  year={2011},
  publisher={World Scientific}
}

@article{sagun2017empirical,
  title={Empirical analysis of the hessian of over-parametrized neural networks},
  author={Sagun, Levent and Evci, Utku and Guney, V Ugur and Dauphin, Yann and Bottou, Leon},
  journal={arXiv preprint arXiv:1706.04454},
  year={2017}
}

@inproceedings{ghorbani2019investigation,
  title={An investigation into neural net optimization via hessian eigenvalue density},
  author={Ghorbani, Behrooz and Krishnan, Shankar and Xiao, Ying},
  booktitle={International Conference on Machine Learning},
  pages={2232--2241},
  year={2019},
  organization={PMLR}
}

@article{papyan2020traces,
  title={Traces of class/cross-class structure pervade deep learning spectra},
  author={Papyan, Vardan},
  journal={Journal of Machine Learning Research},
  volume={21},
  number={252},
  pages={1--64},
  year={2020}
}

@inproceedings{larsen2022degrees,
  title     = {How Many Degrees of Freedom Do We Need to Train Deep Networks: A Loss Landscape Perspective},
  author    = {Larsen, Brett W. and Fort, Stanislav and Becker, Nic and Ganguli, Surya},
  booktitle = {International Conference on Learning Representations},
  year      = {2022}
}

@inproceedings{gressmann2020improving,
  title     = {Improving Neural Network Training in Low Dimensional Random Bases},
  author    = {Gressmann, Frithjof and Eaton-Rosen, Zach and Luschi, Carlo},
  booktitle = {Advances in Neural Information Processing Systems},
  volume    = {33},
  year      = {2020}
}

@inproceedings{kopiczko2024vera,
  title     = {{VeRA}: Vector-based Random Matrix Adaptation},
  author    = {Kopiczko, Dawid Jan and Blankevoort, Tijmen and Asano, Yuki M.},
  booktitle = {International Conference on Learning Representations},
  year      = {2024}
}

\appendix                   
\numberwithin{equation}{section}
\numberwithin{figure}{section}
\numberwithin{table}{section}

\section{General Conic Intersection Theorem}
\label{append:conicphase}

Throughout this appendix, $S\subset\mathbb{R}^P$ denotes a target set
that the random slice must reach, and $\theta_0\in\mathbb{R}^P$ denotes
the reference point. In the setting of the main text,
\[
S
=
S_\varepsilon
\cap
\overline{B}(\theta_0,R_{\mathrm{loc}}),
\]
where $R_{\mathrm{loc}}>0$ is the localization radius. The ball
localization serves two purposes. First, it makes $S$ compact whenever
$S_\varepsilon$ is closed, which is exactly what is needed for the cone
reduction below to be an exact equivalence
(Remark~\ref{rem:unbounded_pitfall} explains what can go wrong without
localization). Second, it restricts the geometric accessibility statement
to a bounded neighborhood of the reference point.

The localization radius $R_{\mathrm{loc}}$ should be distinguished from
the displacement-radius quantity
\[
R
=
\|\theta^\dagger-\theta_0\|_2
\]
used in the quadratic predictors
$r_{\mathrm{eff}}(\varepsilon,R)$ and
$d^\star_{\mathrm{iso}}$. The former specifies the bounded target region
for the rigorous conic theorem, whereas the latter measures the
reference-to-solution displacement entering the quadratic surrogate.

\begin{lemma}[Exact reduction of affine-slice intersection to a conic
intersection]
\label{lem:cone_reduction}
Let $S \subset \mathbb{R}^P$ be nonempty, compact, and convex, and let
$\theta_0 \notin S$. Define
\[
    C
    \;:=\;
    \bigl\{\, t(\theta - \theta_0) \;:\; t \ge 0,\ \theta \in S \,\bigr\}.
\]
Then:
\begin{enumerate}[label=(\roman*)]
    \item $C$ is a closed convex cone (no closure operation is needed),
    and $C$ contains no line; in particular, $C \neq \{0\}$ and $C$ is
    not a linear subspace.
    \item For every linear subspace $E \subseteq \mathbb{R}^P$,
    \[
        (\theta_0 + E) \cap S \neq \emptyset
        \qquad \Longleftrightarrow \qquad
        E \cap C \neq \{0\}.
    \]
\end{enumerate}
\end{lemma}

\begin{proof}
\emph{(i) Convexity.} Let $v_1 = t_1(\theta_1 - \theta_0)$ and
$v_2 = t_2(\theta_2 - \theta_0)$ with $t_1, t_2 \ge 0$ and
$\theta_1, \theta_2 \in S$, and let $\mu \in [0,1]$. Put
$s := \mu t_1 + (1-\mu) t_2$. If $s = 0$ then
$\mu v_1 + (1-\mu) v_2 = 0 \in C$. If $s > 0$ then
\[
    \mu v_1 + (1-\mu) v_2
    =
    s \,(\bar\theta - \theta_0),
    \qquad
    \bar\theta
    :=
    \frac{\mu t_1 \theta_1 + (1-\mu) t_2 \theta_2}{s} \in S
\]
by convexity of $S$, so the combination lies in $C$.

\emph{Closedness.} Since $S$ is compact and $\theta_0 \notin S$, we have
$\rho_0 := \operatorname{dist}(\theta_0, S) > 0$, so the direction map
$\varphi(\theta) := (\theta - \theta_0)/\|\theta - \theta_0\|_2$ is
continuous on $S$ and its image
$D := \varphi(S) \subset \mathbb{S}^{P-1}$ is compact. Note
$C = \{ t u : t \ge 0,\ u \in D \}$. Let $v_k = t_k u_k \to v$ with
$t_k \ge 0$, $u_k \in D$. If $v = 0$ then $v \in C$. If $v \neq 0$ then
$t_k = \|v_k\|_2 \to \|v\|_2 > 0$, and by compactness of $D$ a
subsequence $u_{k_j} \to u \in D$, whence $v = \|v\|_2\, u \in C$.

\emph{No line.} Suppose $u\neq 0$ with $u\in C$ and $-u\in C$. Then
\[
    u=s(\theta_1-\theta_0),
    \qquad
    -u=t(\theta_2-\theta_0)
\]
for some $\theta_1,\theta_2\in S$ and necessarily $s,t>0$. Adding these
two identities gives
\[
    s(\theta_1-\theta_0)+t(\theta_2-\theta_0)=0,
\]
and hence
\[
    \theta_0
    =
    \frac{s\theta_1+t\theta_2}{s+t}.
\]
Thus $\theta_0\in S$ by convexity of $S$, contradicting
$\theta_0\notin S$. Since $S\neq\emptyset$ and $\theta_0\notin S$, there
exists $\theta\in S$ with $\theta-\theta_0\neq 0$, and hence
$\theta-\theta_0\in C$. Thus $C\neq\{0\}$. Therefore, a cone containing no
line cannot be a nonzero linear subspace, and $C$ is not a linear subspace.

\emph{(ii).}
($\Rightarrow$) If $\theta \in (\theta_0 + E) \cap S$, then
$u := \theta - \theta_0 \in E$, $u = 1 \cdot (\theta - \theta_0) \in C$,
and $u \neq 0$ because $\theta_0 \notin S$.
($\Leftarrow$) If $0 \neq v \in E \cap C$, write
$v = t(\theta - \theta_0)$ with $\theta \in S$; since $v \neq 0$ we have
$t > 0$, hence $\theta = \theta_0 + v/t \in \theta_0 + E$ because $E$ is
a subspace. Thus $\theta \in (\theta_0 + E) \cap S$.
\end{proof}

\begin{remark}[Why compactness matters: a closure pitfall for unbounded
sublevel sets]
\label{rem:unbounded_pitfall}
If $S$ is closed convex but unbounded, the set
$\{t(\theta-\theta_0) : t \ge 0,\ \theta \in S\}$ need not be closed,
and taking its closure can destroy the equivalence in
Lemma~\ref{lem:cone_reduction}(ii). Example: in $\mathbb{R}^2$, let
$S = \{(x,y) : x > 0,\ y \ge 1/x\}$ (closed and convex) and
$\theta_0 = (0,0) \notin S$. The ray generated by $e_1 = (1,0)$ lies in
the \emph{closure} of the cone but not in the cone itself, and the line
$\theta_0 + \operatorname{span}\{e_1\}$ never meets $S$. This situation
is not exotic in our application: when the Hessian $H \succeq 0$ has
zero eigenvalues (flat directions), the quadratic sublevel set
$S_\varepsilon$ is an unbounded cylinder. Let 
$S
=
S_\varepsilon
\cap
\overline{B}(\theta_0,R_{\mathrm{loc}})$. 
The ball localization with radius $R_{\mathrm{loc}}$ removes this issue. 
\end{remark}

We recall the statistical dimension and the two properties used below.

\begin{definition}[Statistical dimension]
For a closed convex cone $C \subseteq \mathbb{R}^P$,
\[
    \delta(C)
    :=
    \mathbb{E}_{g \sim \mathcal{N}(0, I_P)}
    \bigl[\, \|\Pi_C(g)\|_2^2 \,\bigr],
\]
where $\Pi_C$ denotes Euclidean projection onto $C$. The polar cone is
$C^\circ := \{ v : \langle v, u\rangle \le 0 \ \forall u \in C\}$.
\end{definition}

\begin{lemma}[Two standard facts]
\label{lem:sd_facts}
\begin{enumerate}[label=(\alph*)]
    \item If $L$ is a linear subspace, $\delta(L) = \dim L$.
    \item (Moreau complementarity) For every closed convex cone $C$,
    \[
        \delta(C) + \delta(C^\circ) = P .
    \]
\end{enumerate}
\end{lemma}

\begin{proof}
(a) $\Pi_L$ is the orthogonal projector onto $L$, so
$\mathbb{E}\|\Pi_L g\|_2^2 = \operatorname{tr}(\Pi_L) = \dim L$.
(b) By the Moreau decomposition
\cite{moreau1962decomposition}, every $g$ splits as
$g = \Pi_C(g) + \Pi_{C^\circ}(g)$ with
$\langle \Pi_C(g), \Pi_{C^\circ}(g)\rangle = 0$, hence
$\|g\|_2^2 = \|\Pi_C(g)\|_2^2 + \|\Pi_{C^\circ}(g)\|_2^2$. Taking
expectations over $g \sim \mathcal{N}(0, I_P)$ gives
$P = \delta(C) + \delta(C^\circ)$.
\end{proof}

\begin{theorem}[Random affine-slice intersection: conic phase
transition]
\label{thm:random_affine_slice_conic}
Let $S \subset \mathbb{R}^P$ be nonempty, compact, and convex, let
$\theta_0 \notin S$, and let $C$ be the closed convex cone of
Lemma~\ref{lem:cone_reduction}. Let $E \subset \mathbb{R}^P$ be a
uniformly random $d$-dimensional linear subspace; equivalently, $E$ is
the column space of a matrix $A \in \mathbb{R}^{P \times d}$ with
i.i.d.\ $\mathcal{N}(0,1)$ entries, which has rank $d$ almost surely.
Fix $\eta \in (0,1)$ and set
\[
    a_\eta := \sqrt{8 \log(4/\eta)}.
\]
Then
\begin{align}
    d + \delta(C) \;\ge\; P + a_\eta \sqrt{P}
    &\quad\Longrightarrow\quad
    \mathbb{P}\bigl[(\theta_0 + E) \cap S \neq \emptyset\bigr]
    \;\ge\; 1 - \eta,
    \label{eq:conic_hit}\\[2pt]
    d + \delta(C) \;\le\; P - a_\eta \sqrt{P}
    &\quad\Longrightarrow\quad
    \mathbb{P}\bigl[(\theta_0 + E) \cap S \neq \emptyset\bigr]
    \;\le\; \eta.
    \label{eq:conic_miss}
\end{align}
Consequently, the random-slice accessibility transition is centered at
\[
d_{\mathrm{conic}}
:=
P-\delta(C)
=
\delta(C^\circ),
\]
within an $O(\sqrt{P})$ transition window, where the last equality follows
from Lemma~\ref{lem:sd_facts}(b). Thus,
$d_{\mathrm{conic}}$ denotes the geometric transition center rather than
an exact finite-dimensional critical dimension.
\end{theorem}

\begin{proof}
By Lemma~\ref{lem:cone_reduction}(ii), the event
$(\theta_0 + E)\cap S \neq \emptyset$ coincides with the event
$E \cap C \neq \{0\}$.

Write $E = Q E_0$, where $E_0$ is a fixed $d$-dimensional subspace and
$Q$ is a Haar-distributed random orthogonal matrix; this is the
definition of a uniformly random subspace, and it agrees with the
column-space description because the column space of an i.i.d.\ Gaussian
matrix is rotation invariant. A linear subspace is a closed convex cone
with $\delta(E_0) = d$ by Lemma~\ref{lem:sd_facts}(a).

By Lemma~\ref{lem:cone_reduction}(i), $C$ is a closed convex cone that
is \emph{not} a linear subspace. The approximate kinematic formula of
Amelunxen, Lotz, McCoy, and Tropp~\cite[Theorem~I]{amelunxen2014living}
applies to the pair $(C, E_0)$, one of which is not a subspace, and
states that for a Haar-random rotation $Q$:
\[
    \delta(C) + \delta(E_0) \le P - a_\eta\sqrt{P}
    \;\Longrightarrow\;
    \mathbb{P}\bigl[C \cap Q E_0 \neq \{0\}\bigr] \le \eta,
\]
\[
    \delta(C) + \delta(E_0) \ge P + a_\eta\sqrt{P}
    \;\Longrightarrow\;
    \mathbb{P}\bigl[C \cap Q E_0 \neq \{0\}\bigr] \ge 1 - \eta.
\]
Substituting $\delta(E_0) = d$ yields
\eqref{eq:conic_hit}--\eqref{eq:conic_miss}, and
Lemma~\ref{lem:sd_facts}(b) gives
$P - \delta(C) = \delta(C^\circ)$.
\end{proof}

\begin{remark}[Sharper window in the fat-cone regime]
\label{rem:sharp_window}
Theorem~\ref{thm:random_affine_slice_conic} uses the dimension-free
constant $a_\eta\sqrt{P}$ from \cite[Theorem~I]{amelunxen2014living}.
Because our random object is a subspace, the refined subspace form
\cite[Theorem~7.1]{amelunxen2014living} applies and gives a window
governed by $\omega(C):=\sqrt{\delta(C)\wedge\delta(C^\circ)}$ instead
of $\sqrt{P}$: for $\lambda \ge 0$, $d \ge \delta(C^\circ)+\lambda$
implies intersection with probability at least $1-p_C(\lambda)$, and
$d \le \delta(C^\circ)-\lambda$ implies a miss with probability at
least $1-p_C(\lambda)$, where
$p_C(\lambda) = 4\exp\bigl(-\tfrac{\lambda^2/8}{\omega^2(C)+\lambda}\bigr)$.
In the regime of interest,
$\delta(C^\circ)=d_{\mathrm{conic}}\ll P$, so
$\omega^2(C)=d_{\mathrm{conic}}$ and the transition width is
\[
O\!\left(
\sqrt{d_{\mathrm{conic}}\log(1/\eta)}
+
\log(1/\eta)
\right),
\]
which can be far narrower than $O(\sqrt{P})$ and is consistent with the
sharp empirical transitions reported by
Li et al.~\cite{li2018measuring}.
 Concretely,
$p_C(\lambda) \le \eta$ holds for
$\lambda = \max\{4\sqrt{\log(4/\eta)}\,\omega(C),\, 16\log(4/\eta)\}$.
\end{remark}

\begin{remark}[Nonconvex sublevel sets]
For nonconvex $S_\varepsilon$, the exact transition above does not
apply, but two one-sided statements survive. A \emph{miss} bound
follows from Gordon's escape-through-a-mesh theorem
\cite{gordon1988milman} applied to the spherical shadow
$\Sigma = \{(\theta-\theta_0)/\|\theta-\theta_0\|_2 : \theta \in
S_\varepsilon\}$: if $P - d$ exceeds a constant multiple of
$w(\Sigma)^2$, the slice misses with high probability. A \emph{hit}
bound follows by restricting to any compact convex subset of a single
basin (e.g., its localized quadratic ellipsoid) and applying
Theorem~\ref{thm:random_affine_slice_conic}; a union of basins can only
increase the intersection probability.
\end{remark}

\begin{corollary}[Layer-wise slices under product structure]
\label{cor:layerwise}
Suppose the localized target set contains a product
$S^{(1)}\times\cdots\times S^{(L)}$, where each
$S^{(l)}\subset\mathbb{R}^{P_l}$ is nonempty, compact, and convex with
$\theta_0^{(l)}\notin S^{(l)}$, and let $C_l$ be the cone of
Lemma~\ref{lem:cone_reduction} for $(S^{(l)},\theta_0^{(l)})$. Let
$E_1,\ldots,E_L$ be independent uniformly random subspaces of dimensions
$d_1,\ldots,d_L$. If
\[
    d_l \;\ge\; P_l-\delta(C_l)+a_{\eta/L}\sqrt{P_l},
    \qquad l=1,\ldots,L,
\]
then with probability at least $1-\eta$ every layer slice
$\theta_0^{(l)}+E_l$ intersects $S^{(l)}$ simultaneously, so the product
slice intersects the product target.
\end{corollary}
\begin{proof}
Apply Theorem~\ref{thm:random_affine_slice_conic} to each layer with
failure probability $\eta/L$ and take a union bound.
\end{proof}


\section{Heuristic Derivation of the Master Formula and Its
Specializations}
\label{append:reff_derivation}

This section derives the orientation-resolved master formula and its
isotropic-orientation and orientation-uniform specializations, and
clarifies their precise status. The derivation is exact up to a single,
clearly labeled random-matrix approximation
(Approximation~\ref{approx:det_equiv}); we verify this approximation
against an exactly solvable model in
Section~\ref{subsec:block_verification}. The resulting hierarchy contains
three predictors:
\begin{itemize}
    \item the \textbf{orientation-resolved master-formula predictor},
    which uses the complete displacement profile
    $\{\Delta_i\}_{i=1}^{P}$ relative to the curvature eigenvectors and is
    the primary quadratic-model predictor used by Mode~1 of
    Algorithm~\ref{alg:hessian_dim_selection};

    \item the \textbf{isotropic-orientation predictor}
    $d^\star_{\mathrm{iso}}$, obtained by imposing the equal-energy model
    $\Delta_i^2=R^2/P$. It is generally sharper than the
    orientation-uniform predictor and is used primarily as a diagnostic
    prediction when the isotropic-orientation model is appropriate; and

    \item the \textbf{orientation-uniform predictor}
    \[
        d_{\mathrm{unif}}(\varepsilon,R)
        =
        r_{\mathrm{eff}}(\varepsilon,R),
    \]
    obtained by removing the displacement-orientation dependence through
    a uniform upper bound. It requires only the displacement radius $R$
    and is used by Mode~2 of
    Algorithm~\ref{alg:hessian_dim_selection}.
\end{itemize}
The master-formula predictor is the most informative when a reliable
displacement profile is available. The isotropic predictor provides a
sharper falsifiable diagnostic under its equal-energy assumption, whereas
$d_{\mathrm{unif}}=r_{\mathrm{eff}}$ is the conservative radius-only
choice when directional information is unavailable. All three remain
mean-level quadratic-model predictors rather than high-probability
consequences of the conic theorem.

\subsection{Exact reduction to a least-squares residual}

Work in the quadratic model of
Theorem~\ref{thm:hessian_effective_rank}: $H \succeq 0$ with
eigendecomposition $H = \sum_{i=1}^P \lambda_i v_i v_i^\top$, minimizer
$\theta^\dagger$, reference point $\theta_0$, displacement
\[
    \Delta_0 := \theta^\dagger - \theta_0,
    \qquad
    \Delta_i := \langle \Delta_0, v_i\rangle,
    \qquad
    \|\Delta_0\|_2 \le R .
\]
Consider the random slice $\theta(z) = \theta_0 + A z$ with
$A \in \mathbb{R}^{P \times d}$ having i.i.d.\ $\mathcal{N}(0, 1/P)$
entries (the scaling is immaterial: it affects neither the column space
nor the intersection event). Since
\[
    2\bigl(\widetilde{\mathcal{L}}(\theta(z)) -
    \mathcal{L}(\theta^\dagger)\bigr)
    =
    (Az - \Delta_0)^\top H (Az - \Delta_0)
    =
    \bigl\| M z - b \bigr\|_2^2,
    \qquad
    M := H^{1/2} A,
    \quad
    b := H^{1/2}\Delta_0 ,
\]
minimizing over $z$ gives the exact identity
\begin{equation}
    \min_{z\in\mathbb{R}^d}
    2\bigl(\widetilde{\mathcal{L}}(\theta(z)) -
    \mathcal{L}(\theta^\dagger)\bigr)
    \;=\;
    \bigl\| (I - P_W)\, b \bigr\|_2^2
    \;=:\;
    \rho(A),
    \qquad
    W := \operatorname{col}(M) = H^{1/2} \operatorname{col}(A),
    \label{eq:ls_reduction}
\end{equation}
where $P_W$ is the orthogonal projector onto $W$. 
Let
\[
    \widetilde S_\varepsilon
    :=
    \left\{
    \theta:
    \frac12(\theta-\theta^\dagger)^\top
    H(\theta-\theta^\dagger)
    \le \varepsilon
    \right\}
\]
denote the unlocalized quadratic sublevel set. Then
\begin{equation}
    \left\{
    (\theta_0+\operatorname{range}(A))
    \cap
    \widetilde S_\varepsilon
    \neq \emptyset
    \right\}
    \;=\;
    \bigl\{\rho(A) \le 2\varepsilon\bigr\}.
    \label{eq:event_identity}
\end{equation}
Everything so far is exact for the unlocalized quadratic sublevel set.
To connect this residual criterion to the localized target
\[
    S
    =
    \widetilde S_\varepsilon
    \cap
    \overline B(\theta_0,R_{\mathrm{loc}})
\]
used in Theorem~\ref{thm:hessian_effective_rank}, one additionally requires
a least-squares solution $z^\star$ attaining \eqref{eq:ls_reduction} to
satisfy
\[
    \|Az^\star\|_2 \le R_{\mathrm{loc}} .
\]
Thus, the effective-rank derivation should be viewed as a predictor for the
localized conic threshold when the least-squares solution lies inside the
chosen localization ball. In experiments, this condition can be checked by
monitoring the trained displacement norm $\|\Delta\theta\|_2$.

Note that $b$ carries the full orientation
dependence: $b_i = \sqrt{\lambda_i}\,\Delta_i$, so displacement along
flat directions ($\lambda_i = 0$) contributes nothing, while
displacement along stiff directions must be canceled by the random
subspace $W$.

\subsection{Deterministic-equivalent approximation}
\label{app_sec:master_formula}

The columns of $M$ are i.i.d.\ $\mathcal{N}(0, H/P)$ vectors, so $W$ is
the span of $d$ i.i.d.\ Gaussian vectors with covariance proportional to
$H$. This subspace is \emph{not} Haar distributed: it is biased toward
large-$\lambda$ directions, which is precisely why few latent dimensions
suffice.

\begin{approximation}[Ridge-leverage surrogate for the Gaussian-span
projector]
\label{approx:det_equiv}
Let $A\in\mathbb{R}^{P\times d}$ have i.i.d.\
$\mathcal{N}(0,1/P)$ entries, and let
\[
    W=\operatorname{col}(H^{1/2}A).
\]
For $d<\operatorname{rank}(H)$, we use the surrogate
\[
    \mathbb{E}[P_W]
    \;\approx\;
    H(H+\kappa I)^{-1},
\]
where $\kappa=\kappa(d)>0$ is chosen to preserve the exact projector
trace:
\begin{equation}
    d
    =
    \operatorname{tr}\!\left[H(H+\kappa I)^{-1}\right]
    =
    \sum_{i=1}^{P}\frac{\lambda_i}{\lambda_i+\kappa}.
    \label{eq:self_consistent_d}
\end{equation}
The approximation is supported by random-projection results under
high-dimensional concentration conditions, but is used here as a
finite-dimensional surrogate outside those regimes.
\end{approximation}

\paragraph{Leave-one-out derivation.}
We next explain why the ridge-leverage form in
Approximation~\ref{approx:det_equiv} arises. The projector identities below
are exact; the approximation enters only when the leave-one-out quadratic
forms are replaced by a common deterministic equivalent. Define
\[
    M:=H^{1/2}A\in\mathbb{R}^{P\times d},
    \qquad
    W=\operatorname{col}(M).
\]
Let $r=\operatorname{rank}(H)$ and write
$H=U_r\Lambda_rU_r^\top$, where $\Lambda_r\succ0$. Then
\[
    M
    =
    U_r\Lambda_r^{1/2}(U_r^\top A).
\]
Since $A$ is Gaussian and $U_r$ has orthonormal columns,
$U_r^\top A\in\mathbb{R}^{r\times d}$ is also Gaussian and has full
column rank $d$ almost surely whenever $d\le r$. Hence, for
$d<\operatorname{rank}(H)$, the columns of $M$ are linearly independent
almost surely, and the orthogonal projector onto $W$ is
\[
    P_W
    =
    M(M^\top M)^{-1}M^\top.
\]
Since
\[
    M^\top M
    =
    A^\top H A,
\]
we define
\[
    Q:=A^\top H A\in\mathbb{R}^{d\times d},
\]
and obtain
\[
    P_W
    =
    H^{1/2}A Q^{-1}A^\top H^{1/2}.
\]

Because $H$ is symmetric positive semidefinite, it admits an
eigendecomposition
\[
    H
    =
    U\Lambda U^\top,
    \qquad
    \Lambda
    =
    \operatorname{diag}(\lambda_1,\ldots,\lambda_P),
\]
where $U$ is orthogonal. We now express all quantities in this orthonormal
eigenbasis. Under the coordinate transformation
$\widetilde A=U^\top A$, the matrix $\widetilde A$ has the same i.i.d.\
Gaussian distribution as $A$ by rotational invariance. Hence, without
loss of generality, we may work in coordinates in which
\[
    H
    =
    \operatorname{diag}(\lambda_1,\ldots,\lambda_P),
\]
and, for notational simplicity, continue to denote the transformed Gaussian
matrix by $A$.

Let $a_i^\top\in\mathbb{R}^{1\times d}$ denote the $i$th row of $A$.
In these coordinates,
\[
    Q
    =
    A^\top H A
    =
    \sum_{j=1}^{P}\lambda_j a_j a_j^\top.
\]
Moreover, the $i$th row of $H^{1/2}A$ is
$\sqrt{\lambda_i}\,a_i^\top$. Therefore,
\[
\begin{aligned}
    (P_W)_{ii}
    &=
    e_i^\top
    H^{1/2}A Q^{-1}A^\top H^{1/2}
    e_i \\
    &=
    \lambda_i a_i^\top Q^{-1}a_i.
\end{aligned}
\]
Thus, understanding $\mathbb{E}[P_W]$ reduces to approximating the
quantities
\[
    \mathbb{E}\!\left[
        \lambda_i a_i^\top Q^{-1}a_i
    \right],
    \qquad i=1,\ldots,P.
\]

To separate the dependence of $Q^{-1}$ on the row $a_i$, write
\[
    Q
    =
    \sum_{j=1}^{P}\lambda_j a_j a_j^\top
    =
    Q_{-i}+\lambda_i a_i a_i^\top,
    \qquad
    Q_{-i}
    :=
    \sum_{j\ne i}\lambda_j a_j a_j^\top.
\]

For $d<r:=\operatorname{rank}(H)$, the matrix $Q_{-i}$ is invertible
almost surely: after removing the $i$th term, at least $r-1\ge d$
independent Gaussian rows with positive weights remain when
$\lambda_i>0$, while all $r>d$ positively weighted rows remain when
$\lambda_i=0$. We may therefore define
\[
    t_i
    :=
    a_i^\top Q_{-i}^{-1}a_i.
\]
For $\lambda_i>0$, applying the Sherman--Morrison identity to
$Q=Q_{-i}+\lambda_i a_i a_i^\top$ gives
\[
    Q^{-1}
    =
    Q_{-i}^{-1}
    -
    \frac{
        \lambda_i Q_{-i}^{-1}a_i a_i^\top Q_{-i}^{-1}
    }{
        1+\lambda_i a_i^\top Q_{-i}^{-1}a_i
    }.
\]
Sandwiching this identity between $a_i^\top$ and $a_i$ yields
\[
\begin{aligned}
    a_i^\top Q^{-1}a_i
    &=
    t_i
    -
    \frac{\lambda_i t_i^2}{1+\lambda_i t_i}
    \\
    &=
    \frac{t_i}{1+\lambda_i t_i}.
\end{aligned}
\]
Therefore,
\[
    (P_W)_{ii}
    =
    \lambda_i a_i^\top Q^{-1}a_i
    =
    \frac{\lambda_i t_i}{1+\lambda_i t_i}.
\]
The same identity holds trivially when $\lambda_i=0$, since both sides
then vanish. The advantage of this representation is that $Q_{-i}$ depends only on
the rows $\{a_j:j\ne i\}$ and is therefore independent of $a_i$.

Conditional on $Q_{-i}$, the vector
$a_i\sim\mathcal{N}(0,I_d/P)$ is independent of $Q_{-i}$. Therefore,
the standard Gaussian quadratic-form identities give
\[
    \mathbb{E}[t_i\mid Q_{-i}]
    =
    \frac{1}{P}\operatorname{tr}(Q_{-i}^{-1}),
    \qquad
    \operatorname{Var}(t_i\mid Q_{-i})
    =
    \frac{2}{P^2}\operatorname{tr}(Q_{-i}^{-2}).
\]
Consequently, $t_i$ concentrates around its conditional mean whenever
$Q_{-i}^{-1}$ is not spectrally dominated by a small number of
directions. Under standard high-dimensional regularity conditions---for
example, sufficiently large stable rank
\[
    r_H:=\frac{\operatorname{tr}(H)}{\|H\|}
\]
with $r_H/d>1$ bounded away from one---leave-one-out resolvent arguments
further approximate the normalized traces
$P^{-1}\operatorname{tr}(Q_{-i}^{-1})$ by a common deterministic scalar
$\tau$; see \cite[Theorem~2 and its proof discussion]
{derezinski2020surrogate}. We therefore make the approximation
\[
    t_i\approx\tau,
    \qquad i=1,\ldots,P.
\]

 Substituting $t_i\approx\tau$ into the exact leave-one-out identity gives
\[
    \mathbb{E}[(P_W)_{ii}]
    \approx
    \frac{\lambda_i\tau}{1+\lambda_i\tau}
    =
    \frac{\lambda_i}{\lambda_i+\kappa},
    \qquad
    \kappa:=\tau^{-1}.
\]
Gaussian sign symmetry makes the off-diagonal entries of
$\mathbb{E}[P_W]$ vanish in the eigenbasis of $H$, producing
\[
    \mathbb{E}[P_W]
    \approx
    H(H+\kappa I)^{-1}.
\]
Finally, because $\operatorname{tr}(P_W)=d$ exactly, matching the trace
gives \eqref{eq:self_consistent_d}. This leave-one-out argument is the
finite-dimensional heuristic underlying
Approximation~\ref{approx:det_equiv}; rigorous approximation bounds under
stable-rank conditions are developed in
\cite{derezinski2020surrogate}.
The map $\kappa \mapsto \sum_i \lambda_i/(\lambda_i+\kappa)$ is strictly
decreasing from $\operatorname{rank}(H)$ (as $\kappa \to 0$) to $0$ (as
$\kappa\to\infty$), so \eqref{eq:self_consistent_d} defines a bijection
between $d \in (0, \operatorname{rank} H)$ and $\kappa \in (0,\infty)$.

Related deterministic-equivalent results
for sample-covariance resolvents are developed in
\cite{rubio2011spectral,ledoit2011eigenvectors}.
Approximation~\ref{approx:det_equiv} is \emph{exact} when $H \propto I$ (a Haar subspace, where
$\mathbb{E}[P_W] = (d/P) I$) and exact in the block-constant model of
Section~\ref{subsec:block_verification}. 
We use Approximation~\ref{approx:det_equiv} as a finite-dimensional
heuristic. The conic theorem provides a high-probability transition window centered at $\delta(C^\circ)$; it does not, by
itself, control the approximation error of
Approximation~\ref{approx:det_equiv} or the fluctuations of $\rho(A)$.

Combining \eqref{eq:ls_reduction} with
Approximation~\ref{approx:det_equiv} yields the \textbf{master formula}
\begin{equation}
    \mathbb{E}\bigl[\rho(A)\bigr]
    \;\approx\;
    b^\top \bigl(I-H(H+\kappa I)^{-1}\bigr)b
    \;=\;
    \sum_{i=1}^{P}
    \frac{\kappa}{\lambda_i+\kappa}\,b_i^2
    \;=\;
    \sum_{i=1}^{P}
    \frac{\kappa\lambda_i}{\lambda_i+\kappa}\,
    \Delta_i^2,
    \label{eq:master}
\end{equation}
where $Hv_i=\lambda_i v_i$,
\[
    \Delta_i
    :=
    \langle\theta^\dagger-\theta_0,v_i\rangle
\]
is the component of the reference-to-solution displacement along the
$i$th Hessian eigenvector, and
$b_i=\sqrt{\lambda_i}\,\Delta_i$. The scalar
$\kappa=\kappa(d)>0$ is determined by
\eqref{eq:self_consistent_d}.

The final expression in \eqref{eq:master} defines the
orientation-resolved mean-level residual predictor
\[
    \widehat{\rho}_{\mathrm{MF}}(d;\Delta_0)
    :=
    \sum_{i=1}^{P}
    \frac{
        \kappa(d)\lambda_i
    }{
        \lambda_i+\kappa(d)
    }
    \Delta_i^2.
\]
The corresponding predicted critical dimension is
\begin{equation}
    \widehat d_{\mathrm{MF}}
    :=
    \min
    \left\{
        d\in\{1,\ldots,r-1\}:
        \widehat{\rho}_{\mathrm{MF}}(d;\Delta_0)
        \le
        \varepsilon_{\mathrm q}
    \right\},
    \label{eq:d_master_appendix}
\end{equation}
with the same endpoint conventions as in
Definition~\ref{def:predictors}: we set
$\widehat d_{\mathrm{MF}}=0$ when
$\Delta_0^\top H\Delta_0\le\varepsilon_{\mathrm q}$, and
$\widehat d_{\mathrm{MF}}=r$ when the set in
\eqref{eq:d_master_appendix} is empty. Since
$\varepsilon_{\mathrm q}=2\varepsilon$, this is equivalent to using
$\mathbb{E}[\rho(A)]\le2\varepsilon$ as a mean-level predictor of the
transition. Converting this mean-level condition into a high-probability
success guarantee would require separate concentration control for
$\rho(A)$.

\subsection{Isotropic-orientation specialization}
\label{subsec:isotropic_prediction}

Under the typical-orientation model in which $\Delta_0$ is spread evenly
across the eigenbasis, $\Delta_i^2 = R^2/P$ for all $i$, the master
formula \eqref{eq:master} becomes
\[
    \mathbb{E}\bigl[\rho(A)\bigr]
    \;\approx\;
    \frac{R^2}{P}
    \sum_{i=1}^{P}
    \frac{\kappa\,\lambda_i}{\lambda_i + \kappa}
    \;=\;
    \frac{R^2}{P}\,\kappa\, d(\kappa),
\]
using $\sum_i \kappa\lambda_i/(\lambda_i+\kappa) = \kappa\, d(\kappa)$
from \eqref{eq:self_consistent_d}. Setting this equal to $2\varepsilon$
gives the \textbf{self-consistent isotropic predictor}
$d^\star_{\mathrm{iso}}$, defined by the pair of equations
\begin{equation}
    d^\star_{\mathrm{iso}}
    =
    \sum_{i=1}^{P}
    \frac{\lambda_i}{\lambda_i + \kappa^\star},
    \qquad\qquad
    \kappa^\star \, d^\star_{\mathrm{iso}}
    =
    \frac{2\varepsilon P}{R^2}.
    \label{eq:d_iso}
\end{equation}
The system is monotone (larger $d$ means smaller $\kappa$, hence smaller
residual) and is solved numerically by bisection on $\kappa$: given
$\kappa$, compute $d(\kappa)$ from \eqref{eq:self_consistent_d} and
check the sign of $\kappa\, d(\kappa) - 2\varepsilon P / R^2$.

Because $\kappa^\star = 2\varepsilon P / (R^2 d^\star_{\mathrm{iso}})$
exceeds $2\varepsilon/R^2$ by the factor $P/d^\star_{\mathrm{iso}} \gg
1$ in the low-dimensional regime $d^\star_{\mathrm{iso}} \ll P$, 
the isotropic cutoff is much larger than the orientation-uniform
cutoff, and therefore
\[
    d^\star_{\mathrm{iso}}
    \;\le\;
    d_{\mathrm{unif}}(\varepsilon, R),
\]
often by a substantial margin. The interpretation is a smoothed
tail-mass rule: under isotropy each direction carries excess loss
$\lambda_i R^2 / (2P)$, and the random subspace must (approximately)
cancel the stiffest directions until the \emph{total remaining} excess
drops below $\varepsilon$ --- a tail-sum criterion --- rather than
canceling every direction whose \emph{individual worst-case} excess
$\lambda_i R^2$ exceeds $\varepsilon$, which is what
$d_{\mathrm{unif}}$ does.

\emph{Interpretation.} $d^\star_{\mathrm{iso}}$ is the sharper
falsifiable prediction for phase-transition experiments under random
initialization, where the isotropic model is plausible (and directly
testable by recording the eigenbasis profile of
$H^{1/2}(\theta_0 - \theta_{\mathrm{final}})$).

\subsection{Orientation-uniform specialization}
\label{subsec:orientation_uniform}

Since $\dfrac{\kappa\lambda_i}{\lambda_i+\kappa} \le \kappa$ for every
$i$, the master formula \eqref{eq:master} gives, \emph{uniformly over
all orientations} of the displacement with $\|\Delta_0\|_2 \le R$,
\[
    \mathbb{E}\bigl[\rho(A)\bigr]
    \;\le\;
    \kappa \sum_i \Delta_i^2
    \;\le\;
    \kappa R^2.
\]

Under the deterministic-equivalent approximation, imposing
\[
    \kappa R^2
    \le
    \varepsilon_{\mathrm q}
\]
ensures that the orientation-uniform residual surrogate does not exceed
$\varepsilon_{\mathrm q}$ for any displacement satisfying
$\|\Delta_0\|_2\le R$. Setting
$\kappa=\varepsilon_{\mathrm q}/R^2$ gives
\begin{equation}
    d_{\mathrm{unif}}(\varepsilon,R)
    =
    \sum_{i=1}^{P}
    \frac{\lambda_i}
         {\lambda_i+\varepsilon_{\mathrm q}/R^2}
    =
    \sum_{i=1}^{P}
    \frac{\lambda_iR^2}
         {\lambda_iR^2+\varepsilon_{\mathrm q}}.
    \label{eq:d_unif}
\end{equation}

This orientation-uniform bound can be nearly tight when the displacement
is concentrated along a sufficiently stiff direction. In particular, if
$\Delta_0=Rv_j$ and
$\lambda_j\gg\varepsilon_{\mathrm q}/R^2$, then
\[
    \frac{\kappa\lambda_j}{\lambda_j+\kappa}R^2
    \approx
    \kappa R^2,
\]
so the required dimension approaches
$d_{\mathrm{unif}}$. Intuitively, a random isotropic subspace cannot
selectively target one prescribed stiff direction; accurately capturing
that direction requires a sufficiently large ridge leverage, while also
allocating leverage to directions of comparable or greater curvature.

\emph{Interpretation.}
$r_{\mathrm{eff}}(\varepsilon,R)$ is the orientation-uniform mean-level
predictor. It is a conservative radius-only choice and can be approached
when the displacement concentrates along sufficiently stiff curvature
directions.

\subsection{Verification in an exactly solvable block model}
\label{subsec:block_verification}

Let $H = \lambda \bigl(I_k \oplus 0_{P-k}\bigr)$: $k$ stiff directions
of equal curvature $\lambda$ and $P-k$ exactly flat directions, with
isotropic displacement $\Delta_i^2 = R^2/P$.

\emph{Exact computation.} Here $b$ is supported on the stiff block with
$\|b\|_2^2 = \lambda k R^2 / P$, and the columns of $M$ are i.i.d.\
Gaussian vectors supported on the stiff block, so $W$ is a
\emph{uniformly random} $d$-dimensional subspace of the $k$-dimensional
stiff space (for $d \le k$). For a fixed vector and a Haar $d$-subspace
of $\mathbb{R}^k$,
$\mathbb{E}\|(I - P_W) b\|_2^2 = (1 - d/k)\,\|b\|_2^2$. Hence
\[
    \mathbb{E}[\rho]
    =
    \Bigl(1 - \frac{d}{k}\Bigr) \frac{\lambda k R^2}{P}
    =
    \frac{\lambda (k - d) R^2}{P}
    \;\le\; 2\varepsilon
    \quad\Longleftrightarrow\quad
    d \;\ge\; k - \frac{2\varepsilon P}{\lambda R^2}.
\]

\emph{Deterministic-equivalent route.} Equation
\eqref{eq:self_consistent_d} gives $d = k\lambda/(\lambda+\kappa)$,
i.e., $\kappa = \lambda (k-d)/d$, and the master formula
\eqref{eq:master} evaluates to
\[
    \mathbb{E}[\rho]
    \approx
    \frac{\kappa\lambda}{\lambda+\kappa}
    \cdot \frac{k R^2}{P}
    =
    \lambda\,\frac{k-d}{k}\cdot\frac{k R^2}{P}
    =
    \frac{\lambda (k-d) R^2}{P},
\]
in exact agreement. Solving \eqref{eq:d_iso} likewise returns
\[
d^\star_{\mathrm{iso}}
=
\left[
k-\frac{2\varepsilon P}{\lambda R^2}
\right]_+,
\]
with the resulting value rounded upward when an integer latent dimension is
required. In this isotropic block model, the orientation-resolved
master-formula prediction coincides with
$d^\star_{\mathrm{iso}}$. For comparison, the orientation-uniform
dimension is
\[
    d_{\mathrm{unif}}
    =
    \frac{k\lambda R^2}{\lambda R^2+2\varepsilon}
    \approx k
\]
when $\lambda R^2\gg\varepsilon$. The two specialized predictors
$d^\star_{\mathrm{iso}}$ and $d_{\mathrm{unif}}$ coincide in the
strict-tolerance limit
$\varepsilon\ll\varepsilon_0:=\lambda kR^2/(2P)$ and separate as
$\varepsilon$ approaches $\varepsilon_0$, where
$d^\star_{\mathrm{iso}}\to0$ while
$d_{\mathrm{unif}}$ remains approximately $k$.

\subsection{Remarks}

\begin{remark}[Ball localization is typically inactive]
The least-squares minimizer $z^\star$ of \eqref{eq:ls_reduction} is expected
to produce a parameter displacement $Az^\star$ with norm comparable to the
reference-to-solution displacement scale $R$ in the regimes above; e.g., in
the block model of Section~\ref{subsec:block_verification}, one computes
$\|Az^\star\|_2 = O(R\sqrt{d/k})$ for the flat-direction component and
$O(R)$ overall. Thus, when the localization radius is chosen as
$R_{\mathrm{loc}}=cR$ for a moderate constant $c$, the ball constraint is
expected to be inactive. Empirically, $\|\Delta\theta\|_2$ should be
monitored during training as a check.
\end{remark}

\begin{remark}[Rank saturation]
Since $b = H^{1/2}\Delta_0 \in \operatorname{range}(H)$ always, taking
$d \ge \operatorname{rank}(H)$ drives $\rho \to 0$ in the model; 
hence the zero-residual threshold is at most
$\operatorname{rank}(H)$,  and
\eqref{eq:self_consistent_d} is used on the domain
$d < \operatorname{rank}(H)$.
\end{remark}

\begin{remark}[Relation to the conic theorem]
The event analyzed in this section is the quadratic-model analogue of the
affine-slice intersection event in
Theorem~\ref{thm:random_affine_slice_conic}. More precisely,
\eqref{eq:event_identity} gives the exact residual characterization for the
unlocalized quadratic sublevel set $\widetilde S_\varepsilon$. When the
corresponding least-squares solution also lies inside the localization ball,
this event coincides with the localized intersection event used in the
conic theorem.

The conic theorem gives a rigorous transition window centered at the
statistical-dimension quantity $\delta(C^\circ)$ associated with the
localized cone
\[
    C
    =
    \operatorname{cone}
    \left(
        \left(
            \widetilde S_\varepsilon
            \cap
            \overline B(\theta_0,R_{\mathrm{loc}})
        \right)
        -
        \theta_0
    \right).
\]
The three quantities derived in this section,
\[
    \widehat d_{\mathrm{MF}},
    \qquad
    d^\star_{\mathrm{iso}},
    \qquad
    d_{\mathrm{unif}}
    =
    r_{\mathrm{eff}},
\]
should therefore be viewed as mean-level Hessian-based surrogates for this
conic phase-transition center under different levels of displacement
information. The orientation-resolved predictor
$\widehat d_{\mathrm{MF}}$ retains the complete displacement profile
$\{\Delta_i\}$ and is the primary prediction when that profile is
available. The isotropic predictor $d^\star_{\mathrm{iso}}$ replaces the
profile with the equal-energy model $\Delta_i^2=R^2/P$, whereas
$d_{\mathrm{unif}}=r_{\mathrm{eff}}$ removes the orientation dependence
through a uniform upper bound and requires only the displacement radius
$R$. None of these mean-level predictors, by itself, inherits the
high-probability guarantee of the conic theorem; their accuracy is tested
empirically.
\end{remark}

\begin{remark}[Experimental guidance]
Phase-transition experiments should report all predictors supported by the
available displacement information. When the complete displacement profile
$\{\Delta_i\}$ is available or can be estimated reliably,
$\widehat d_{\mathrm{MF}}$ is the primary prediction and should be compared
directly with the measured critical dimension. The same experiments should
also report $d^\star_{\mathrm{iso}}$ as a sharper diagnostic under the
isotropic-orientation model and
\[
    d_{\mathrm{unif}}
    =
    r_{\mathrm{eff}}
\]
as the conservative orientation-uniform, radius-only prediction.

Under controlled isotropic displacements, the measured transition is
expected to be close to $d^\star_{\mathrm{iso}}$. As the displacement
energy is increasingly concentrated along stiff curvature directions,
$\widehat d_{\mathrm{MF}}$ should track the resulting change in the
critical dimension, while $d_{\mathrm{unif}}$ provides the
orientation-uniform mean-level upper prediction. When only a reliable
radius estimate $R$ is available, $d_{\mathrm{unif}}$ is the operational
prediction used by Mode~2 of
Algorithm~\ref{alg:hessian_dim_selection}.
\end{remark}



\end{document}